\documentclass[9.5pt,journal,compsoc]{IEEEtran}

\ifCLASSOPTIONcompsoc
  \usepackage[nocompress]{cite}
\else
  \usepackage{cite}
\fi

\ifCLASSINFOpdf
\else
\fi
\usepackage{graphicx}
\usepackage{amsmath}
\usepackage{amssymb}
\usepackage{color}
\usepackage{makecell}
\usepackage{multicol,lipsum}
\usepackage[most]{tcolorbox}
\usepackage{quiver}
\usepackage{cuted} 
\usepackage{mathrsfs}
\usepackage[stable]{footmisc}

\usepackage{listings}
\usepackage{xcolor}

\lstdefinestyle{mystyle}{
    language=Matlab,
    basicstyle=\ttfamily\small,
    keywordstyle=\color{blue},
    commentstyle=\color{green!60!black},
    stringstyle=\color{red},
    numbers=left,
    numberstyle=\tiny,
    stepnumber=1,
    numbersep=5pt,
    backgroundcolor=\color{gray!10},
    frame=single,
    breaklines=true,
    captionpos=b
}

\newtcbtheorem{TODO}{\bfseries TODO}{enhanced,drop shadow={black!50!white},
  coltitle=black,
  top=0.3in,
  attach boxed title to top right=
  {xshift=0em,yshift=-\tcboxedtitleheight/2},
  boxed title style={size=small,colback=pink}
}{TODO}

\newtheorem{prop}{Proposition}
\newtheorem{theorem}[prop]{Theorem}

\newtheorem{lemma}[prop]{Lemma}
\newtheorem{proposition}[prop]{Proposition}
\newtheorem{example}[prop]{Example}
\newtheorem{definition}[prop]{Definition}
\newtheorem{remark}[prop]{Remark}

\newcommand{\eg}{{\textit{e.g. }}}

\newcommand{\ie}{{\textit{i.e. }}}

\newcommand{\R}{\mathbf{R}}

\newcommand{\bigzero}{\mbox{\normalfont\Huge  $0$}}
\newcommand{\bigeye}{\mbox{\normalfont\Huge  $I$}}
\newcommand{\rvline}{\hspace*{-\arraycolsep}\vline\hspace*{-\arraycolsep}}
\usepackage[pagebackref=true,breaklinks=true,colorlinks,bookmarks=false]{hyperref}

\newenvironment{proof}{\textbf{Proof:}}{\hfill$\square$}

\begin{document}

\title{Condition numbers in multiview \\ geometry,  instability in relative pose \\ estimation, and RANSAC}

\author{Hongyi Fan,
        Joe Kileel,
        and~Benjamin~Kimia%
\IEEEcompsocitemizethanks{\IEEEcompsocthanksitem Hongyi Fan and Benjamin Kimia are with the School
of Engineering, Brown University, Providence,
RI, 02912.\protect\\
\IEEEcompsocthanksitem Joe Kileel is with the Department of Mathematics and Oden Institute for Computational Engineering and Sciences at the University of Texas at Austin, Austin, TX, 78705.}%
\thanks{Kimia and Fan were supported in part by NSF awards IIS-1910530 and IIS-2312745. Kileel was supported in part by NSF awards DMS-2309782 and IIS-2312746.  Hongyi Fan and
Joe Kileel contributed equally to this work. Benjamin Kimia is the corresponding author.}}

\markboth{Condition numbers in multiview
geometry, instability in relative pose
estimation, and RANSAC}%
{Shell \MakeLowercase{\textit{et al.}}: Bare Demo of IEEEtran.cls for Computer Society Journals}

\IEEEtitleabstractindextext{%
\begin{abstract}
In this paper, we introduce a general framework for analyzing the numerical conditioning of minimal problems in multiple view geometry, using tools from computational algebra and Riemannian geometry. 
Special motivation comes from the fact that relative pose estimation, based on standard 5-point or 7-point Random Sample Consensus (RANSAC) algorithms, can fail even when no outliers are present and there is enough data to support a hypothesis. 
We argue that these cases arise due to the intrinsic instability of the 5- and 7-point minimal problems.
We apply our framework to characterize the instabilities, both in terms of the world scenes that lead to infinite condition number, and directly in terms of ill-conditioned image data. 
The approach produces computational tests for assessing the condition number before solving the minimal problem.
Lastly, synthetic and real data experiments suggest that RANSAC serves not only to remove outliers, but in practice it also selects for well-conditioned image data, which is consistent with our theory.
\end{abstract}

}

\maketitle

\IEEEdisplaynontitleabstractindextext

\IEEEpeerreviewmaketitle

\IEEEraisesectionheading{\section{Introduction}\label{sec:introduction}}

\IEEEPARstart{T}{he} past two decades have seen an explosive growth of multiview geometry applications such as the reconstruction of 3D object models for use in
video games~\cite{Ablan:3DPhoto:book},
film~\cite{Kitagawa:Mocap:book},
archaeology~\cite{pollefeys2001image},
architecture~\cite{Luhmann:Photogrammetry:book}, and urban modeling (\eg Google
Street View); match-moving in augmented reality and cinematography for mixing virtual content and real video~\cite{Dobbert:Matchmoving:book}; the organization of a collection of photographs with respect to a scene known as Structure-from-Motion \cite{myozyecsil2017survey};
robotic manipulation~\cite{Horn:Robot:Vision}; and meteorology from cameras in
automobile manufacture and autonomous driving~\cite{Luhmann:Photogrammetry:book}. One key building block of a multiview system is the relative pose estimation of two cameras~\cite{hartleyzisserman,szeliski2010computer}. A methodology that is dominant in applications is RANSAC~\cite{raguram2008comparative}.  This forms hypotheses from a few randomly selected correspondences in two views, \eg 5 in calibrated camera pose estimation~\cite{nister:PAMI:2004} and 7 in uncalibrated camera pose estimation~\cite{stewart1999robust, myseitz2006comparison}, and validates these hypotheses using the remaining putative correspondences. 
A commonly understood reason for practitioners to use RANSAC is its robustness against outliers \cite{barath2020magsacpp, mishkin,brachmann2019neural}. 
The pose of multiple cameras can then be recovered either by incremental ~\cite{schoenberger2016sfm} or global ~\cite{kasten2019algebraic} Structure-from-Motion methods. 
This approach has been quite successful in many applications. 
\begin{figure}[t]
    \centering
    (a) \includegraphics[width=0.4\linewidth]{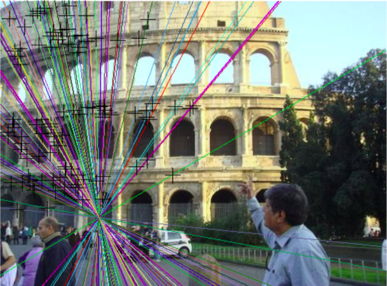}  \includegraphics[width=0.4\linewidth]{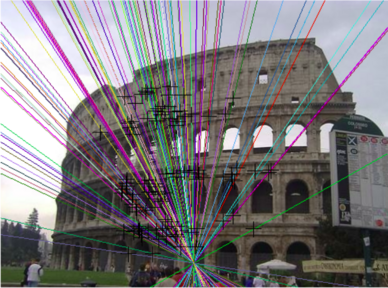} \\
    (b) \includegraphics[width=0.4\linewidth]{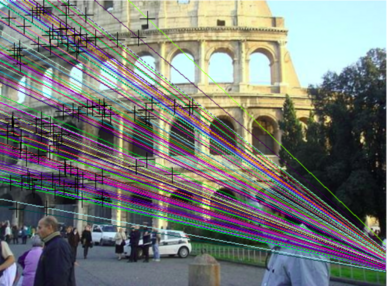}  \includegraphics[width=0.4\linewidth]{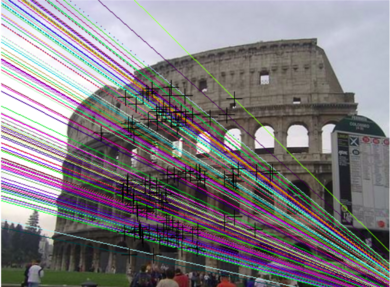}
    \caption{Typical relative pose estimation can fail catastrophically, even with a large number of correspondences (100 correspondences shown in the figure), all of which are inliers. (a) Ground-truth epipolar geometry. (b) Erroneous estimated epipolar geometry, from the 7-point algorithm and LO-RANSAC which also uses local refinement \cite{chum2003locally}. The prime cause of such failure is numerical instability as shown in this paper. }
    \label{fig:Teaser}
\end{figure}

However, there are a non-negligible number of scenarios where the RANSAC-based approach fails, for instance in producing the relative pose between two cameras. 
As an example, when the number of candidate correspondences drops to say $50$ to $100$ correspondences, as is the case for images of homogeneous and low textured surfaces, the pose estimation process  often fails. 
It is curious why the estimation should fail, even if only a modicum of correspondences are available: after all RANSAC can select $5$ from $50$ in $\binom{50}{5} \approx 2.1$ million combinations for the 5-point problem.  Thus there are plenty of veridical correspondences available if the ratio of outliers is low. 
In an experiment with {\em no outliers}, either with synthetic data (see Section~\ref{sec:experimentals}) or manually-cleaned real data (see Figure~\ref{fig:Teaser}), the process can still frequently fail! 
This is mysterious, unless the role of RANSAC goes beyond weeding out the outliers. 
Indeed, this paper argues that a key role for RANSAC is to \textit{stabilize} (or \textit{precondition}) the estimation process, without denying its importance in filtering outliers. 
We will show that the process of estimating pose from a minimal problem is typically unstable to realistic levels of noise, with a gradation of instability depending on the specific choice of $5$ or $7$ points. 
One role of RANSAC then is to find better-conditioned image data by integrating non-selected correspondences: 
if a large number of non-selected correspondences agree with the estimate, then the hypothesis is both free of outliers and -- perhaps as importantly -- it is stable to typical image noise. 

Motivated by relative pose estimation, this paper inspects the general issue of instability in minimal problems in multiview geometry.  
We develop a broadly applicable framework that connects the conditioning of world scene configurations to the condition number of the inverse of the Jacobian matrix of the forward map, and relate this directly to the image data  (Section~\ref{sec:examples}).  
Applying the framework, we give condition number formulas for the $5$-point and $7$-point minimal problems (Propositions~\ref{prop:formula-E} and~\ref{prop:formula-F}). 
Further, we characterize when the condition number is infinite for the $5$- and $7$-point minimal problems,  in terms of the world scene (Theorems~\ref{thm:illposed-world} and~\ref{thm:F-ill-posed-world}) and image data (Theorems~\ref{thm:image-E} and~\ref{thm:F-image}). 
  
Note that the notion of ``critical world scenes" has appeared extensively in the multiview geometry literature already, \eg \cite{krames1941ermittlung,kahl2002critical, maybank2012theory,hartleyzisserman,kahl2001critical, braatelund2024critical,bertolini2007instability,bertolini2022smooth,snapshot} and the corresponding constraint on the images, named as ``degenerate image points" was also investigated~\cite{torr1998robust,bertolini2019critical}. These concepts are defined as follows.

\begin{definition} \label{def:degen} (\eg  \cite[Sec.~2.1]{torr1998robust}) Consider a (not necessarily minimal) multiview geometry estimation problem.  A configuration of noiseless image data is called \textbf{degenerate} if the corresponding epipolar geometry 
is not unique.
\end{definition}
\begin{definition} \label{def:critical} (\eg {\cite{kahl2001critical}})
Consider a (not necessarily minimal) multiview geometry estimation problem. 
A world scene is called \textbf{critical} if it maps to degenerate image data.
\end{definition}

We stress that the ``degenerate" and ``critical" concepts considered previously are fundamentally different from what is studied here. Prior work has mostly on non-uniqueness: that is, configurations of cameras and many world points where the image data does not determine the world scene, despite there being more than a minimal set of measurements. 
By contrast, we consider minimal problems, where usually there are multiple solutions. We quantify the stability of these solutions.  Breakdown of stability occurs at ill-posed world scenes (Definition~\ref{def:ill-W}) and ill-posed image data (Definition~\ref{def:ill-X}), which differ from critical world scenes and degenerate image data.
The results here include characterizations in image space, where the data observations occur, and are not just in terms of the world scene. We remark that prior work has qualitatively shown  instability occurs when the corresponding points are near the critical loci~\cite{bertolini2020critical} or the corresponding points are too close to each other~\cite{zhang1998determining}, but not fully characterized the instability of minimal problems nor provided a general framework. 
See Section~\ref{sec:relationship} for further comparisons with the  literature.

In addition to providing  new theory, this paper proposes a computational method to evaluate the stability of a given minimal image correspondence set.  
The proposed method is general, although we only develop it here for relative pose estimation.  
Our idea is to measure the distance from one image point to a certain  ``ill-posed curve" on the image plane, computed using the other point correspondences (Section~\ref{sec:ImageData}). 
The distance is a means of gauging the stability of the data, before solving the minimal problem.
We present procedures to compute the ill-posed curve for $5$-point and $7$-point minimal problems, using methods from numerical and symbolic algebra (Section~\ref{sec:compute-curve}). 
Experimental results on synthetic and real data show that the measure accurately captures the stability of minimal problems  (Section~\ref{sec:experimentals}).  
Further we find that in practice, poor conditioning presents a  challenge to RANSAC, and RANSAC exhibits a bias for selecting image data well-separated from the ill-posed curve.  

This paper extends our previous conference work~\cite{fan2022instability} with significant improvements. We introduce a new method for computing the ill-posed curve for the $7$-point problem, achieving a $100 \times$ speedup and enabling real-time computations (Section~\ref{sec:symbolic-curve}). The framework is expanded (Section~\ref{sec:framework}) by incorporating a condition number for image data and introducing computational algebra tools (Section~\ref{sec:comp-ag}), making our approach more broadly applicable to other minimal problems. 
We also compare our instability analysis with traditional critical configuration methods (Section~\ref{sec:relationship}) and enhance Section~\ref{sec:examples} with a rigorous quotient manifold construction for world scenes. 
New numerical experiments (Section~\ref{sec:experimentals}) further bridge theory and practice, and indicate that instability also plagues estimation schemes that use local refinement on top of RANSAC. 
Lastly, the overall presentation is substantially refined.

\section{Minimal Problems for Relative Pose Estimation}
\label{sec:examples}
In this section, we review the classic $5$- and $7$-point minimal problems for relative pose estimation, which serve as key motivating examples for our general framework of stability analysis.
We formulate the problems in a way that fits into the general framework presented in Section~\ref{sec:framework}.  Terms appearing in bold will be reused there.

\subsection{Essential matrices and the 5-point problem} \label{example:essential}
This is the minimal problem for computing the essential matrix between two calibrated cameras given $5$ point matches.

Formally, we regard the 5-point problem as described by $3$ spaces and $2$ maps between the spaces.  Firstly, we define world scenes to be tuples of $2$ calibrated cameras with distinct centers whose projection matrices are $P$ and $\bar{P}$, and $5$ world points $\Gamma_1$, ..., $\Gamma_5$, up to a global change of world coordinates.  
Equivalence up to change of world coordinates is made precise using the quotient manifold theorem  \cite[Thm.~21.10]{lee2013smooth}.
Specifically, let $\mathcal{W}$ be the \textbf{world scene space}:
\begin{multline} \label{eq:W-essential}
\!\!\!\!\!\! \mathcal{W} = \big{\{} (P, \bar{P}, \Gamma_1, \ldots, \Gamma_5) \in \mathcal{C}^{\times 2} \times (\mathbb{R}^3)^{\times 5} \! \! : 
\operatorname{ker}(P) \neq \operatorname{ker}(\bar{P}), \\ 
\,\,\,\,\,\,\,\,\,\,\, (P \, (\Gamma_i; \, 1 ))_3 \neq 0 \text{ and } (\bar{P} (\Gamma_i;  1 ))_3 \neq 0  \text{ for all } i  \big{\}}  \big{/} \, G \\[0.2em]
 = \{ (P, \bar{P}, \Gamma_1, \ldots, \Gamma_5) \text{ mod } G \}, \hspace{3.1cm}
\end{multline}
where  
\begin{equation*}
\mathcal{C} = \{P \in \mathbb{P}(\mathbb{R}^{3 \times 4}) : \exists \, \mathbf{R} \in \operatorname{SO}(3), \mathbf{T} \in \mathbb{R}^3 \text{ s.t. } P = (\mathbf{R} \,\,  \mathbf{T}
)\}
\end{equation*}
consists of calibrated cameras, and $\Gamma_i \in \mathbb{R}^3$ $(i = 1, \ldots, 5)$ are the $3$D points, $\operatorname{ker}(P)$ represents the kernel of the projection matrix, $\operatorname{ker}(P) \neq \operatorname{ker}(\bar{P})$ limits the definition to cameras with distinct centers, and $(\cdot)_3$ means taking the third coordinate of the vector so that the conditions $(\Gamma_i; \, 1 ))_3 \neq 0 \text{ and } (\bar{P} (\Gamma_i;  1 ))_3 \neq 0$  indicate that the world points do not project to points at infinity.
The group in \eqref{eq:W-essential} is
{\small
\begin{align*}
G = \{ g \in \mathbb{P}(\mathbb{R}^{4 \times 4}) : \exists \, \mathbf{R} \in \operatorname{SO}(3), \mathbf{T} \in \mathbb{R}^3, \lambda \in \mathbb{R} \setminus \{0\} \\ \text{ such that } g = \begin{pmatrix} \mathbf{R} & \mathbf{T} \\ 0 & \lambda \end{pmatrix} \},
\end{align*}
}
\!\!\!\! i.e., world transformations preserving calibration.  The group $G$ acts on 
$\mathcal{C}^{\times 2} \times (\mathbb{R}^3)^{\times 5}$ via
\begin{equation}\label{eq:ess-action}
g \cdot (P, \bar{P}, \Gamma_1, \ldots) := (Pg^{-1}, \bar{P}g^{-1},  \tfrac{1}{\lambda} (\mathbf{R} \Gamma_1 + \mathbf{T}), \ldots).
\end{equation}
Further, in \eqref{eq:W-essential} ``mod $G$" is shorthand for the orbit under the group action.  
Altogether, the space $\mathcal{W}$ collects together the possible world scenes in the 5-point problem.

\begin{proposition} \label{prop:quotient-ess}
The action of $G$ on the subset of $\mathcal{C}^{\times 2} \times (\mathbb{R}^3)^{\times 5}$ in \eqref{eq:W-essential} is smooth, free and proper.  Therefore the quotient  $\mathcal{W}$ is canonically a smooth manifold of dimension $20$.
\end{proposition} 
Proposition~\ref{prop:quotient-ess} is proven in the appendices.

The quotient space construction in Proposition~\ref{prop:quotient-ess} is nice because it is intrinsic.  This means we may use any coordinate system and our conclusions should not be affected. 
That said, we  use the following coordinates (here  $\mathbb{S}^2 \subseteq \mathbb{R}^3$ is the unit sphere, consisting of unit-norm translations $\hat{\mathbf{T}}$): 

\begin{lemma}\label{lem:double}
Let $\mathcal{U} = \{(\mathbf{R}, \hat{\mathbf{T}}, \Gamma_1, \ldots, \Gamma_5)\}$ be the open subset of  $\operatorname{SO}(3) \times \,  \mathbb{S}^2 \times (\mathbb{R}^3)^{\times 5}$ where $(\Gamma_i)_3 \neq 0$ and  $(\mathbf{R}\Gamma_i + \hat{\mathbf{T}})_3 \neq 0$ for $i = 1, \ldots, 5$.  
Then the map $\mathcal{U} \rightarrow \mathcal{W}$ sending $(\mathbf{R}, \hat{\mathbf{T}}, \Gamma_1, \ldots, \Gamma_5)$ to $((I \,\, 0), (\mathbf{R} \,\, \hat{\mathbf{T}} ), \Gamma_1, \ldots, \Gamma_5) \, \text{mod } G$ is a smooth double cover of $\mathcal{W}$ (as in \cite[Ch.~4]{lee2013smooth}).
\end{lemma}
Lemma~\ref{lem:double} is proven in the appendices.  
The map $\mathcal{U} \rightarrow \mathcal{W}$ in Lemma~\ref{lem:double} identifies $(\mathbf{R}$, $\hat{\mathbf{T}}$, $\Gamma_1$, $\ldots$, $\Gamma_5)$ and $(\mathbf{R}$, $-\hat{\mathbf{T}}$, $-\Gamma_1$, $\ldots$, $-\Gamma_5)$, which is a well-known ambiguity in the 5-point problem \cite[Sec.~9.6.2]{hartleyzisserman}.  

Next up, define $\mathcal{X}$ as the relevant \textbf{image data space}: 
\begin{equation}
   \mathcal{X} = \left( \mathbb{R}^2 \times \mathbb{R}^2 \right)^{\times 5} = \{(\gamma_1, \bar{\gamma}_1), \ldots, (\gamma_5, \bar{\gamma}_5)\}.
\end{equation}
This consists of $5$-tuples of image point pairs.  
Here 
\begin{equation*}
\gamma_i \in \mathbb{R}^2 \quad \text{and} \quad \bar{\gamma_i} \in \mathbb{R}^2 \quad \quad (i = 1, \ldots, 5)
\end{equation*}
denote corresponding points in the two images.   The space $\mathcal{X}$ collects together all possible image data in the 5-point problem.

Thirdly, define $\mathcal{Y}$ as the relevant \textbf{output space}:
\begin{align} \label{eq:essential-ideal}
     \mathcal{Y} &= \{E \in \mathbb{P}(\mathbb{R}^{3 \times 3}) : E \text{ is an essential matrix}\}  \nonumber \\ 
     &= \{E \in \mathbb{P}(\mathbb{R}^{3 \times 3}): 2EE^{\top}\!E - \operatorname{tr}(EE^{\top}\!)E  = 0, \nonumber  \det(E)=0 \} \nonumber \\
    &= \{E \in \mathbb{P}(\mathbb{R}^{3 \times 3}) : \sigma_1(E) = \sigma_2(E) > \sigma_3(E) = 0\}.
\end{align}
where $\sigma_i$ ($i = 1,2,3$) denotes the three singular values of the essential matrix. 
The output space $\mathcal{Y}$ consists of the quantities of interest that we solve for in the 5-point problem, \ie essential matrices.    
It is well-known that essential matrices can be characterized in terms of the vanishing of ten cubic equations  \cite{demazure1988deux}, or in terms of singular values, as  in \eqref{eq:essential-ideal}.

As one of two relevant maps in the 5-point problem, define $\Phi$ as the  \textbf{forward map} from world scenes to image data:
\begin{align} \label{eq:E-forward}
& \Phi(\R,\hat{\mathbf{T}},\Gamma_1,\ldots, \Gamma_5 ) = ((\gamma_1, \bar{\gamma}_1), \ldots, (\gamma_5, \bar{\gamma}_5)), \\
& \,\, \text{for }  \gamma_i = \pi(\Gamma_i) \, \text{ and } \, \bar{\gamma}_i = \pi(\R\Gamma_i + \hat{\mathbf{T}}) \nonumber \\
& \,\, \text{where }  \pi(x,y,z) = (x/z, y/z) \, \text{ is perspective projection.} \nonumber
\end{align}
Since $(\mathbf{R}, \hat{\mathbf{T}}, \Gamma_1, \ldots, \Gamma_5)$ and $(\mathbf{R}, -\hat{\mathbf{T}}, -\Gamma_1, \ldots, -\Gamma_5)$ have the same image in \eqref{eq:E-forward}, $\Phi$ is well-defined on $\mathcal{W}$.  The forward map formalizes data generation in the 5-point problem.

As a second map, define $\Psi$ to be the relevant \textbf{output map} from world scenes to the quantities of interest:
\begin{equation}\label{eq:E-psi-map}
    \Psi(  \R, \hat{\mathbf{T}}, \Gamma_1, \ldots, \Gamma_5 ) = E = [\hat{\mathbf{T}}]_{\times} \R  \in \mathbb{P}(\mathbb{R}^{3 \times 3}).
\end{equation}
Here $[\hat{\mathbf{T}}]_{\times} \in \mathbb{R}^{3 \times 3}$
is the skew symmetric matrix representation of  cross product with  $\hat{\mathbf{T}} \in \mathbb{R}^3$ \cite[Sec.~9.6]{hartleyzisserman}. 

The 5-point \textbf{minimal problem} is given by the tuple $\mathscr{M} = (\mathcal{W}, \mathcal{X}, \mathcal{Y}, \Phi, \Psi)$.  The computational task of the problem is to compute outputs $\Psi(\Phi^{-1}(\cdot))$ in $\mathcal{Y}$ for different instances of image data in $\mathcal{X}$.
The output $\Psi(\Phi^{-1}(\cdot))$ consists of up to $10$ essential matrices in $\mathcal{Y}$. 
\cite[Part~II]{hartleyzisserman} shows these are the essential matrices satisfying the \textbf{epipolar constraints}:
\begin{equation}\label{eq:E-epipolar}
(\bar{\gamma}_i; 1)^{\top} E \, (\gamma_i; 1) = 0 \quad \text{ for } i =1, \ldots, 5.
\end{equation}
The essential matrices can be found by solvers based on computing the real roots of a degree $10$ univariate polynomial \cite{nister:PAMI:2004}.  
In this paper our concern is not with specific solvers for the 5-point problem, but with the intrinsic numerical conditioning of the problem: How sensitive are the essential matrices to perturbations in the image data? Which world scenes correspond to ill-conditioned image data?  
How can we detect directly from the image data that a problem instance is poorly conditioned? 
After presenting a general analysis framework in Section~\ref{sec:framework}, we answer these questions for the 5-point problem in Section~\ref{sec:main-results}.

\subsection{Fundamental matrices and the 7-point problem}
 \label{example:fundamental}
The other minimal problem analyzed in depth in this paper is the $7$-point problem.  This is the minimal problem for computing the fundamental matrix between two uncalibrated cameras given $7$ point matches.

Formally, we regard the 7-point problem as described by $3$ spaces and $2$ maps between the spaces, similarly to our setup of the 5-point problem.
We define the \textbf{world scene space} $\mathcal{W}$ as a quotient:
\begin{multline} \label{eq:W-fundamental}
\!\!\!\!\!\! \mathcal{W} = \big{\{} (P, \bar{P}, \tilde{\Gamma}_1, \ldots, \tilde{\Gamma}_7) \in \mathcal{C}^{\times 2} \times (\mathbb{P}_{\mathbb{R}}^3)^{\times 7} \! \! : 
\operatorname{ker}(P) \neq \operatorname{ker}(\bar{P}), \\ 
\,\,\,\,\,\,\,\,\,\,\,\,\,\,\,\, (P \, \tilde{\Gamma}_i)_3 \neq 0 \text{ and } (\bar{P} \tilde{\Gamma}_i)_3 \neq 0  \text{ for all } i  \big{\}} \! \Big{/} \! \operatorname{PGL}(4) \\[0.2em]
  = \{ (P, \bar{P}, \tilde{\Gamma}_1, \ldots, \tilde{\Gamma}_7) \text{ mod } \operatorname{PGL}(4) \}. \hspace{2.4cm}
\end{multline} 
Like in the 5-point problem, $\operatorname{ker}(P) \neq \operatorname{ker}(\bar{P})$ ensures that the two cameras have distinct centers, and the other inequations ensure that none of the world points project to infinity in either image. Here
\begin{equation*}
\mathcal{C} = \{ P \in \mathbb{P}(\mathbb{R}^{3 \times 4}) : \operatorname{rank}(P) = 3\}
\end{equation*}
stands for projective cameras, and 
$
\tilde{\Gamma}_i \in \mathbb{P}_{\mathbb{R}}^3
$
are projective world points.
Also,  $$\operatorname{PGL}(4) = \{ g \in \mathbb{P}(\mathbb{R}^{4 \times 4}) : \det(g) \neq 0\}$$ is the group of projective world transformations.  The group $G$ acts on $\mathcal{C}^{\times 2} \times (\mathbb{P}_{\mathbb{R}}^{3})^{\times 7}$ in \eqref{eq:W-fundamental} via
\begin{equation} \label{eq:action}
g \cdot (P, \bar{P}, \tilde{\Gamma}_1, \ldots, \tilde{\Gamma}_7) = (Pg^{-1}, \bar{P}g^{-1}, g \tilde{\Gamma}_1, \ldots, g \tilde{\Gamma}_7).
\end{equation}
The world scene space $\mathcal{W}$ collects together the possible world scenes in the 7-point problem.   

\begin{proposition}
\label{prop:quotient-smooth} The action of $\operatorname{PGL}(4)$ on the subset of $\mathcal{C}^{\times 2} \times (\mathbb{P}_{\mathbb{R}}^3)^{\times 7}$ in \eqref{eq:W-fundamental} is smooth, free and proper.  Therefore the quotient $\mathcal{W}$ is canonically a smooth manifold of dimension $28$. 
\end{proposition}
Proposition~\ref{prop:quotient-smooth} is proven in the appendices.

\begin{lemma}\label{lem:b-coords}
Let $\mathbf{M}(b) = \begin{pmatrix} 1 & b_1 & b_2 & b_3 \\ b_4 & b_5 & b_6 & b_7 \\ 0 & 0 & 0 & 1 \end{pmatrix}$ for $b \in \mathbb{R}^7$ and $\tilde{\Gamma}_i = (\Gamma_i; 1) \in \mathbb{R}^4$ with $\Gamma_i \in \mathbb{R}^3$. 
Let $\mathcal{U} = \{(b, \Gamma_1, \ldots, \Gamma_7)\}$ be the open subset of $\mathbb{R}^7 \times (\mathbb{R}^3)^{\times 7}$
where $\operatorname{rank}(\mathbf{M}(b)) = 3$, $(\Gamma_i)_3 \neq 0$ and $(\mathbf{M}(b) \tilde{\Gamma}_i)_{3} \neq 0$ for $i = 1, \ldots, 5$.
Then the map from $\mathcal{U} \rightarrow \mathcal{W}$ sending $(b, \Gamma_1, \ldots, \Gamma_7)$ to $((I \,\, 0), \mathbf{M}(b), \tilde{\Gamma}_1, \ldots, \tilde{\Gamma}_7) \, \text{mod} \operatorname{PGL}(4)$ is a smooth parameterization of an open dense subset of $\mathcal{W}$.
Further, suppose in $\mathbf{M}$ we swap the third row of $\mathbf{M}$ with the first or second row and/or the first column of $\mathbf{M}$ with the second or third column, and suppose in $\tilde{\Gamma}_i$ we swap the fourth entry of $\tilde{\Gamma}_i$ with one of the first three entries.
Then the resulting maps form an atlas \nolinebreak for \nolinebreak $\mathcal{W}$.
\end{lemma}
Lemma~\ref{lem:b-coords} is proven in the appendices.   It gives coordinates for world scenes in the 7-point problem.

Next, we define $\mathcal{X}$ as the  \textbf{image data space}:
\begin{equation}
\mathcal{X} = (\mathbb{R}^2 \times \mathbb{R}^2)^{\times 7} = \{(\gamma_1, \bar{\gamma}_1), \ldots, (\gamma_7, \bar{\gamma}_7)\},
\end{equation}
consisting of $7$-tuples of image point pairs.  The space $\mathcal{X}$ collects together all possible image data in the 7-point problem.

Thirdly, let $\mathcal{Y}$ be the relevant \textbf{output space}:
\begin{align}
    \mathcal{Y} &= \{ F \in \mathbb{P}(\mathbb{R}^{3 \times 3}) : F \text{ is a fundamental matrix} \} \nonumber
    \\ &= \{ F \in \mathbb{P}(\mathbb{R}^{3 \times 3} ) : \operatorname{rank}(F) = 2\}.
\end{align}
This includes the quantities of interest, i.e., the fundamental matrices.

Let $\Phi$ be \textbf{forward map} from world scenes to image data:
 \begin{align} \label{eq:F-forward}
& \Phi( P, \bar{P}, \tilde{\Gamma}_1,\ldots, \tilde{\Gamma}_7 ) = ((\gamma_1, \bar{\gamma}_1), \ldots, (\gamma_7, \bar{\gamma}_7)), \\
& \,\, \text{for }  \gamma_i = \pi(P \tilde{\Gamma}_i) \, \text{ and } \, \bar{\gamma}_i = \pi(\bar{P} \tilde{\Gamma}_i) \nonumber \\
& \,\, \text{where }  \pi(x,y,z) = (x/z, y/z). \nonumber
\end{align}
Note $P \tilde{\Gamma}_i$ and $\bar{P} \tilde{\Gamma}_i$ are well-defined on $\mathcal{W}$, as $g^{-1}$ and $g$ cancel out in the action \eqref{eq:action}.

As our second map, define $\Psi$ as the \textbf{output map} from world scenes to the quantities of interest: 
\begin{equation}\label{eq:fund-output}
    \Psi(P, \bar{P}, \tilde{\Gamma}_1, \ldots, \tilde{\Gamma}_7) = \text{fund. matrix of } P, \bar{P} \,\, \in \mathbb{P}(\mathbb{R}^{3 \times 3}).
\end{equation}
The explicit formula for \eqref{eq:fund-output} is in~\cite[Sec. 9.2.2]{hartleyzisserman}. 

The 7-point \textbf{minimal problem} is given by the tuple $\mathscr{M} = (\mathcal{W}, \mathcal{X}, \mathcal{Y}, \Phi, \Psi)$.  The computational task is to compute outputs $\Psi(\Phi^{-1}(\cdot))$ in $\mathcal{Y}$ for different instances of image data in $\mathcal{X}$.
The output $\Psi(\Phi^{-1}(\cdot))$ is up to $3$ real fundamental matrices in $\mathcal{Y}$. \cite[Part~II]{hartleyzisserman} implies these are the fundamental matrices satisfying the \textbf{epipolar constraints}:
\begin{equation}\label{eq:F-epipolar}
(\bar{\gamma}_i; 1)^{\top} F \, (\gamma_i; 1) = 0 \quad \text{ for } i =1, \ldots, 7.
\end{equation}
The solutions can be found by computing the real roots of a cubic univariate polynomial \cite[Sec.~11.1.2]{hartleyzisserman}.
The questions we want to answer concern  numerical conditioning, and are the same as in Section~\ref{example:essential}: 
How sensitive are the outputs to noise in the input?  Which world scenes and image data instances are ill-posed? 
Answers for the 7-point problem are in Section~\ref{sec:main-results}, obtained using our general framework.

\section{General Framework} \label{sec:framework}

In this section, we introduce a new general framework to analyze instabilities in minimal problems in multiview geometry. We apply this framework to the 5-point and 7-point problems later in the paper. The framework is applicable broadly to analyzing minimal problems, and future research will extend its application to other important problems.

\subsection{Spaces and maps} \label{sec:spaces}
In general, we model minimal problems as involving three spaces.   
Typically, they are all smooth manifolds:
\begin{itemize}
    \item $\mathcal{W}$, the \textbf{world scene space};
    \item $\mathcal{X}$, the \textbf{image data space}; 
    \item $\mathcal{Y}$, the \textbf{output space}.
\end{itemize}
To quantify instability,  distances are needed on the input space $\mathcal{X}$ and the output space $\mathcal{Y}$. 
For this we adopt the formalism used by B{\"u}rgisser and Cucker \cite{burgisser2013condition}, and assume that $\mathcal{X}$ and $\mathcal{Y}$ are endowed with Riemannian metrics.
Let $T(\mathcal{X}, \cdot)$ denote tangent spaces to $\mathcal{X}$, 
$\langle \cdot, \cdot \rangle$ its inner product, and $d_{\mathcal{X}}(\cdot, \cdot)$ the geodesic distance on $\mathcal{X}$, and likewise for other spaces.
In the appendices, we list tangent spaces and Riemannian metrics for the manifolds involved in the minimal problems  from Section~\ref{sec:examples} for relative pose estimation.

We also model minimal problems as featuring two maps:
\begin{itemize}
\item $\Phi$, the \textbf{forward map}; 
\item $\Psi$, the \textbf{output map}.
\end{itemize}
The maps are typically smooth, and connect the spaces as:
\[\begin{tikzcd}
	{\mathcal{Y}} & {\mathcal{W}} & {\mathcal{X}}
	\arrow["\Psi", from=1-2, to=1-1]
	\arrow["\Phi"', from=1-2, to=1-3].
\end{tikzcd}\]

The \textbf{minimal problem} is the tuple $\mathscr{M} = (\mathcal{W}, \mathcal{X}, \mathcal{Y}, \Phi, \Psi)$.  The corresponding computational task is to compute:
\begin{equation}\label{eq:general-output}
\Psi(\Phi^{-1}(\cdot))
\end{equation}
for different image data inputs $x \in \mathcal{X}$.  The set $\Psi(\Phi^{-1}(\cdot))$ is a subset of $\mathcal{Y}$,  all outputs compatible with the given input.

The descriptor ``minimal" refers to the case where the output set \eqref{eq:general-output} is finite almost always, and nonempty for $x$ in a positive measure subset of $\mathcal{X}$ (with the measure on $\mathcal{X}$ induced by its Riemannian metric).  
To ensure this is satisfied, following \cite{duff2020pl} we assume further that $\mathcal{W}, \mathcal{X}, \mathcal{Y}$ are quasi-projective irreducible connected real algebraic varieties with $\dim(\mathcal{W}) = \dim(\mathcal{X})$.
Additionally, we assume that $\Phi$ and $\Psi$ are dominant algebraic maps. 
 Together these assumptions guarantee that
\eqref{eq:general-output} is a finite set for almost all $x$, and that \eqref{eq:general-output} is nonempty for $x$ in a positive measure subset of $\mathcal{X}$ \cite{harris2013algebraic}.   
Informally, the assumptions simply say that all spaces and maps can be described by polynomial equations.

\subsection{Definition of condition number}
Let $\mathscr{M} = (\mathcal{W}, \mathcal{X}, \mathcal{Y}, \Phi, \Psi)$ be a minimal problem.
Condition numbers quantify the worst-case first-order amplification of error from the input to the output \cite[Sec.~O.2]{burgisser2013condition}.
For minimal problems we can attach condition numbers either to world scenes or to image data. 

\begin{definition}\label{def:cond-w}
Let $w \in \mathcal{W}$ be a world scene. 
Define the \textbf{condition number of $\mathscr{M}$ associated to $w$} by
\begin{equation}\label{eq:condition-num}
\operatorname{cond}(\mathscr{M}, w) := \sigma_{\max} \left(D \Psi(w) \circ D \Phi(w)^{-1}\right).
\end{equation}
\end{definition}
We explain the notation.  
Here $D\Phi(w)$ and $D \Psi (w)$ denote the linearizations (or derivatives) of $\Phi$ and $\Psi$ at $w$.  Thus $D\Phi(w): T(\mathcal{W},w) \rightarrow T(\mathcal{X}, \Phi(x))$ is a linear map between tangent spaces, and likewise for  $D\Psi(w)$.
On the right-hand side of \eqref{eq:condition-num}, $\sigma_{\operatorname{max}}$ denotes the maximum singular value (or operator norm).  For singular values to be well-defined, we need inner products on the domain and codomain of $D \Psi(w) \circ D \Phi(w)^{-1}$. These are provided by the Riemannian metric on $\mathcal{X}$ and $\mathcal{Y}$ respectively.
All together to compute \eqref{eq:condition-num}, we need to express the derivatives as Jacobian matrices with respect to orthonormal bases on 
$T(\mathcal{X}, \Phi(w))$ and $T(\mathcal{Y}, \Psi(w))$ and any basis on $T(\mathcal{W}, w)$.  Then 
$\operatorname{cond}(\mathscr{M}, w)$ 
is just the maximum singular value of the matrix product in \eqref{eq:condition-num}. 
We remark that our definition uses maximum singular value rather than the ratio of the maximum to minimum singular values, in order to capture the amplification of \textit{absolute} errors rather than \textit{relative} errors; see \cite[Sec.~14.1]{burgisser2013condition}.

\begin{lemma}\label{lem:justify}
Let $w \in \mathcal{W}$ be a world scene, and assume the derivative 
$D \Phi(w)$ is invertible.  
Then there exist open neighborhoods $\mathcal{X}_0 \subseteq \mathcal{X}$ of $x=\Phi(w)$ and $\mathcal{W}_0 \subseteq \mathcal{W}$ of $w$ such that there is a unique continuous map $\Theta : \mathcal{X}_0 \rightarrow \mathcal{Y}_0$ with $\Theta(x) = w$ and $\Phi \circ \Theta = \operatorname{id}_{\mathcal{U}}$.  
Also, $\Theta$ is smooth, and
\begin{equation}
D(\Psi \circ \Theta)(x) = D\Psi(w) \circ D \Phi(w)^{-1}.
\end{equation}
\end{lemma}
\begin{proof}
This is immediate from the inverse function theorem and chain rule \cite[Thm.~4.5 and Prop.~3.6]{lee2013smooth}.
\end{proof}

\smallskip

Lemma~\ref{lem:justify} shows that Definition~\ref{def:cond-w} quantifies the most that the output associated with $w$ (to first order), as $x = \Phi(w)$ is perturbed.

\begin{definition}\label{def:cond-x}
Let $x \in \mathcal{X}$ be an instance of image data.  Define the \textbf{condition number of $\mathscr{M}$ associated to $x$} by 
\begin{equation}\label{eq:cond-x}
\operatorname{cond}(\mathscr{M}, x) := \max \{ \operatorname{cond}(\mathscr{M}, w) : \Phi(w) = x  \}.
\end{equation}
\end{definition}
Again by Lemma~\ref{lem:justify}, Definition~\ref{def:cond-x} captures the most that \textit{any} of the outputs associated to $x$ can change, as $x$ is perturbed.

We emphasize that computing \eqref{eq:cond-x} is a priori as difficult  solving the minimal problem \eqref{eq:general-output} itself for $x$, because all world scenes satisfying $\Phi(w) = x$ are involved in \eqref{eq:cond-x}.  
However, there is a method to precompute an approximation to the condition number $\operatorname{cond}(\mathscr{M},x)$ without solving the minimal problem for $x$, explained below (Sections~\ref{sec:comp-ag} \nolinebreak and \nolinebreak \ref{sec:compute-curve}).

\subsection{Definition of ill-posed locus}\label{sec:ill-pose}

The world scenes and image data that have \textit{infinite} condition number play important roles in instability analysis.

\begin{definition}\label{def:ill-W}
Define the \textbf{ill-posed locus of $\mathscr{M}$ in $\mathcal{W}$} by
\begin{equation}
\operatorname{ill}(\mathscr{M}, \mathcal{W}) := \{ w \in \mathcal{W} : D\Phi(w) \text{ is rank-deficient} \}.
\end{equation}
\end{definition}

\begin{definition}\label{def:ill-X}
Define the \textbf{ill-posed locus of $\mathscr{M}$ in $\mathcal{X}$} by
\begin{equation}\label{eq:ill-X-def}
\operatorname{ill}(\mathscr{M}, \mathcal{X}) := \Phi(\operatorname{ill}(\mathscr{M}, \mathcal{W})).
\end{equation}
\end{definition}
Ill-posed loci are where minimal problems can break down completely.  Their neighborhoods are serious danger zones. 
If $w$ lies \textit{exactly} in $\operatorname{ill}(\mathscr{M}, \mathcal{W})$, the output associated with $w$ has unbounded error compared to the error in the input, as  
$x = \Phi(w)$  is perturbed.  
Similarly if $x \in \operatorname{ill}(\mathscr{M},\mathcal{X})$, \textit{one of the outputs} associated to $x$ suffers unbounded error amplification.  
Meanwhile, at world scenes or image data lying \textit{close} to the ill-posed loci, the condition number is finite but very large, which is problematic numerically.  

Useful heuristics are the \textbf{reciprocal distances} to the ill-posed loci: 
\begin{equation} \label{eq:proxy}
1 \, \big{/} \, d_{\mathcal{W}}(w, \operatorname{ill}(\mathscr{M}, \mathcal{W})) \, \quad \text{and} \quad \, 1 \, \big{/} \, d_{\mathcal{X}}(x, \operatorname{ill}(\mathscr{M}, \mathcal{X})).
\end{equation}
These can serve as proxies for the conditions numbers:
\begin{equation*}
\operatorname{cond}(\mathscr{M},w) \,\,\, \quad \text{and} \,\,\, \quad \operatorname{cond}(\mathscr{M},x),
\end{equation*}
respectively.
The approximation given by the reciprocal distances is known to be exact in some cases in numerical computing.  
For example, in \cite[Sec.~3]{demmel1987condition} Demmel studied matrix inversion. 
Given a square matrix, he proved that its reciprocal Frobenius distance to rank-deficient matrices equals the condition number relevant to matrix inversion.

\subsection{Tool from computational algebra}\label{sec:comp-ag}
Many minimal problems in multiview geometry can be cast in terms of epipolar constraints, which are linear constraints in the output variable.  
When this holds, a tool from computational algebra may be applied to instability analysis. 

We assume that the output space
$\mathcal{Y}$ is embedded in a real projective space $\mathbb{P}_{\mathbb{R}}^n$ (some $n$). Let $\dim(\mathcal{Y}) = d$.

\begin{definition}\label{def:epipolar-map}
Consider a minimal problem $\mathscr{M} = (\mathcal{W}, \mathcal{X}, \mathcal{Y}, \Phi, \Psi)$.  Say that an algebraic map $\mathbf{L} : \mathcal{X} \rightarrow \mathbb{P}(\mathbb{R}^{k \times (n+1)})$ is an \textbf{epipolar map for $\mathscr{M}$}  if 
\begin{equation}
\Psi(\Phi^{-1}(x)) = \mathcal{Y} \cap \operatorname{ker} \mathbf{L}(x) \quad \text{for all } x \in \mathcal{X}, \text{ and} \nonumber
\end{equation}
\begin{equation}\label{eq:epi-general}
\dim(\operatorname{ker} \mathbf{L}(x)) = n - d \text{ for Zariski-generic } x \in \mathcal{X}.
\end{equation}
\end{definition}
In Definition~\ref{def:epipolar-map}, $k$ is a fixed integer, $\mathbf{L}(x)$ is a $k \times (n+1)$ matrix,
and the right kernel $\operatorname{ker} \mathbf{L}(x)$ is a linear subspace of $\mathbb{P}_{\mathbb{R}}^{n}$. 
The term ``Zariski-generic" means that the property holds for all $x \in \mathcal{X} \setminus \tilde{\mathcal{X}}$ where $\tilde{\mathcal{X}}$ is a proper subvariety of $\mathcal{X}$; in layman's terms, it holds for almost all $x$.  
Definition~\ref{def:epipolar-map} captures the situation that the output sets $\Phi(\Psi^{-1}(\cdot))$ can be computed by imposing linear constraints on $\mathcal{Y}$.

\begin{example}\label{ex:mathL}
In the $5$-point problem (Section~\ref{example:essential}),  $n=8$, $d=5$ and $k=5$. 
The epipolar constraints are encoded by the $5 \times 9$ matrix $\mathbf{L}\big{(}(\gamma_1, \bar{\gamma}_1), \ldots, (\gamma_5, \bar{\gamma}_5)\big{)}$, whose $i$th row is the vectorization of $[\gamma_1; 1] \otimes [\bar{\gamma}_i; 1]$ (see \cite{connelly2023geometry,connelly2024geometryp2}).
We can compute the essential matrices that are compatible with image data by imposing $5$ linear constraints on $\mathcal{Y}$.
Similarly, the 7-point problem has an epipolar map.
\end{example}

Minimal problems with an epipolar map have
a very nice geometric interpretation:  
their output is the intersection of $\mathcal{Y}$ with a varying linear space in $\mathbb{P}_{\mathbb{R}}^n$ of complementary dimension.

\begin{definition}
Define the \textbf{discriminant of $\mathscr{M}$ and $\mathbf{L}$ in $\mathcal{X}$} as 
\begin{multline} \label{eq:non-transverse}
\operatorname{disc}(\mathscr{M}, \mathbf{L}) := \{ x \in \mathcal{X} : \operatorname{rank}(\mathbf{L}(x)) < d \, \text{ or } \,\,\,\,\,\,\,\,\,\, \\ \mathcal{Y} \cap \operatorname{ker} \mathbf{L}(x) \text{ is not a transversal intersection}\}.
\end{multline}
\end{definition}
Recall that $\mathcal{Y} \cap \operatorname{ker} \mathbf{L}(x) \subseteq \mathbb{P}_{\mathbb{R}}^n$ is a transversal intersection if for all $y \in \mathcal{Y} \cap \operatorname{ker} \mathbf{L}(x)$ it holds 
$
T(\mathcal{Y},y) + T(\operatorname{ker} \mathbf{L}(x), y) = T(\mathbb{P}^n_{\mathbb{R}}, y);
$
see \cite[Sec.~6]{lee2013smooth}.  
In layman's terms, the discriminant captures partial tangencies, or ``double roots".  

There is a characterization of the discriminant, due to Sturmfels.

\begin{theorem}{\cite[Thm.~1.1]{sturmfels2017hurwitz}} \label{thm:hurwitz}
There exists a nonzero homogenous polynomial $\mathbf{P}_{\mathcal{Y}}$ in the $\binom{n+1}{d}$ Pl\"ucker coordinates for codimension-$d$ linear subspaces of $\mathbb{P}_{\mathbb{R}}^n$, depending only on $\mathcal{Y}$, such that 
\begin{equation}\label{eq:Py}
x \in \operatorname{disc}(\mathscr{M}, \mathbf{L}) \, \implies \, \mathbf{P}_{\mathcal{Y}}(\operatorname{ker} \mathbf{L}(x)) = 0.
\end{equation} 
Furthermore, the degree of $\mathbf{P}_{\mathcal{Y}}$ is $2p + 2g - 2$ where $p$ and $g$ are the degree and sectional genus of the complex Zariski closure $\mathcal{Y}_{\mathbb{C}} \subseteq \mathbb{P}^{n}_{\mathbb{C}}$ of $\mathcal{Y}$, assuming $p \geq 2$ and the singular locus of $\mathcal{Y}_{\mathbb{C}}$ has codimension at least $2$.  
\end{theorem}
We review Pl\"ucker coordinates in the appendices; here, they are the maximal minors of $\mathbf{L}(x)$.  In particular, in the case of $\operatorname{rank}(\mathbf{L}(x)) < d$ in \eqref{eq:non-transverse}, the Pl\"ucker coordinates are identically $0$, so that the conclusion $\mathbf{P}_{\mathcal{Y}}(\operatorname{ker} \mathbf{L}(x)) = 0$ in \eqref{eq:Py} holds automatically. 
For background on degree, sectional genus and other notions from algebraic geometry, see \cite{sturmfels2017hurwitz} and references therein.

The discriminant is closely related to the ill-posed locus.

\begin{lemma}\label{lem:local-disc}
Let $x \in \mathcal{X} \setminus \operatorname{disc}(\mathscr{M}, \mathbf{L})$ and $y \in \mathcal{Y} \cap \operatorname{ker} \mathbf{L}(x) = \Psi(\Phi^{-1}(x)) $.  Then there exist open neighborhoods $\mathcal{X}_0 \subseteq \mathcal{X}$ of $x$ and $\mathcal{Y}_0 \subseteq \mathcal{Y}$ of $y$ such that for all $x' \in \mathcal{X}_0$ the intersection $\mathcal{Y}_0 \cap \operatorname{ker} \mathbf{L}(x')$ is a singleton.  Further, the map $f: \mathcal{X}_0 \rightarrow \mathcal{Y}_0$ sending $x'$ to $\mathcal{Y}_0 \cap \operatorname{ker} \mathbf{L}(x')$ is smooth.
\end{lemma}
\begin{proof}
This follows from the implicit function theorem \cite[Thm.~C.40]{lee2013smooth}.  Also see \cite{burgisser2017condition}. 
\end{proof}

\begin{proposition}\label{prop:world-lifting}
Assume $\mathscr{M}$ and $\mathbf{L}$ satisfy the following 
\textbf{smooth lifting property}: 
 \textit{For all $x \in \mathcal{X} \setminus {\operatorname{disc}}(\mathscr{M}, \mathbf{L})$, 
 $y \in \mathcal{Y} \cap \operatorname{ker} \mathbf{L}(x)$ and 
 $w \in \Phi^{-1}(x) \cap \Psi^{-1}(y)$, we can take 
 $\mathcal{X}_0$ and $\mathcal{Y}_0$ in Lemma~\ref{lem:local-disc} such that there further exists a 
 smooth map $h : \mathcal{X}_0 \rightarrow \mathcal{W}$ such that $h(x) = w$, $\Psi \circ h = f$ and $\Phi \circ h = \operatorname{id}_{\mathcal{X}_0}$.}  
Then it holds
\begin{equation}\label{eq:ill-in-disc}
\operatorname{ill}(\mathscr{M}, \mathcal{X}) \, \subseteq \, \operatorname{disc}(\mathscr{M}, \mathbf{L}).
\end{equation}
\end{proposition}
Proposition~\ref{prop:world-lifting} is proven in the appendices.

\smallskip

Putting Theorem~\ref{thm:hurwitz} and Proposition~\ref{prop:world-lifting} together, 
\begin{equation*}
x \in \operatorname{ill}(\mathscr{M}, \mathcal{X}) \, \implies \, \mathbf{P}_{\mathcal{Y}}(\operatorname{ker} \mathbf{L}(x)) = 0
\end{equation*}
if the property in Proposition~\ref{prop:world-lifting} holds.
Importantly, we can precompute an approximation to $\operatorname{cond}(\mathscr{M}, x)$, by estimating 
the reciprocal distance of $x$ to the zero set of $\mathbf{P}_{\mathcal{Y}}(\operatorname{ker} \mathbf{L}(\cdot))$.

\section{Main Theoretical Results} \label{sec:main-results}
We demonstrate the power of the general framework in Section~\ref{sec:framework} by performing stability analyses of the $5$- and $7$-point minimal problems from Section~\ref{sec:examples}.
The conclusions for relative pose estimation are all, to our knowledge, new.

\subsection{Condition number formulas}\label{sec:conditionNumsFormulas}

It is a straightforward computation to obtain explicit condition number formulas.  Metrics on the input and output spaces must be chosen in order to quantify error amplification; we choose the Euclidean distance on the image data space $\mathcal{X}$, and the metric inherited from $\mathbb{P}(\mathbb{R}^{3 \times 3})$ on \nolinebreak $\mathcal{Y}$.

\begin{proposition}
\label{prop:formula-E}
Let $\mathscr{M}$ be the $5$-point problem from Section~\ref{example:essential}.  
Let $w = (\mathbf{R}, \hat{\mathbf{T}}, \Gamma_1, \ldots, \Gamma_5)$ be a world scene, expressed using the coordinates in Lemma~\ref{lem:double}.
If it is finite, then $\operatorname{cond}(\mathscr{M}, w)$ equals the largest singular value of an explicit $5 \times 20$ matrix.  This matrix naturally factors as a $5 \times 20$ matrix multiplied by a $20 \times 20$ matrix.
\end{proposition}

\smallskip

\noindent {\textbf{Sketch:}}   
By Definition~\ref{def:cond-w}, if it is finite then 
\begin{equation}\label{eq:proof-cond1}
\operatorname{cond}(\mathscr{M}, w) = \sigma_{\max}(D \Psi(w) \circ D \Phi(w)^{-1})
\end{equation}
where $\Phi$ and $\Psi$ are the forward and output maps given by \eqref{eq:E-forward} and \eqref{eq:E-psi-map}.  
We factor the forward map as
\begin{equation*}
\Phi = \Phi_2 \circ \Phi_1,
\end{equation*}
 where  
$\Phi_1 : \mathcal{U} \subseteq \operatorname{SO}(3) \times \mathbb{S}^2 \times (\mathbb{R}^3)^{\times 5} \rightarrow (\mathbb{R}^{3} \times \mathbb{R}^3)^{\times 5}$ 
 is 
\begin{align} 
    &\Phi_1(\R, \hat{\mathbf{T}}, \Gamma_1, \ldots, \Gamma_5) = \nonumber \\  
    &((\Gamma_1, \R\Gamma_1 + \hat{\mathbf{T}}), \ldots, (\Gamma_5, \R\Gamma_5 + \hat{\mathbf{T}})), \label{eq:E-Phi1-main}
\end{align}
and $\Phi_2: (\mathbb{R}^{3} \times \mathbb{R}^3)^{\times 5} \dashrightarrow (\mathbb{R}^2 \times \mathbb{R}^2)^{\times 5}$ is 
\begin{align} 
&\Phi_2((\Gamma_1,\bar{\Gamma}_1), \ldots, (\Gamma_5, \bar{\Gamma}_5)) = \nonumber \\  &((\pi(\Gamma_1), \pi(\bar{\Gamma}_1)), \ldots, (\pi(\Gamma_5),\pi(\bar{\Gamma}_5))) \label{eq:E-Phi2-main}
\end{align}
where the broken arrow indicates a rational map and $\pi$ is perspective projection. 
By the chain rule, 
\begin{equation}\label{eq:E-cond2}
D \Phi = D \Phi_2 \circ D \Phi_1.
\end{equation}
Here $\Phi_1$ is multi-affine in $\Gamma_i,  \R, \hat{\mathbf{T}}$, while $\Phi_2 = (\pi \times \pi)^{\times 5}$, so they are easily differentiated.  
See the appendices for explicit expressions, 
with respect to an orthonormal basis of $T(\mathcal{X}, x)$. 
Meanwhile $\Psi$ is $(\R,\hat{\mathbf{T}}, \Gamma_1, \ldots, \Gamma_7) \mapsto [\hat{\mathbf{T}}]_{\times} \R$.
See the appendices for its Jacobian matrix with respect to an orthonormal basis for $T(\mathcal{Y},y)$.
Inserting these into \eqref{eq:proof-cond1} finishes the computation. \hfill $\square$

\medskip
\smallskip

\begin{proposition}
\label{prop:formula-F}
Let $\mathscr{M}$ be the $7$-point problem from Section~\ref{example:fundamental}.  
Let $w = (b, \Gamma_1, \ldots, \Gamma_7)$ be a world scene, expressed using the coordinates in Lemma~\ref{lem:b-coords}.  
If it is finite,  then $\operatorname{cond}(\mathscr{M}, w)$ equals the largest singular value of an explicit $7 \times 28$ matrix.  This matrix naturally factors as a $7 \times 28$ matrix multiplied by a $28 \times 28$ matrix.
\end{proposition}

\noindent {\textbf{Sketch:}} 
If the condition number is finite, then 
\begin{equation}
\operatorname{cond}(\mathscr{M}, w) = \sigma_{\max}(D \Psi(w) \circ D \Phi(w)^{-1})
\end{equation}
where $\Phi$ and $\Psi$ are given by \eqref{eq:F-forward} and \eqref{eq:fund-output}.  
We factor 
\begin{equation*}
\Phi = \Phi_2 \circ \Phi_1,
\end{equation*}
 where $\Phi_1 : \mathcal{U} \subseteq \mathbb{R}^7 \times (\mathbb{R}^3)^{\times 7} \rightarrow ({\mathbb{R}}^{3} \times {\mathbb{R}}^3)^{\times 7}$ is 
\begin{align}\label{eq:Phi1-7pt}
    &\Phi_1(b, \Gamma_1, \ldots, \Gamma_7) =  
    (({\Gamma}_1, \mathbf{M}(b) \tilde{\Gamma}_1), \ldots, ({\Gamma}_7, \mathbf{M}(b)\tilde{\Gamma}_7)),
\end{align} 
and $\Phi_2: ({\mathbb{R}}^{3} \times {\mathbb{R}}^3)^{\times 7} \dashrightarrow (\mathbb{R}^2 \times \mathbb{R}^2)^{\times 7}$ is 
\begin{align}\label{eq:Phi2-7pt}
&\Phi_2((\Gamma_1,\bar{\Gamma}_1), \ldots, (\Gamma_7, \bar{\Gamma}_7)) = \\ \nonumber &((\pi(\Gamma_1), \pi(\bar{\Gamma}_1)), \ldots, (\pi(\Gamma_7),\pi(\bar{\Gamma}_7))).
\end{align}
Using the chain rule and an orthonormal basis for $T(\mathcal{X}, x)$, we compute $D \Phi$ in the appendices. 
We also explicitly compute the Jacobian $D\Psi$ with respect to an orthonormal basis for $T(\mathcal{Y}, y)$. \hfill $\square$

\subsection{Ill-posed world scenes}\label{sec:ill-pose-world}
In this section, we derive geometric characterizations of ill-posed world scenes.
The results are in terms of the existence of certain quadric surfaces in $\mathbb{R}^3$.  
The proofs are involved.

\begin{figure}[h]
    \centering
    (a)\includegraphics[height=0.325\linewidth]{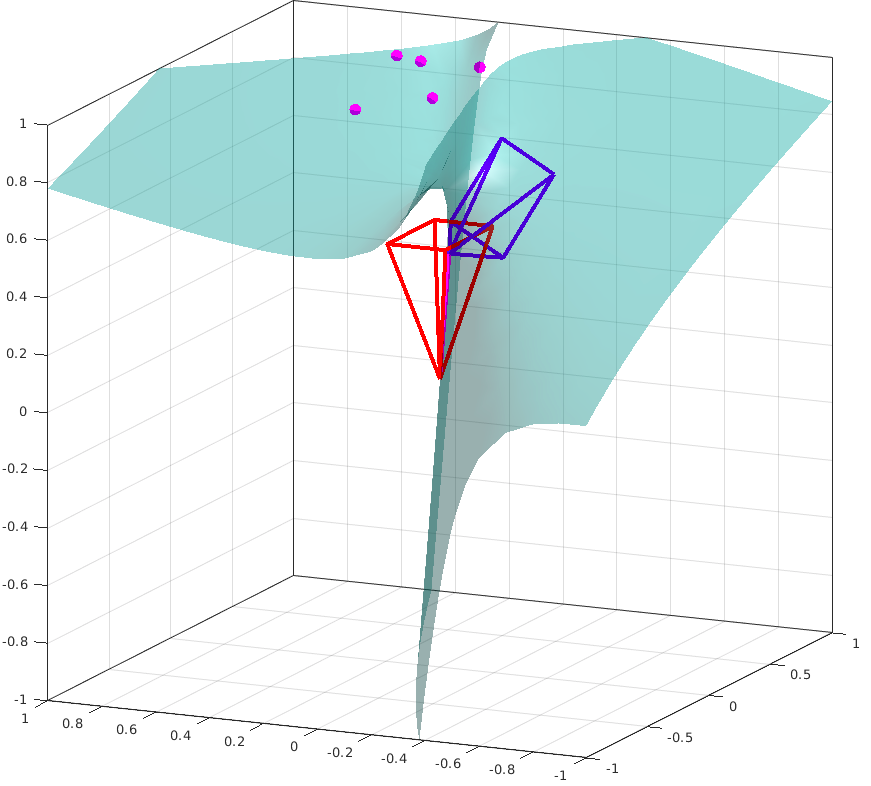}
    (b)\includegraphics[height=0.37\linewidth]{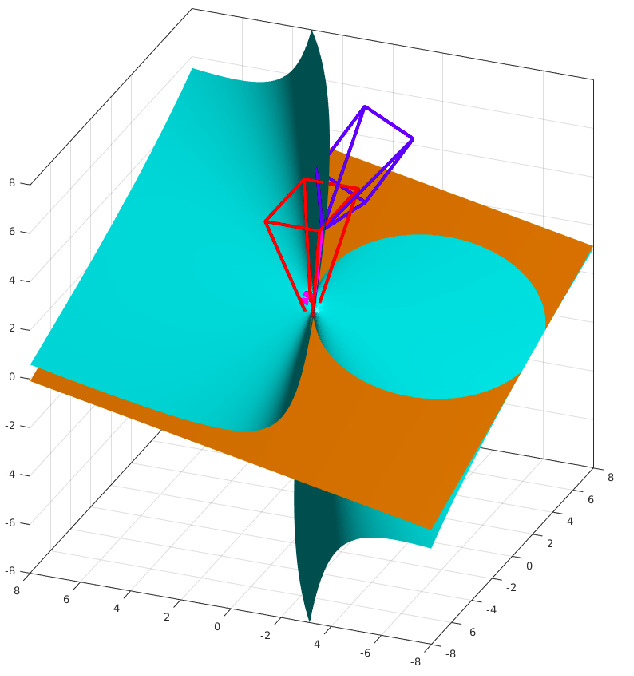}
    (c)\includegraphics[height=0.25\linewidth]{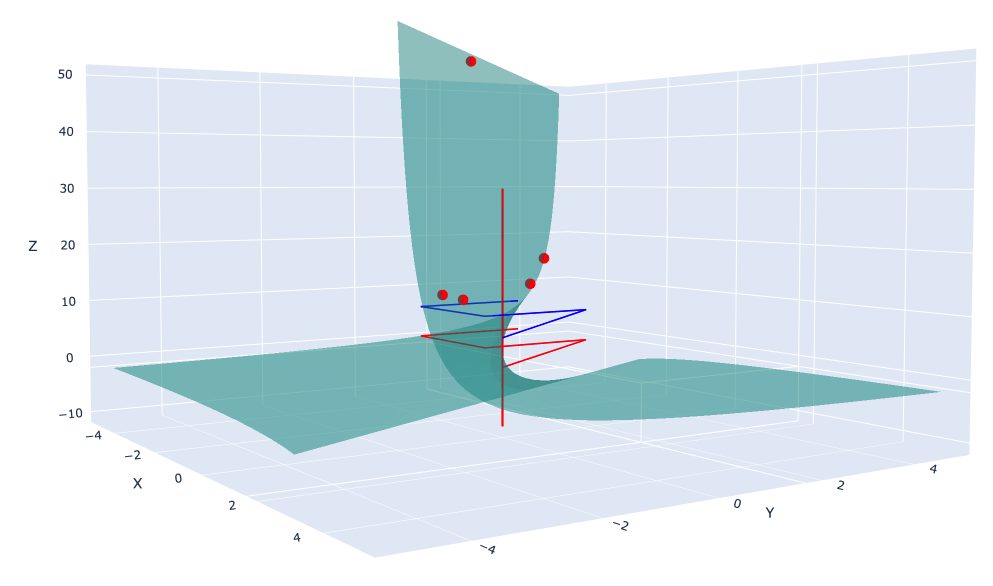}
    (d)\includegraphics[height=0.25\linewidth]{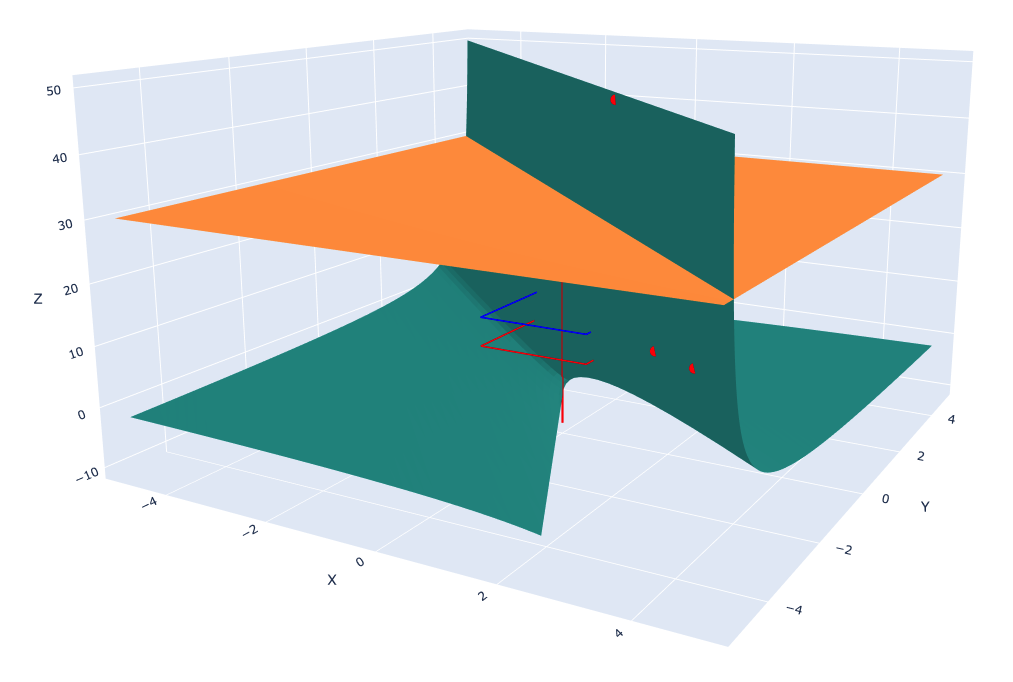}
    \caption{(a) An ill-posed world scene in the calibrated case. Red and blue pyramids represent two cameras. Magenta points represent five world points. The green surface is a quadric surface satisfying the conditions of Theorem~\ref{thm:illposed-world}.  The view (b) shows an orange plane perpendicular to the baseline whose intersection with the quadric surface is a circle. (c) Another ill-posed world scene for the $5$-point problem, in the case where a plane perpendicular to the baseline intersect with the quadric surface in a line. View (d) shows the orange plane perpendicular to the baseline.} 
    \label{fig:surface}
\end{figure}

\begin{theorem}
\label{thm:illposed-world}
Let $\mathscr{M}$ be the $5$-point problem from Section~\ref{example:essential}.
Let $w = (\R, \hat{\mathbf{T}}, \Gamma_1, \ldots, \Gamma_5) \in \mathcal{W}$ be a world scene, expressed using the coordinates in Lemma~\ref{lem:double}. 
Then $w \in \operatorname{ill}(\mathscr{M}, \mathcal{W})$  
if and only if there exists a quadric surface $\mathcal{Q}$ in  $\mathbb{R}^3$ such \nolinebreak that:
\begin{itemize}\setlength\itemsep{0.35em}
    \item $\Gamma_i \in \mathcal{Q}$ for each $i =1 , \ldots, 5$; 
    \item $\ell \subseteq \mathcal{Q}$ where $\ell = \operatorname{span} \, (\mathbf{R}^{\top} \hat{\mathbf{T}})$ is the baseline through the two camera centers; and
    \item $\mathcal{N} \cap \mathcal{Q}$ is a circle for each affine plane $\mathcal{N}$ in $\mathbb{R}^3$ normal to $\ell$ (possibly of radius $0$), or $\mathcal{N} \cap \mathcal{Q}$ is a line for each affine plane $\mathcal{N}$ in $\mathbb{R}^3$ normal to $\ell$ (with at most one exception).
\end{itemize}
\end{theorem}
\noindent {\textbf{Sketch:}} 
Let $\delta w = (\delta \mathbf{R}, \delta \hat{\mathbf{T}}, \delta \Gamma_1, \ldots, \delta \Gamma_5) \in T(\mathcal{W}, w)$ stand for a tangent vector to $\mathcal{W}$ at $w$.  
We need to characterize the world scenes $w$ such that 
\begin{equation} \label{eq:E-linearsys}
D\Phi(w)(\delta w) = 0
\end{equation}
admits a nonzero solution in $\delta w$.  
In the appendices, we show there exists an invertible linear change of coordinates in \eqref{eq:E-linearsys}, so that without loss of generality we may assume $\mathbf{R} = I$ and $\hat{\mathbf{T}} = \mathbf{e}_3$.

Recall $D \Phi = D \Phi_2 \circ D \Phi_1$ as in \eqref{eq:E-cond2}.  
Here the second map is $\Phi_2 = (\pi \times \pi)^{\times 5}$, its Jacobian is $D \Phi_2 = (D\pi \oplus D\pi)^{\oplus 5}$, and its kernel is $\operatorname{ker}(D\Phi_2)((\Gamma_1, \bar{\Gamma}_1), \ldots, (\Gamma_5, \bar{\Gamma}_5)) = \mathbb{R} \Gamma_1 \oplus \ldots \oplus \mathbb{R} \bar{\Gamma}_5$.  
Additionally, $D \Phi_1$ takes the explicit form in Proposition~\ref{prop:formula-E}.
In the appendices, we show that \eqref{eq:E-linearsys} can be reduced to the $5 \times 5$ linear system:
\begin{multline}
 a ((\Gamma_i)_1^2 \, + \, (\Gamma_i)_2^2) \, + \, b (\Gamma_i)_1 (\Gamma_i)_3 
\, + \, c (\Gamma_i)_2 (\Gamma_i)_3  \\[0.5pt]
 + \, d  (\Gamma_i)_1 \, + \, e (\Gamma_i)_2 \,\, = \,\, 0   \label{eq:my-reduced-linear}
\end{multline}
for $i=1, \ldots, 5$, in  the sense that \eqref{eq:my-reduced-linear} has a nonzero solution in $a, \ldots, e$   
if and only \eqref{eq:E-linearsys} has a nonzero solution in $\delta w$.  To finish, we note that there is a nonzero solution to \eqref{eq:my-reduced-linear} if and only if there exists  $\mathcal{Q}$ passing through $\Gamma_1, \ldots, \Gamma_5$ as announced, with $a \neq 0$ corresponding to circular cross sections in the third bullet of the theorem and $a = 0$ corresponding to linear cross sections.
\hfill $\square$

See Fig.~\ref{fig:surface} for visualizations of  \nolinebreak Theorem~\ref{thm:illposed-world}.  
Note that the second bullet in Theorem~\ref{thm:illposed-world} implies $\mathcal{Q}$ is ruled, \ie covered by lines. 
Further, \eqref{eq:my-reduced-linear} implies that $\mathcal{Q}$ is a ``rectangular quadric" as in \cite[Thm.~3.13]{maybank2012theory}

\begin{theorem} 
\label{thm:F-ill-posed-world}
Let $\mathscr{M}$ be the $7$-point problem from Section~\ref{example:fundamental}.
 Let $w = (b, \Gamma_1, \ldots, \Gamma_7) \in \mathcal{W}$, expressed using the coordinates in Lemma~\ref{lem:b-coords}.
Then $w \in \operatorname{ill}(\mathscr{M}, \mathcal{W})$ if and only if there exists a quadric surface $\mathcal{Q}$ in  $\mathbb{R}^3$ such that:
\begin{itemize}\setlength\itemsep{0.35em}
    \item $\Gamma_i \in \mathcal{Q}$ for each $i = 1, \ldots, 7$; and
    \item $\ell \subseteq \mathcal{Q}$ where $\ell$ is the baseline through the two camera centers.
\end{itemize}
\end{theorem}
\noindent {\textbf{Sketch:}} 
Let $\delta w = (\delta b, \delta \Gamma_1, \ldots, \delta \Gamma_7) \in T(\mathcal{W}, w)$ be a tangent vector.  
Again we wish to characterize when 
\begin{equation}\label{eq:solvability-DF}
D\Phi(w)(\delta w) = 0
\end{equation}
has a nonzero solution in $\delta w$.  
We use $D \Phi = D \Phi_2 \circ D \Phi_1$, where $D\Phi_2 = (D \pi \oplus D \pi)^{\oplus 7}$ and $D\Phi_1$ is as in the proof of Proposition~\ref{prop:formula-F}.  
Eliminating variables, we reduce \eqref{eq:solvability-DF} to a particular $7 \times 7$ linear system.  Its solvability is equivalent to there existing a quadric surface $\mathcal{Q}$ as described.  See details in the appendices.
\hfill $\square$

\subsection{Ill-posed image data}
\label{sec:ImageData}

We describe ill-posed image data for the $5$- and $7$-point problems.

\begin{theorem}
\label{thm:image-E}
Let $\mathscr{M}$ be the $5$-point problem from Section~\ref{example:essential} 
Let $x = \left((\gamma_1, \bar{\gamma}_1), \ldots, (\gamma_5,\bar{\gamma}_5)\right) \in (\mathbb{R}^2 \times \mathbb{R}^2)^{\times 5}$ be image data.
Then $x \in \operatorname{ill}(\mathscr{M}, \mathcal{X})$  
only if a certain polynomial $\mathbf{P}$ in the entries of $\gamma_1, \ldots, \bar{\gamma}_5$ vanishes.  
Further, $\mathbf{P}$ has degree $30$ separately in each of $\gamma_1, \ldots, \bar{\gamma}_5$. 
\end{theorem}

\begin{remark}\label{rem:4.5-curve}
An explicit form for $\mathbf{P}$ is currently unavailable.  Nonetheless, we can still compute with \textit{specializations} of $\mathbf{P}$.
Specifically, suppose we fix $\gamma_1, \bar{\gamma}_1, \ldots, \gamma_5$ but keep $\bar{\gamma}_5 \in \mathbb{R}^2$ as a variable.   Then generically,  $\mathbf{P}$ is a degree-$30$ polynomial in just $\bar{\gamma}_5$.  
Its zero set is a plane curve.
We are able to plot such curves numerically; see  Section~\ref{sec:compute-curve}.
\end{remark}

\noindent \noindent {\textbf{Sketch:}} 
As noted in Example~\ref{ex:mathL}, the $5$-point problem has an epipolar map $\mathbf{L}$ in the sense of  Definition~\ref{def:epipolar-map} with $n=8$ and $d=5$.
Therefore Theorem~\ref{thm:hurwitz} applies. 
It implies the existence of a nonzero polynomial $\mathbf{P}$ in the Pl\"ucker coordinates for codimension-$5$ linear subspaces of $\mathbb{P}_{\mathbb{R}}^8$ such that \eqref{eq:Py} holds.  
Here 
$\mathcal{Y}_{\mathbb{C}} = \{ E \in \mathbb{P}(\mathbb{C}^{3 \times 3}) : 2 EE^{\top}E - \operatorname{tr}(EE^{\top})E = 0, \det(E) = 0 \}$.
By \cite[Ex.~2.5]{sturmfels2017hurwitz}, this variety has degree $p = 10$ and sectional genus $g = 6$.
By \cite[Prop.~2]{floystad2018chow}, the singular locus of $\mathcal{Y}_{\mathbb{C}}$ has codimension $3$.  
So, Theorem~\ref{thm:hurwitz} gives the degree of $\mathbf{P}$ as $2 \cdot 10 + 2 \cdot 6 - 2 = 30$ in Pl\"ucker coordinates. 
From the constraints in Example~\ref{ex:mathL}, each maximal minor of $\mathbf{L}(x)$ is multi-affine in $\gamma_1,  \ldots, \bar{\gamma}_5$.  
So, $\mathbf{P}$ has degree $30$ separately in each of $\gamma_1, \ldots, \bar{\gamma}_5$.

To conclude, use Proposition~\ref{prop:world-lifting}.  It suffices to verify the property there, which states that given image data $x$ outside of the discriminant, a compatible world scene, and an essential matrix, we can smoothly locally lift the essential matrix and $x$ to the world scene as $x$ moves.   This is justified in the appendices by \cite[Res.~9.19]{hartleyzisserman} and \cite{tron2017space}.   \nolinebreak \hfill $\square$

\medskip

A similar characterization holds in the uncalibrated case, except it is more explicit.

\begin{theorem}
\label{thm:F-image}
Let $\mathscr{M}$ be the $7$-point problem from Section~\ref{example:fundamental}. 
Let $x=\left((\gamma_1, \bar{\gamma}_1), \ldots, (\gamma_7,\bar{\gamma}_7)\right) \in (\mathbb{R}^2 \times \mathbb{R}^2)^{\times 7}$ be image data.
Then $x \in \operatorname{ill}(\mathscr{M}, \mathcal{X})$ 
only if a certain polynomial $\mathbf{P}$ in the entries of $\gamma_1, \ldots, \bar{\gamma}_7$ vanishes.  
The polynomial has degree $6$ separately in each of $\gamma_1, \ldots, \bar{\gamma}_7$.   
In fact, $\mathbf{P}$ can be expressed as an explicit degree-$6$ polynomial in the $36$ Pl\"ucker coordinates for $2$-dimensional subspaces of $\mathbb{R}^9$, with $1668$ monomial terms and all coefficients integers at most $72$ in absolute value. 
\end{theorem}

\begin{remark}\label{rem:6.5-curve}
Suppose we fix $\gamma_1, \bar{\gamma}_1, \ldots, \gamma_7$, but keep $\bar{\gamma}_7$ as a variable.   Then generically,  $\mathbf{P}$ is a degree-$6$ polynomial in just $\bar{\gamma}_7$.  
Its zero set is a plane curve.
In Section~\ref{sec:compute-curve} we compute these curves to visualize $\operatorname{ill}(\mathscr{M}, \mathcal{X})$.
\end{remark}

\noindent {\textbf{Sketch:}} 
The $7$-point problem has an epipolar map, with $n=8$ and $d=7$.
By Theorem~\ref{thm:hurwitz}, there is a nonzero polynomial $\mathbf{P}$ in the $\binom{9}{7} = 36$ Pl\"ucker coordinates for lines in $\mathbb{P}_{\mathbb{R}}^8$ that obeys \eqref{eq:Py}.
Here $\mathcal{Y}_{\mathbb{C}} = \{F \in \mathbb{P}(\mathbb{C}^{3 \times 3}) : \det(F) = 0\}$ has degree $p = 3$ and sectional genus $g=1$.   
So $\mathbf{P}$ has degree $2 \cdot 3 + 2 \cdot 1 - 2 = 6$ in Pl\"ucker coordinates.
As in Theorem~\ref{thm:image-E}, this implies that $\mathbf{P}$ has degree $6$ separately in each of $\gamma_1, \ldots, \bar{\gamma}_7$ as wanted. 
In Section~\ref{sec:compute-curve} we compute $\mathbf{P}$ explicitly, using the formula for the discriminant of a univariate cubic polynomial.  
In the appendices, we verify the smooth lifting property holds using \cite[Res.~9.14]{hartleyzisserman}. \hfill $\square$

\section{Relation Between Ill-Posed ness,  Degeneracy and Criticality }\label{sec:relationship}

In this section we compare our new notions to thoroughly-studied ones in the literature. 

Previous works in multiview geometry have considered degenerate image data and critical world scenes; see Definition~\ref{def:degen} and Definition~\ref{def:critical}. These notions 
fundamentally concern \textit{non-uniqueness}. Studying criticality and degeneracy,  researchers have classified scenarios when there are multiple solutions to the 3D reconstruction problems. 
This line of work is primarily about the super-minimal case.  For example, it answers questions of the following sort: For which scenes of two uncalibrated cameras and $50$ world points does there exist an inequivalent world scene mapping to the same image data?
In particular, degeneracy and criticality seem unhelpful for minimal problems: minimal problems usually have multiple solutions. 

Our concepts of ill-posed world scenes and image data in Definitions~\ref{def:ill-W} and \ref{def:ill-X} are different. They pertain to minimal problems.  Furthermore, instead of non-uniqueness, ill-posedness captures when one (of possibly multiple) solutions is unstable.  Precisely, the condition number is infinite.  

Prior work in two-view geometry classified critical world scenes in terms of quadric surfaces; see \cite{krames1941ermittlung} and \cite{kahl2002critical}. 
In the uncalibrated case, there is a critical ruled quadric in the literature, like ours in Theorem~\ref{thm:F-ill-posed-world}, except it is only required to contain the world points and camera centers and not the baseline.  
The calibrated case is similar, with ``ruled quadric" replaced by ``rectangular quadric".  
Our next result helps clarify the relationship between concepts. 

\begin{proposition}\label{prop:crit-ill} 
Consider a super-minimal world scene $w'$, consisting of two
calibrated cameras and $N$ world points where $N \geq 5$.
Assume that for each subset of $5$ world points, the resulting
minimal world scene $w$ is ill-posed. Then 
$w'$ is critical. The same statement holds with ``calibrated"
and ``5" replaced by ``uncalibrated" and ``7". 

\end{proposition}
\noindent {\textbf{Sketch:}}
For each minimal subscene $w$, there exists a quadric surface satisfying the conditions in Theorem~\ref{thm:illposed-world}.  
In the appendices, we view the conditions as linear in the quadric and prove that there must exist a common quadric $\mathcal{Q}$ which works for each $w$.  In particular, $\mathcal{Q}$ passes through the two camera centers and $N$ world points of $w'$.  
So $w'$ is critical by \cite{kahl2002critical}.  
The uncalibrated case is the same.
\hfill $\square$

\medskip

Proposition~\ref{prop:crit-ill} is tight: if there exists \textit{one} minimal subscene $w$ that is well-posed, then $w'$ can be nondegenerate.  We also stress that the converse of Proposition~\ref{prop:crit-ill} fails.  
Indeed if $w'$ is critical, then \textit{all} minimal subscenes $w$ could be well-posed, as illustrated below.

\begin{figure}[h]
    \centering
    \includegraphics[width=0.6\linewidth]{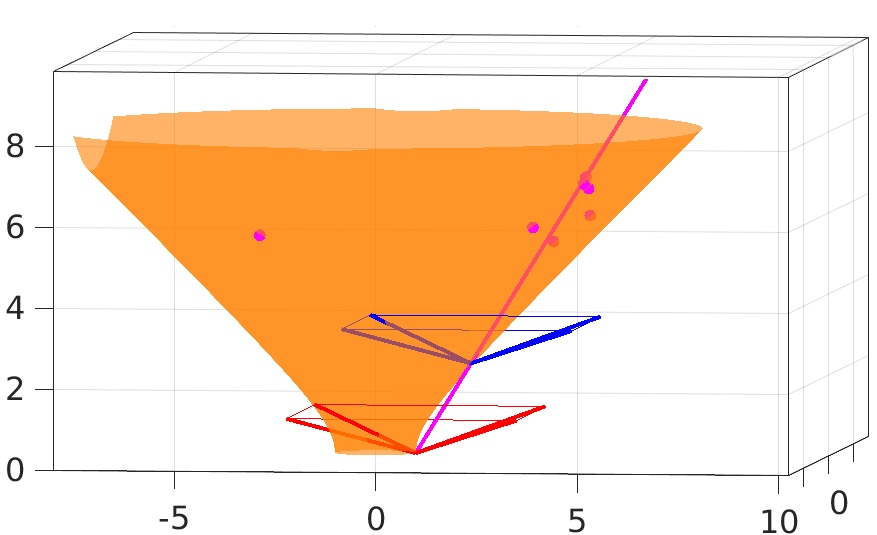}
    \caption[]{Example of an uncalibrated scene that is critical, in which all minimal subscenes are well-posed. 
    Fix a ruled quadric surface (orange surface) as shown.  Any $N$ points and $2$ camera centers (magenta points) drawn from the quadric will form a critical configuration by \cite{krames1941ermittlung}.  
    However generically, the baseline (magenta line) will not lie on the surface.}
    \label{fig:numerical_uncalib}
\end{figure}

\section{Computation of X.5-Point Curves}\label{sec:compute-curve}

We explain how to compute the curves in Remarks~\ref{rem:4.5-curve} and \ref{rem:6.5-curve}, dubbed ``$4.5$-point" and ``$6.5$-point" curves for short. 
The curves represent the ill-posed loci in image data space.

\subsection{Numerical computation of X.5-point curves}
\label{sec:compute_x5_numerical}
There is a numerical method to render the X.5-point curves in both the calibrated and uncalibrated cases. It is based on homotopy continuation \cite{breiding2018homotopycontinuation}.
We describe the procedure only for the calibrated case, as the uncalibrated case is analogous.  

To generate the $4.5$-point curve, four image point pairs $(\gamma_i,\bar{\gamma}_i)$ for $i=1,\ldots,4$
and a stand-alone point $\gamma_5$ are fixed. 
Using SVD, we compute the null space of the linear system $(\bar{\gamma}_i; 1)^{\top} E (\gamma_i; 1) = 0$ for $i=1, \ldots, 4$.  If it is more than $5$-dimensional as in \cite{connelly2023geometry, connelly2024geometryp2}, then all choices of $\bar{\gamma}_5$ give a point in the ill-posed locus.
Otherwise, similarly to \cite{nister:PAMI:2004} 
we parameterize (an affine slice of) the null space as
\begin{align}
    E = E(\alpha) :=  \alpha_1 E_1 + \alpha_2 E_2 + \alpha_3 E_3 + \alpha_4 E_4 + E_5,
    \label{Eq:nullspaceofE}
\end{align}
where $E_i$ are a basis of the null space and $\alpha_i$ are unknown weights. 
The essential matrix should also satisfy 
 \begin{align} \label{eq:P1P2-constraint}
  \left\{ \begin{matrix}
 \det(E) = 0 \\
 \, 2EE^{\top}\!E - \operatorname{tr}(EE^{\top}\!)E  = 0, 
 \end{matrix}
 \right.
 \end{align}
which are $10$ cubic equations. 
Further, the fifth correspondence should satisfy the epipolar constraint
\begin{align}\label{eq:P3-constraint} 
   (\bar{\gamma}_5; 1)^{\top} E (\gamma_5; 1) = 0.
\end{align}
At ill-posed configurations, the Jacobian matrix of the constraints \eqref{eq:P1P2-constraint} and \eqref{eq:P3-constraint} with respect to $\alpha_i$ should be rank-deficient. 
The Jacobian $J$ is a $11 \times 4$ matrix. 
To enforce its rank-deficiency, we introduce $3$ dummy variables $d_1, d_2, d_3 \in \mathbb{R}$ and require
\begin{equation}\label{eq:Jac-drop}
J \, (d_1 , d_2 , d_3 , 1)^{\top} = \, 0.
\end{equation} 
\noindent Note that there are $22$ equations from \eqref{eq:P1P2-constraint}, \eqref{eq:P3-constraint}, \eqref{eq:Jac-drop} and altogether $9$ unknowns, namely $\alpha_1$, $\alpha_2$, $\alpha_3$, $\alpha_4$, $(\bar{\gamma}_5)_1$, $(\bar{\gamma}_5)_2$, $d_1$, $d_2$, $d_3$. 
The solutions to this polynomial system, projected to the $\bar{\gamma}_5$-plane, define the $4.5$-point curve. 

By setting $(\bar{\gamma}_5)_1$ to various different values, we compute the zero-dimensional solution sets as in Fig.~\ref{fig:scanning}. 
The real solutions for $(\bar{\gamma}_5)_2$ correspond to the intersection between the $4.5$-point curve and a column of the image. 
The solutions to these systems are computed using parameter homotopy~\cite{breiding2018homotopycontinuation}.
By Theorem~\ref{thm:image-E}, the systems have $30$ complex solutions, so there are at most $30$ real intersections with the various columns of the second image.  
Linearly connecting the intersection points, the $4.5$-point curve is rendered on the image.

In applications we may wish to estimate the condition number by approximating \eqref{eq:proxy} as the reciprocal distance from a target point to the curve.  Then a full plot of the curve is not required. 
We can simply compute the intersection points as $(\bar{\gamma}_5)_1$ ranges over a small interval around the correspondence candidate, see Fig.~\ref{fig:scanning}(b). 

Finally, computations for different columns of the image are independent.  So the procedure is easily parallelized.

\begin{figure}[h]
    \centering
    (a)\includegraphics[width=0.3\linewidth]{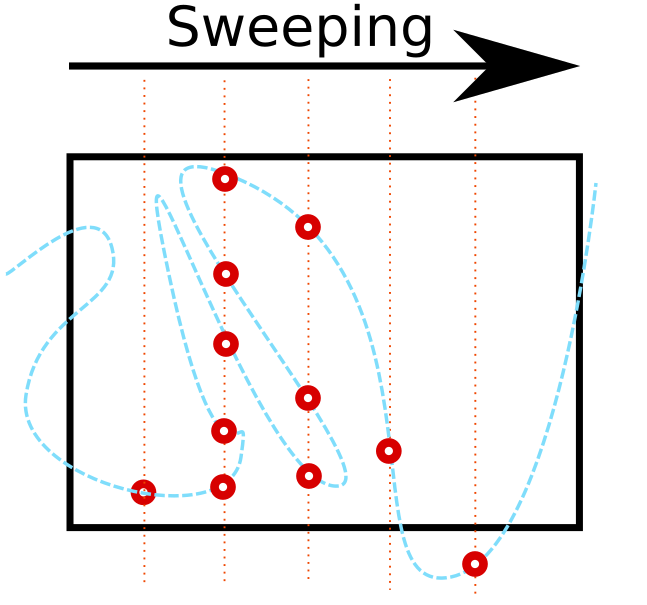}
    (b)\includegraphics[width=0.3\linewidth]{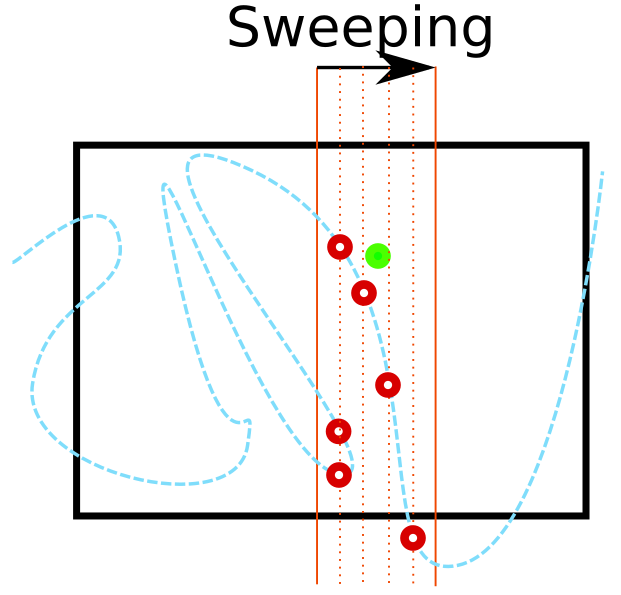} \\
    \vspace{0.1cm}
    \includegraphics[width=0.8\linewidth]{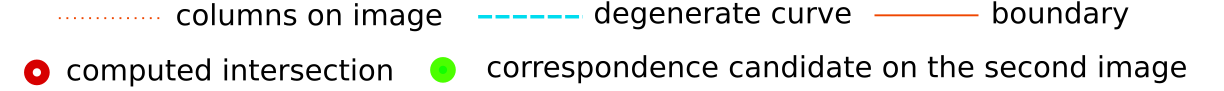}
    \caption{(a) To generate  X.5-point curves on the second image, we can sweep the image columnwise and compute the intersection with vertical lines by solving \eqref{eq:P1P2-constraint}, \eqref{eq:P3-constraint}, \eqref{eq:Jac-drop}. (b) To assess a candidate correspondence, we can  scan just a neighborhood around the candidate point.}
    \label{fig:scanning}
\end{figure}

\subsection{Symbolic computation of 6.5-point curves} \label{sec:symbolic-curve}
The numerical method in Section~\ref{sec:compute_x5_numerical} can generate the $4.5$-point and $6.5$-point curves with high precision. 
However, given the iterative nature of homotopy continuation, it is slow in finding the distance from the point to the curve. 
Engineering with GPUs has accelerated homotopy continuation path-tracking methods as much as $100 \times$ (see \cite{chien2022gpu,verschelde2016polynomial,chien2022parallel}), but the speedup factor varies across problems. 
Alternatively, for the case of $6.5$-point curves, we can symbolically compute the defining equation of the curve with high speed.  This is based on the last sentence of Theorem~\ref{thm:F-image}. 

First, we derive the expression for $\mathbf{P}$ in Theorem~\ref{thm:F-image}.
Consider the $7 \times 9$ matrix encoding the epipolar constraints: 
\begin{equation}\label{eq:L-7}
\mathbf{L}((\gamma_1, \bar{\gamma}_1), \ldots, (\gamma_7, \bar{\gamma}_7))
\end{equation}
from Example~\ref{ex:mathL}.  Let its maximal minors be $p_{ij}$ for $1 \leq i < j \leq 9$, obtained by dropping the $i$th and $j$th columns.  
If these are all $0$, then \eqref{eq:L-7} is rank-deficient and $\mathbf{P}$ vanishes.  
Else, $p_{ij}$ are the Pl\"ucker coordinates for the row-span of \eqref{eq:L-7}.
Each $p_{ij}$ is a degree-$6$ polynomial in $14$ symbolic variables. 
Generically, a basis for the null space of \eqref{thm:F-image} is written as
\begin{align} \label{eq:F-basis-joe}
F_1 = \begin{pmatrix}
        0   &    1      & p_{13} \\
       p_{14} & p_{15} & p_{16} \\
       p_{17} &  p_{18} &  p_{19} 
    \end{pmatrix}\!, \,\, 
    F_2 = \begin{pmatrix}
      1 &       0  &      -p_{23} \\
       -p_{24}& -p_{25}& -p_{26} \\
       -p_{27}& -p_{28}& -p_{29} 
    \end{pmatrix}\!.
\end{align}
Let $F = F(t) = tF_1 + F_2$, where $t$ is an unknown weight.  
The fundamental matrix must satisfy the constraint
\begin{equation}\label{eq:det-F-2}
\det(F) = 0.
\end{equation}
At ill-posed image data, \eqref{eq:det-F-2} should have a double root with respect to $t$.  
So the discriminant of \eqref{eq:det-F-2} (a univariate cubic) vanishes.  
Substituting \eqref{eq:F-basis-joe} into the discriminant in \cite{M2} produces the wanted expression for $\mathbf{P}$ in terms of \nolinebreak $p_{ij}$.

As for the $6.5$-point curve, plug numerical values for $\gamma_1, \bar{\gamma}_1, \ldots, \gamma_7$ into \eqref{eq:L-7}.  
Taking maximal minors, compute the Pl\"ucker coordinates $p_{ij}$ as affine-linear polynomials in  $(\bar{\gamma}_7)_1, (\bar{\gamma}_7)_2$.  
Substitute these into the expression for $\mathbf{P}$ to get the defining equation for the curve.
With the defining equation in hand, we estimate the distance to a target point by the Sampson distance approximation \cite{hartleyzisserman}.
We optimize this procedure using Maple to minimize the number of arithmetical operations needed to compute the $6.5$-point curve in C++. 
In experiments below, we achieve an average speed of $200 \mu$s per $6.5$-point curve on a single CPU, which reaches real-time speeds.  
The computation could be further sped-up, by calculating different coefficients in \nolinebreak parallel.

\section{Experimental Results} \label{sec:experimentals}

This section presents numerical experiments.
Corresponding code may be found at \texttt{\url{https://github.com/HongyiFan/minimalInstability}}.

\subsection{Instability with only inliers present}\label{sec:only-inliers}
We first aim to demonstrate that instabilities do occur in practice for relative pose estimation problems.  We show 
epipolar estimation can fail even with no outliers in the data.

To this end, we generate $3000$ random instances of the $5$- and $7$-point minimal problems.  
First, generate random world scenes $\large{(}\R,\hat{\mathbf{T}},\Gamma_1,\ldots, \Gamma_N \large{)}$ with $N=7$ and an intrisinc matrix $K$ in the uncalibrated case, and otherwise with $N=5$ in the calibrated case, as follows:
\begin{itemize}
\setlength\itemsep{1.5pt}
    \item $\R$: taken from the QR decomposition of a random $3 \times 3$ matrix with i.i.d. standard normal entries;
    \item $\hat{\mathbf{T}}$: a uniformly-sampled point on the sphere with radius $1$ meter;
    \item $\Gamma_i$: uniformly-sampled world points with depths between $1$ and $20$ meters;
    \item $K$: chosen so the image size is $640 \times 480$, focal length is  $32$ mm and principle point is the image \nolinebreak center.
\end{itemize}
First, we calculate the projections of each data point from $\Gamma_i$ onto two images, resulting in pairs of clean image data, $(\gamma_i, \bar{\gamma}_i)$, where $i$ ranges from 1 to $N$. We eliminate and re-create any data instance if any 2D points are outside the image borders or any 3D points are behind the cameras. Finally, we add independent and identically distributed (i.i.d.) noise to each image point. This noise is taken from a 2D Gaussian distribution $\mathcal{N}(0, \sigma^2 I)$, with different levels of noise characterized by the standard deviation $\sigma$.

Next, we separately solve the clean and perturbed minimal problems using the standard solvers. 
Here we regard an instance as unstable if either: \textit{(i)} The error in the computed epipolar matrix for the perturbed data is larger than a pre-defined threshold $\tau$; or \textit{(ii)} there is a difference in the number of real solutions for the clean and perturbed data.  
In criterion \textit{(i)}, the error in the fundamental or essential matrix is taken as 
$e = \operatorname{mean}(\operatorname{abs}(\operatorname{abs}(\bar{M} ./ M) - \mathbf{1}\mathbf{1}^{\!\top}))$
where ``$./$" denotes elementwise division, 
``$\operatorname{abs}$" denotes entrywise absolute value, $\mathbf{1}$ is the all-ones vector, $M$ is the ground-truth model and $\bar{M}$ is the nearest estimated real model.  If $e \geq \tau$ for some threshold $\tau$, then the instance is counted as unstable.  Meanwhile, criterion \textit{(ii)} is included because the number of real solutions changes only if  the ill-posed locus is crossed \nolinebreak \cite{bernal2023machine}. 

Fig.~\ref{fig:Revelation} shows the fraction of unstable instances out of the $3000$ runs, at small to moderate noise levels and various error thresholds.  We stress that all image point pairs are inliers here.  Even so, it is clear that for random perturbations the occurrence of unstable cases cannot be ignored.

\begin{figure}[ht]
    \centering
    (a)\includegraphics[height = 0.45 \linewidth]{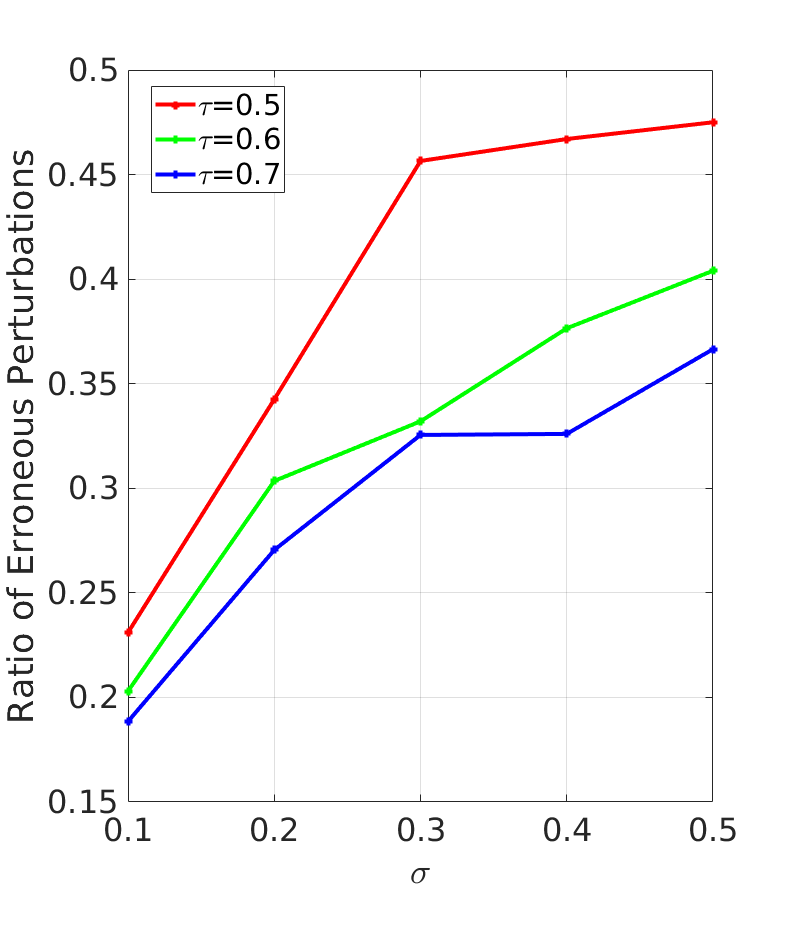}
    (b)\includegraphics[height = 0.45 \linewidth]{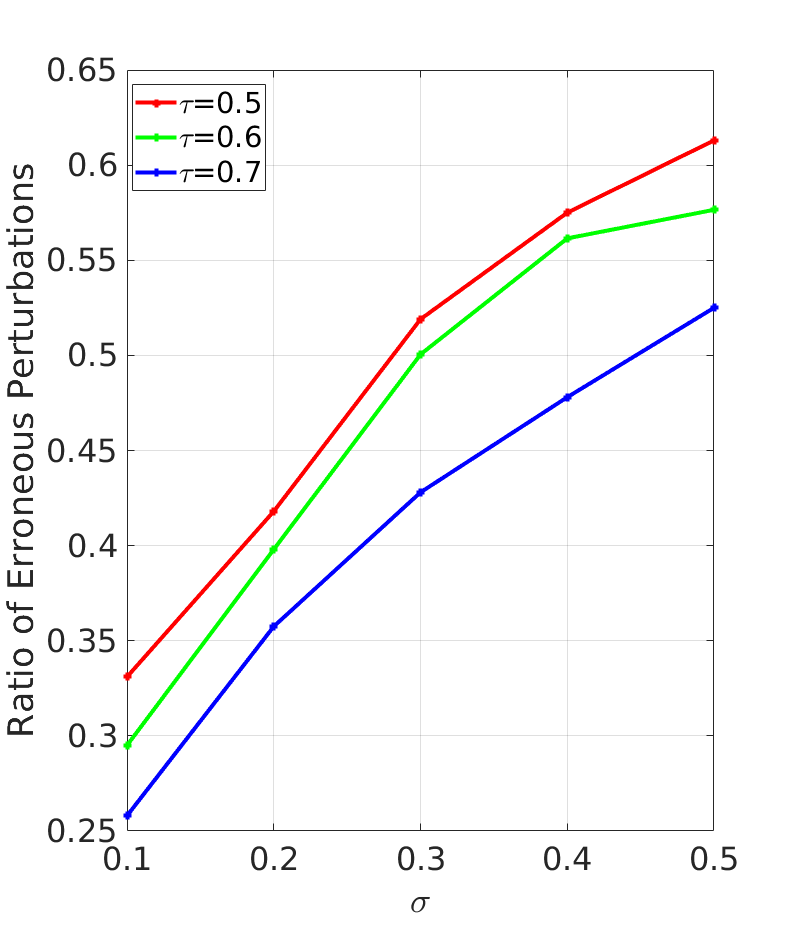}
    \caption{Ratio of 
    unstable instances
    out of $3000$ random synthetic minimal problem instances at different noise levels $\sigma$ and error thresholds $\tau$ for: (a) \textup{fundamental matrices}; and (b) \textup{essential matrices}.}
    \label{fig:Revelation}
\end{figure}

Given that minimal problems have such failure rates, many point pairs are needed for RANSAC to overcome these instabilities.   
We test this by generating calibrated world scenes as before, but with the number $N$ of world points much more than $5$. 
Add i.i.d. Gaussian noise, with $\sigma$ drawn uniformly from $[0,2]$ pixels in each trial.
The datasets again have no outliers.  
We apply a standard RANSAC method to estimate the essential matrix. Fig.~\ref{fig:hongyi} shows results over $100$ trials for two different choices of $N$.  With $N=50$, the median rotation and translation errors are $>10^{\circ}$ and $> 1$ pixels respectively.  The errors improve when $N=150$.

\begin{figure}[h]
    \centering
    (a)\includegraphics[width=0.4 \linewidth]{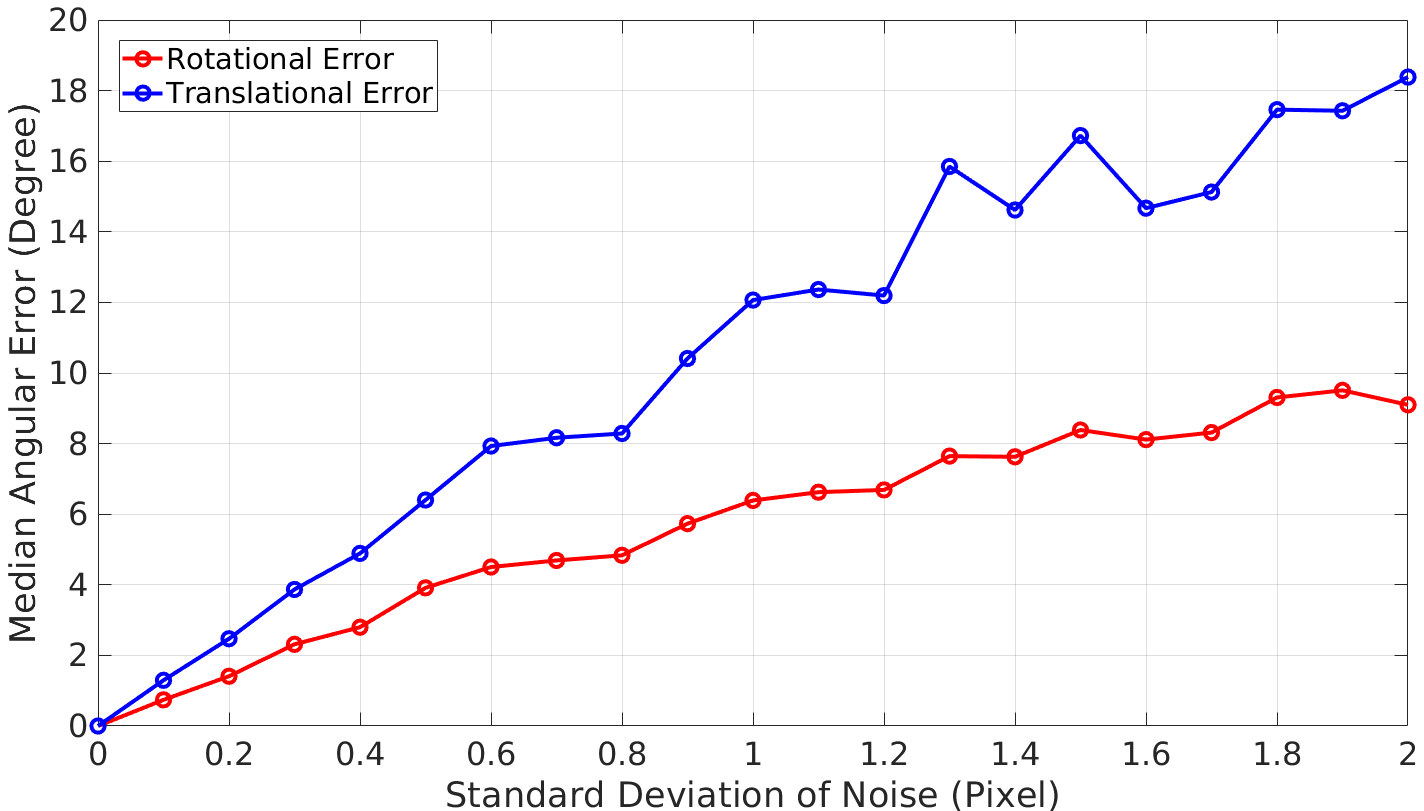}
    (b)\includegraphics[width=0.45 \linewidth]{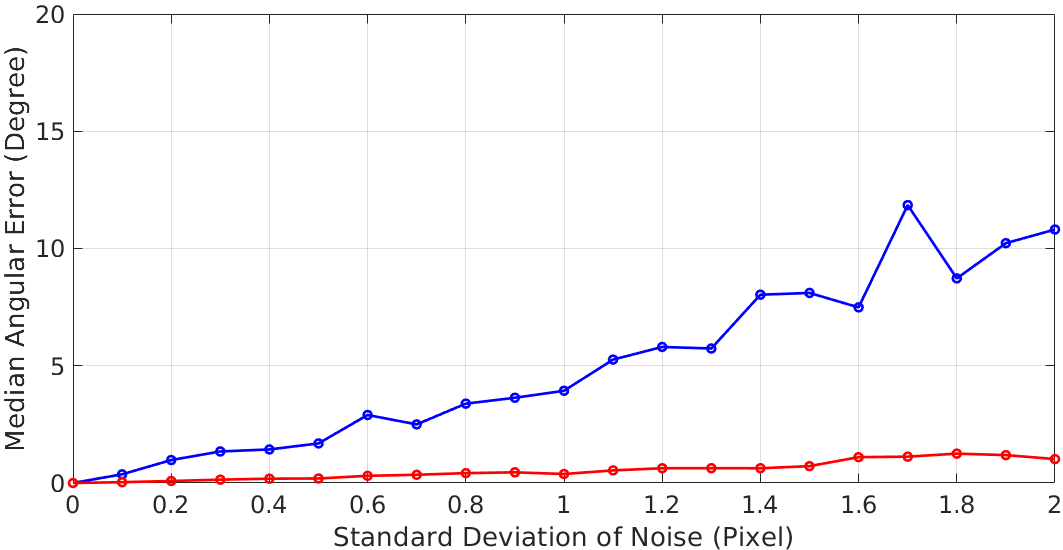}
    \caption{Median translation error and rotation error (both in degrees) vs. standard deviation of noise (in pixels) over $100$ synthetic scenes for: (a) $N =50$ veridical correspondences; and (b) $N=150$ veridical correspondences. More point pairs lead to less error in RANSAC.} 
    \label{fig:hongyi}
\end{figure}

\subsection{Instability estimation by X.5-point curves}

 We apply the methods in Section~\ref{sec:compute-curve} to render the 4.5- and 6.5-point curves.  
Fig.~\ref{fig:F_curves} shows several example curves plotted on the second image plane along with the given image points. 
For the uncalibrated case the degree of the curve is 6 by Theorem~\ref{thm:F-image}, while for calibrated case it is 30 by Theorem~\ref{thm:image-E}.

\begin{figure}[ht]
    \centering
    (a)\includegraphics[height=0.28\linewidth]{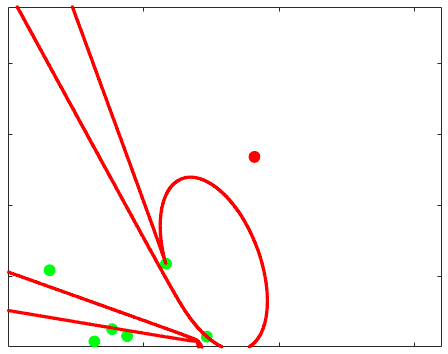}
    (b)\includegraphics[height=0.28\linewidth]{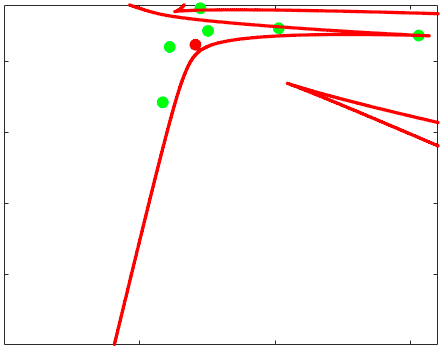}
        (c)\includegraphics[height=0.28\linewidth]{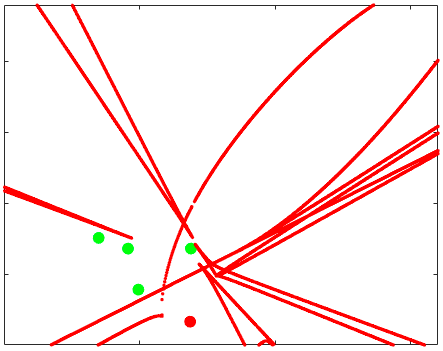}
    (d)\includegraphics[height=0.28\linewidth]{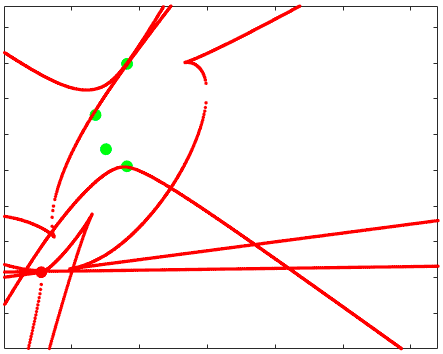}
\caption{Example renderings of X.5-point curves.  
The condition number is approximately the reciprocal distance from the red point to the curve. 
(a) Well-conditioned configuration for $F$ estimation. (b) Poorly-conditioned configuration for $F$ estimation. (c) Well-conditioned configuration for $E$ estimation. (d) Poorly-conditioned configuration for $E$ estimation.}  
    \label{fig:F_curves}
\end{figure}

In another test, we sort the $3000$ random synthetic minimal problems from Section~\ref{sec:only-inliers} into three categories: 
stable, unstable, and borderline cases. 
Here, an instance is classified according to the number of erroneous estimates among $n=20$ perturbations, denoted by $\hat{n}$. If $\hat{n} \in [0, n/3]$ we count the instance as stable; if $\hat{n} \in [2n/3, n]$ we count it as unstable; while if $\hat{n} \in [n/3, 2n/3]$ we count it as borderline. 
We take $\tau = 0.5$ and $\sigma = 0.3$.
In the uncalibrated case, the average distance from the $7$th point to the $6.5$-point curve is $2.35$ pixels among unstable cases, while for the stable cases it is $22.12$ pixels. 
In the calibrated case, the average distance from the $5$th point to the $4.5$-point curve is $0.32$ pixels for unstable cases, while for the stable case it is is $14.95$ pixels. 
See the histograms in Fig.~\ref{fig:statistics}. 
Given the statistical differences, the stable and unstable categories are distinguished 
by the distance from the last point to the X.5-point curve.

\begin{figure}
    \centering
    (a)\includegraphics[width = 0.8 \linewidth]{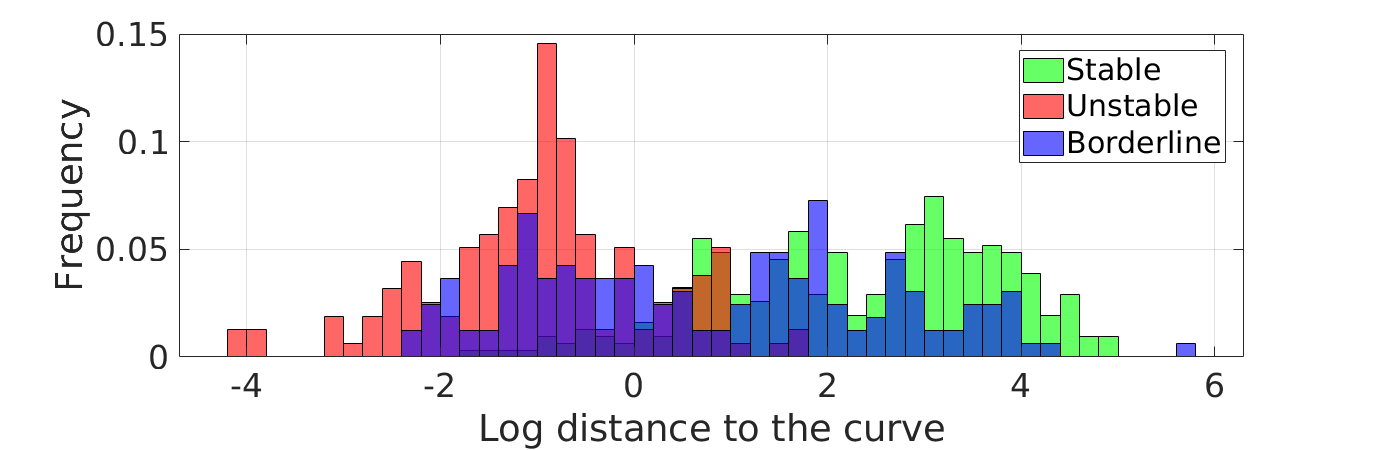}
    (b)\includegraphics[width = 0.8 \linewidth]{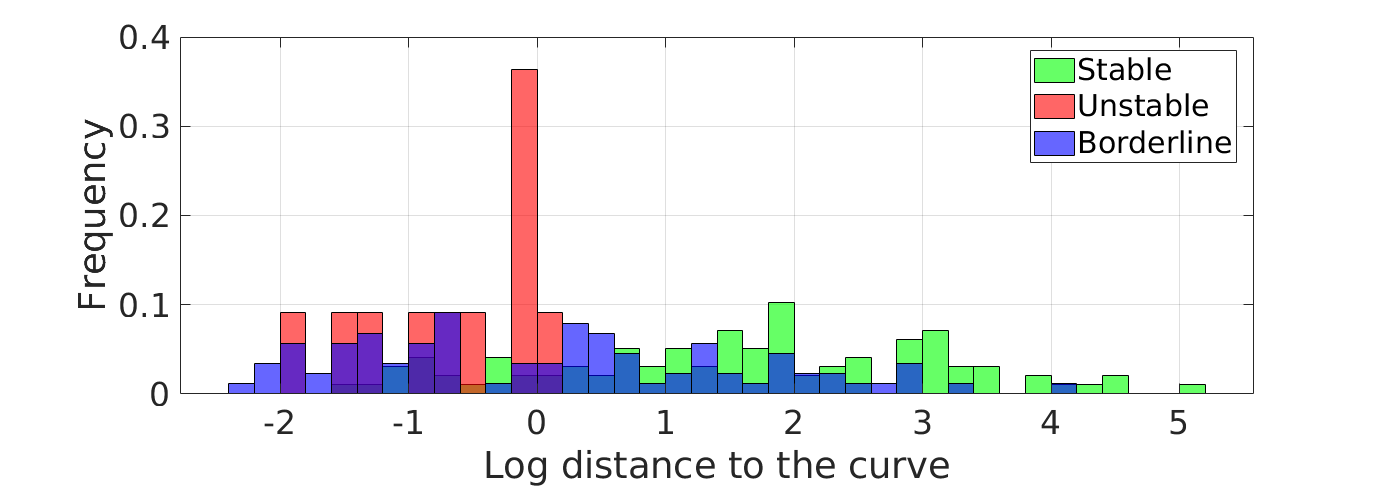}
    \caption{Histograms of distance from the last point to the X.5-point curve sorted by: stable cases (green), unstable cases (red) and borderline cases (blue). (a) Uncalibrated estimation. (b) Calibrated estimation. Stable and unstable categories are separated by distance to the curve.}
    \label{fig:statistics}
\end{figure}

\begin{figure}
    \centering
    (a)\includegraphics[width=0.35\linewidth]{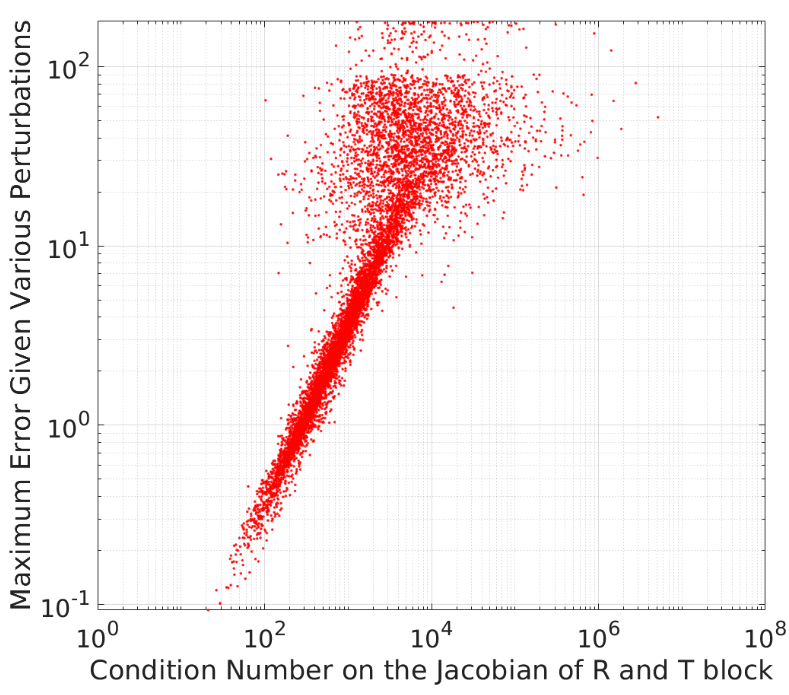}
    (b)\includegraphics[width=0.4\linewidth]{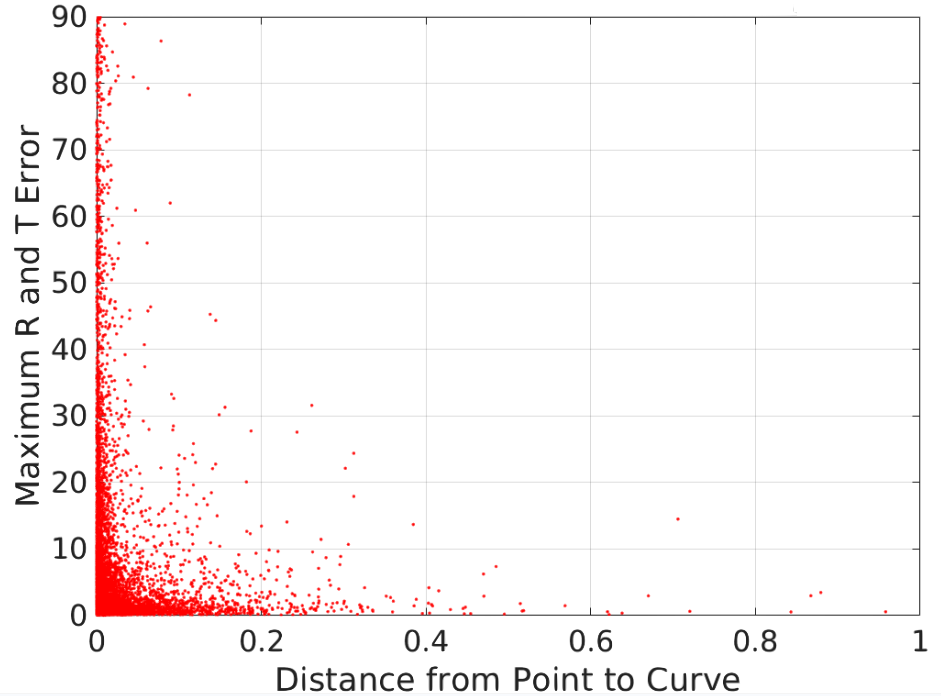}
    \caption{Comparison of the condition number, pose estimation error, and distance from the target point to the X.5-point curve.  (a) Linear relationship between the maximum rotation/translation error among various perturbations on the image point pairs vs. the condition number computed using Proposition~\ref{prop:formula-F}.  As expected,  the linear relation breaks down when the noise is high. (b) Inverse-linear relationship between the maximum rotation/translation error among various perturbation on the image point pairs vs. the distance from the point to $6.5$-point curve.}
    \label{fig:relationshipCheck}
\end{figure}

We next check the relationship between the condition number, the pose estimation error and the distance from the target point to X.5-point curve. 
Results over the $3000$ scenes are plotted in Fig.~\ref{fig:relationshipCheck}.
As expected, there is a linear relationship between the maximum error in rotation/translation and the computed condition number (Fig.~\ref{fig:relationshipCheck}(a)).   We also see an inverse-linear relationship between the distance from the point to the X.5-point curve and the maximum error in the rotation/translation (Fig.~\ref{fig:relationshipCheck}(b)), as expected.

Last, we show that the X.5-point curves themselves are acceptably stable to perturbations in the image data. 
Fig.~\ref{fig:my_label} displays examples of the X.5-point curves when adding noise to the image points. 
Note that the curves do not drastically change.  
Fig.~\ref{fig:stability_syn} collects the statistics over the $6000$ different minimal problems (half calibrated and half uncalibrated), when the perturbations are Gaussian with a standard deviation of $1$ pixel.
We find that the distance to the X.5-point curve is sufficiently robust to be used in practice.

\begin{figure}
    \centering
    (a)\includegraphics[height = 0.23 \linewidth]{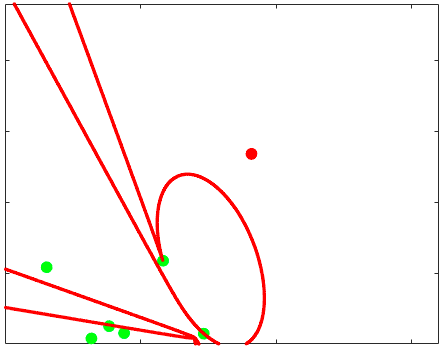} \includegraphics[height = 0.23 \linewidth]{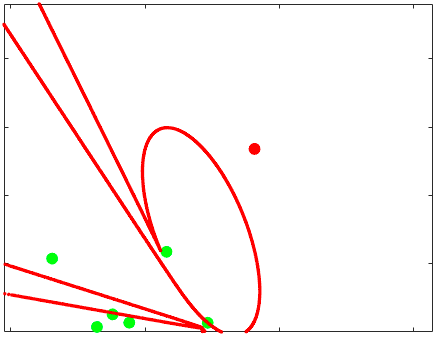} 
    \includegraphics[height = 0.23 \linewidth]{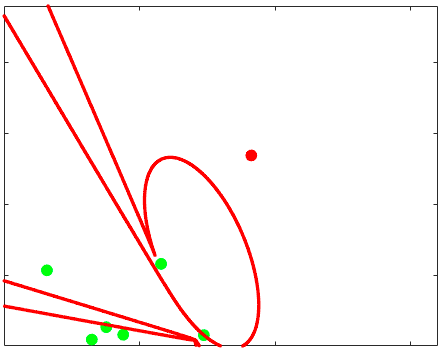}
    (b)\includegraphics[height = 0.23 \linewidth]{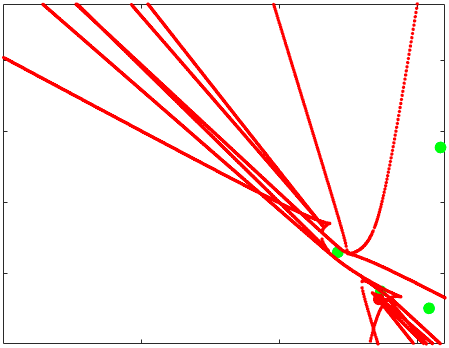} \includegraphics[height = 0.23 \linewidth]{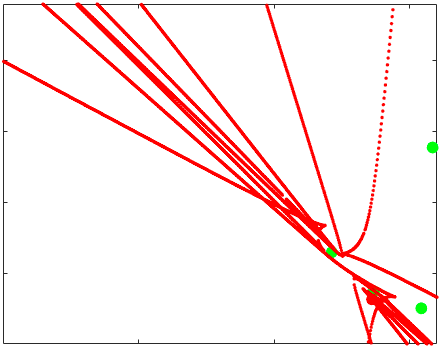} 
    \includegraphics[height = 0.23 \linewidth]{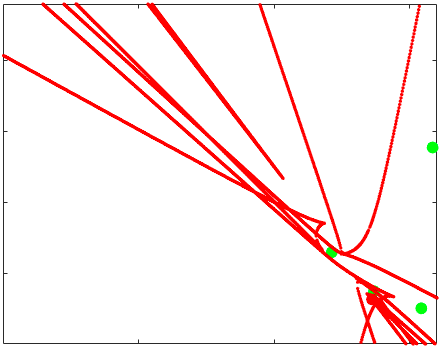}

    \caption{Illustration of the stability of the X.5-point curves themselves. 
    (a) A $6.5$-point curve for uncalibrated estimation with different perturbations on the image points applied. 
    (b) A $4.5$-point curve for calibrated estimation with different perturbations on the image points applied.}
    \label{fig:my_label}
\end{figure}

\begin{figure}
    \centering
    \includegraphics[width=0.8 \linewidth]{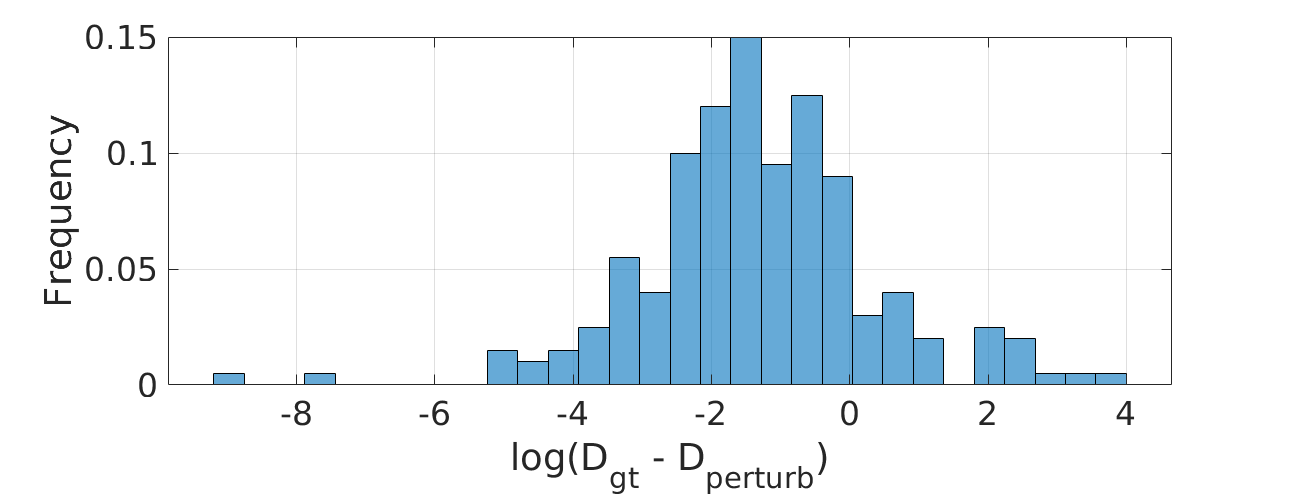}
    \caption{Histogram of the log-difference between two distances:  from the target point to the X.5-point curve without perturbation; and from the target point to the X.5-point curve with perturbation on the other points.}
    \label{fig:stability_syn}
\end{figure}
\subsection{Illustrations with real data}
We first demonstrate that the failure of RANSAC described at the end of Section~\ref{sec:only-inliers} can occur with real data.
To this end, we take various overlapping pairs of images from the RANSAC 2020 dataset \cite{mishkin}, manually remove all outliers, and then apply a standard RANSAC procedure to estimate the epipolar geometries.
Despite there being all inliers, this estimation suffers errors of $> 10^{\circ}$  in the rotation for $> 50 \%$ of image pairs, similarly to the synthetic experiments in Section~\ref{sec:only-inliers}.  
A sample result with real data is shown in Fig.~\ref{fig:Teaser}.

Next, we illustrate utility of the X.5-point curve on real data. 
Again we use image pairs from the RANSAC 2020 dataset~\cite{mishkin}, where veridical image point correspondences are available.  
Fig.~\ref{fig:Real} shows that for a minimal instance in which standard minimal solvers have large errors, the selected image point indeed lies close to the X.5-point curve.
It is consistent with our theory, because such data instances are poorly-conditioned.

Last, we note that X.5-point curves might also be useful in  near-degenerate situations (recall the comparison made in Section~\ref{sec:relationship}).
For example, consider near-planar scenes.
It is well-known that the near-planar scenes are nearly degenerate.  
However, their minimal subscenes are also ill-posed by Theorems~\ref{thm:illposed-world} and \ref{thm:F-ill-posed-world}. 
In the uncalibrated case, for example, minimal subscenes with at least two off-plane points would be both stable as well as non-degenerate. 
If the distance from a selected point to the 6.5-point curve is large, 
it could help in identifying good minimal instances.
The idea is illustrated in Fig.~\ref{fig:planar}.

\subsection{Discussion of the non-minimal case}
A common practice in pose estimation is to include a post-RANSAC refinement stage. 
Here we demonstrate that the refinement stage can also suffer from stability issues. Similar to the example shown in Fig.~\ref{fig:Teaser}, we consider local optimization  applied to the inliers. 
A significant discrepancy remains between the estimated and ground-truth epipolar geometries, suggesting that the stability problem persists even when the number of points is greater than minimal.

To evaluate the issue statistically, we conduct essential matrix estimation using the PoseLib solver~\cite{PoseLib}, including its nonlinear post-RANSAC refinement framework, applied to the RANSAC 2020 dataset. The dataset provides correspondences, from which we select all matches with high confidence scores. We solve for the essential matrix using PoseLib's robust pipeline. 
To ensure a fair evaluation, we exclude image pairs with large baselines and low co-visibility by considering only pairs with more than 50 inliers.

Applying the PoseLib solver to the RANSAC 2020 dataset, we observe a $64\%$ 
success rate under a $10$-degree angular threshold.  The failure rate of $36\%$ indicates that instability can remain after refinement using a non-minimal set of inliers. 
Fig.~\ref{fig:epipolar_geometry} shows two example scenes where the final non-minimal fitting exhibits significant instability.

\begin{figure}
    \centering
    (a)\includegraphics[height = 0.34 \linewidth]{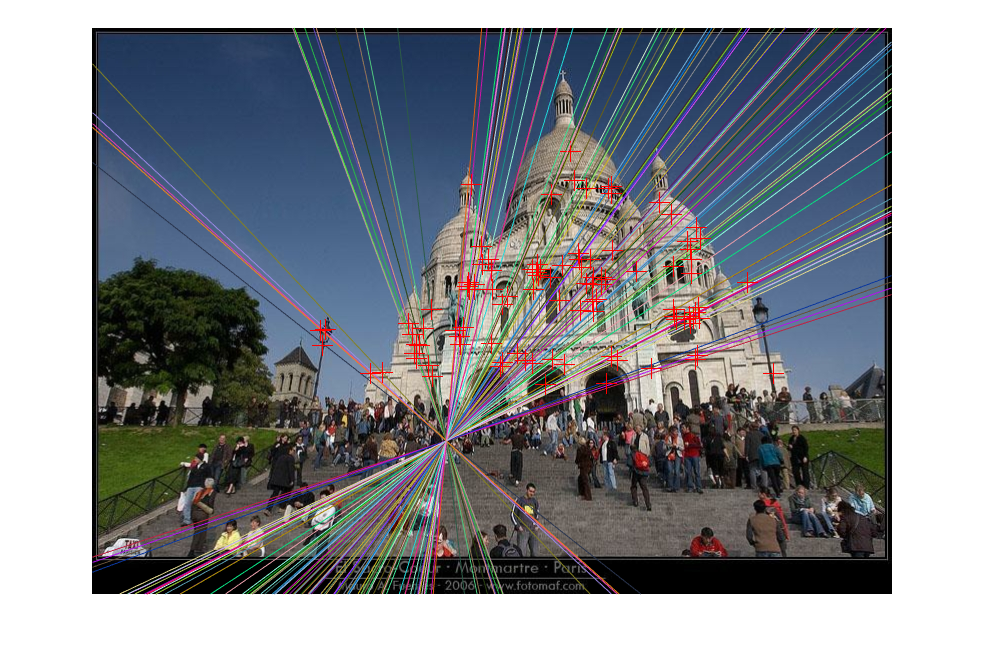} 
    \includegraphics[height = 0.34 \linewidth]{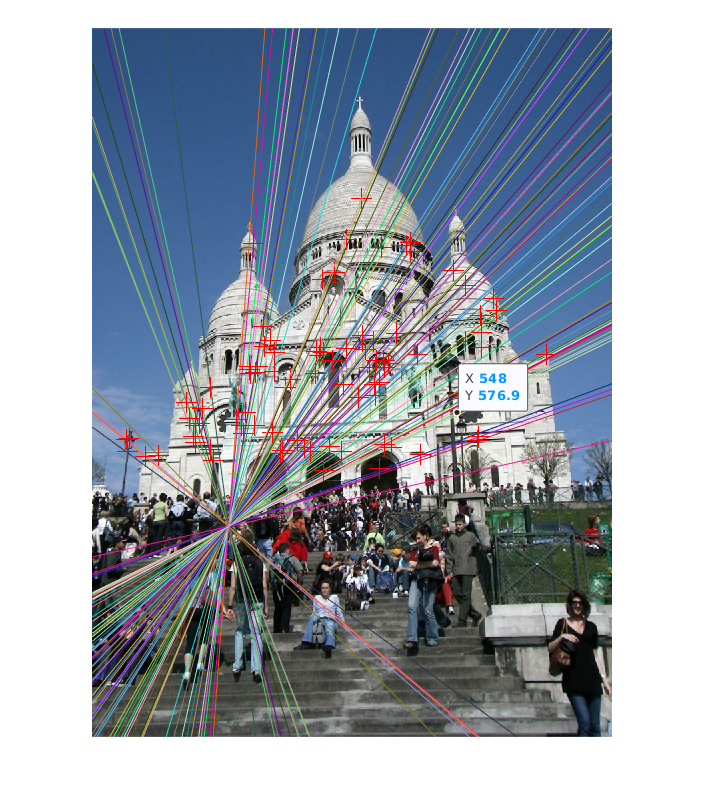} \\
    (b)\includegraphics[height = 0.34 \linewidth]{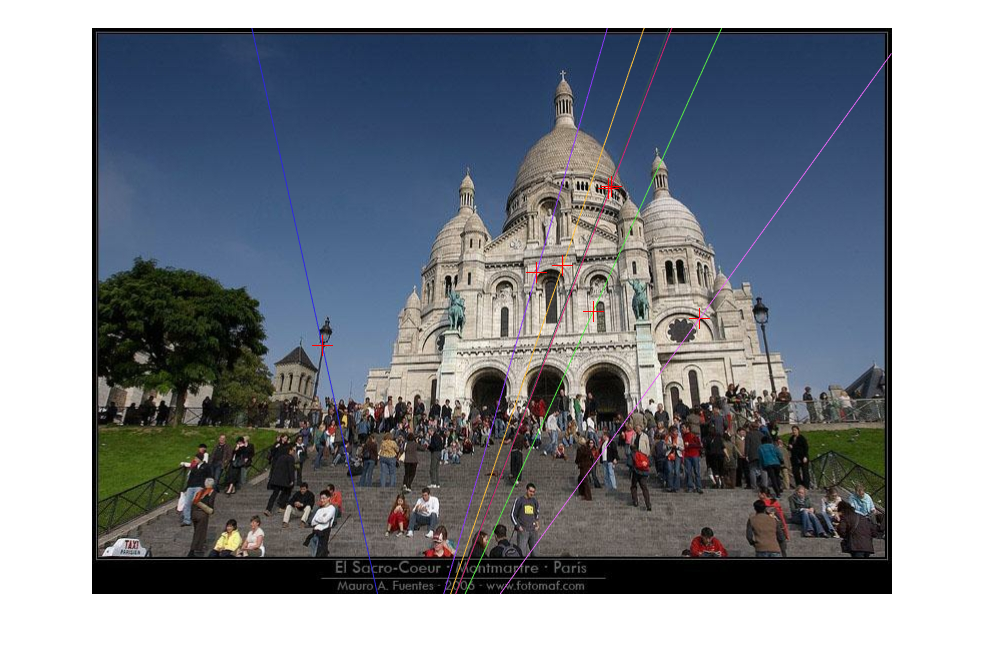} 
    \includegraphics[height = 0.34 \linewidth]{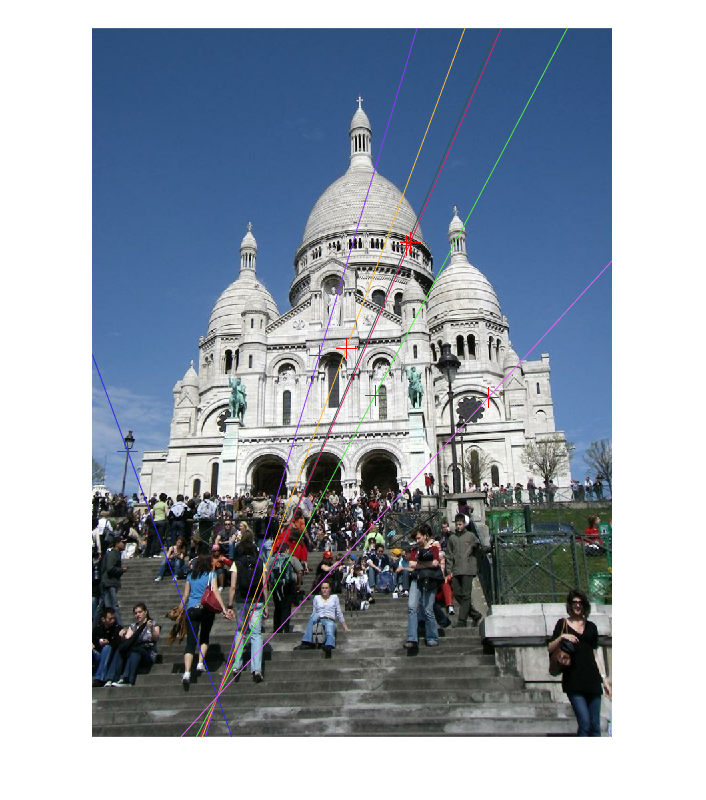} \\
    (c) \includegraphics[width = 0.7 \linewidth]{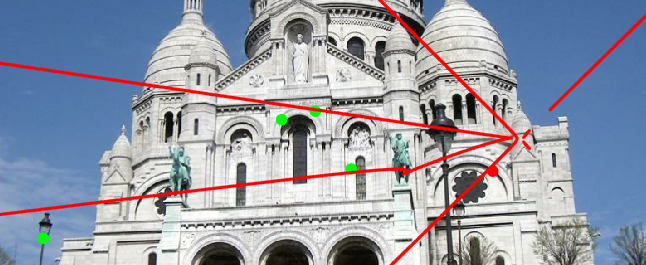}
    \caption{Example with real data demonstrating an unstable minimal configuration from all-inlier correspondences. (a) The ground-truth epipolar geometry of a pair of images. (b) The closest solution found by the $7$-point algorithm give seven inliers. (c) Zoomed-in image showing that the remaining point is close to the $6.5$-point curve.  This indicates the minimal instance is poorly conditioned, so the errors in (b) are expected.}
    \label{fig:Real}
\end{figure}

\begin{figure}
 \centering   
      (a)\includegraphics[height = 0.255\linewidth]{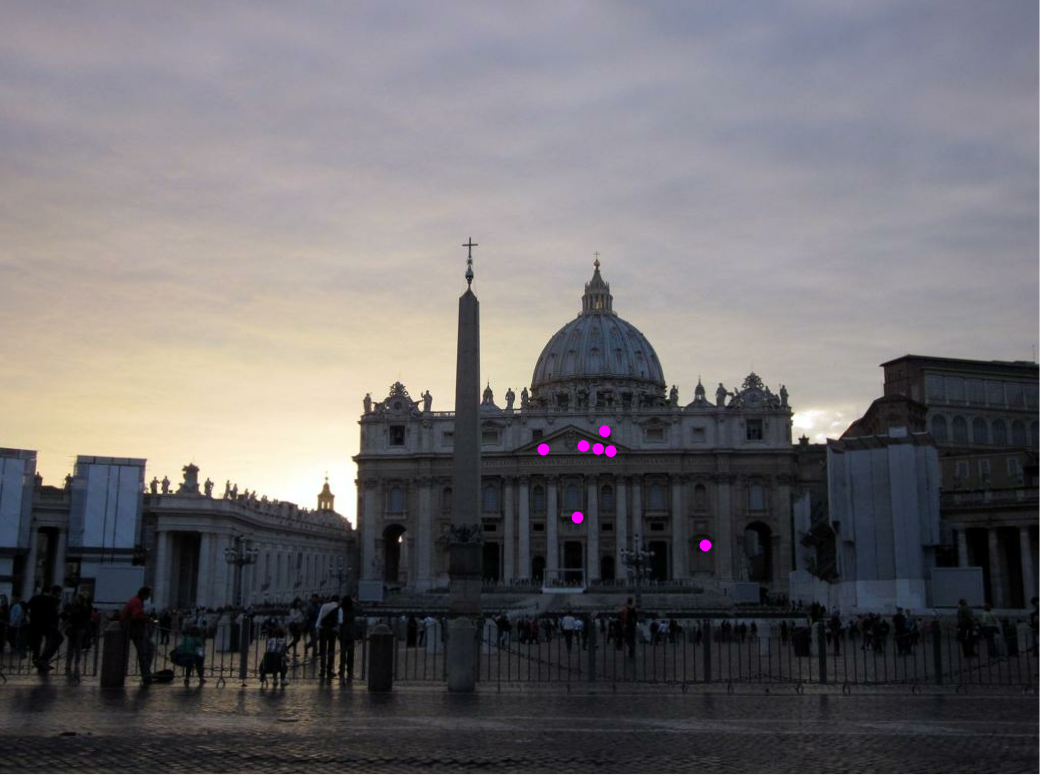}
      \includegraphics[height = 0.255\linewidth]{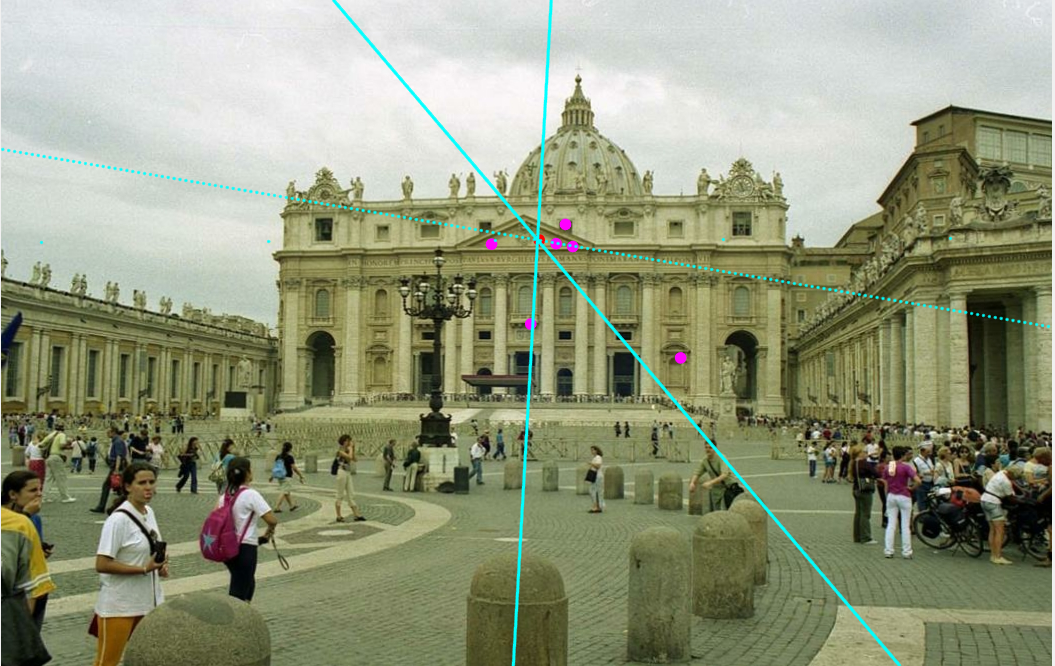} 
      \includegraphics[height = 0.255\linewidth]{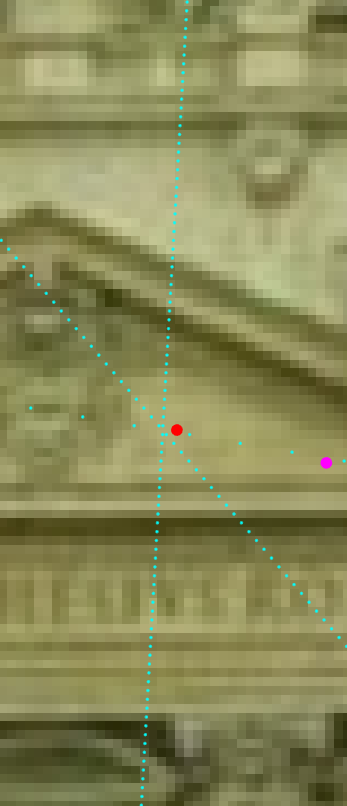} \\
      (b)\includegraphics[height = 0.3\linewidth]{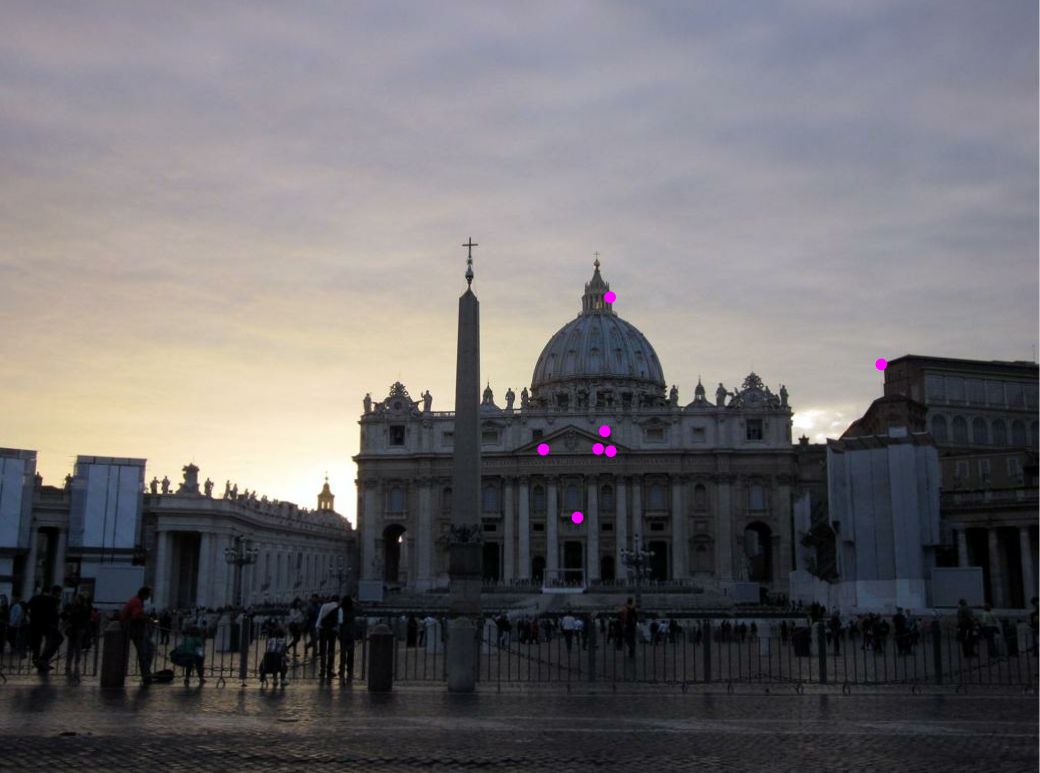}
      \includegraphics[height = 0.3\linewidth]{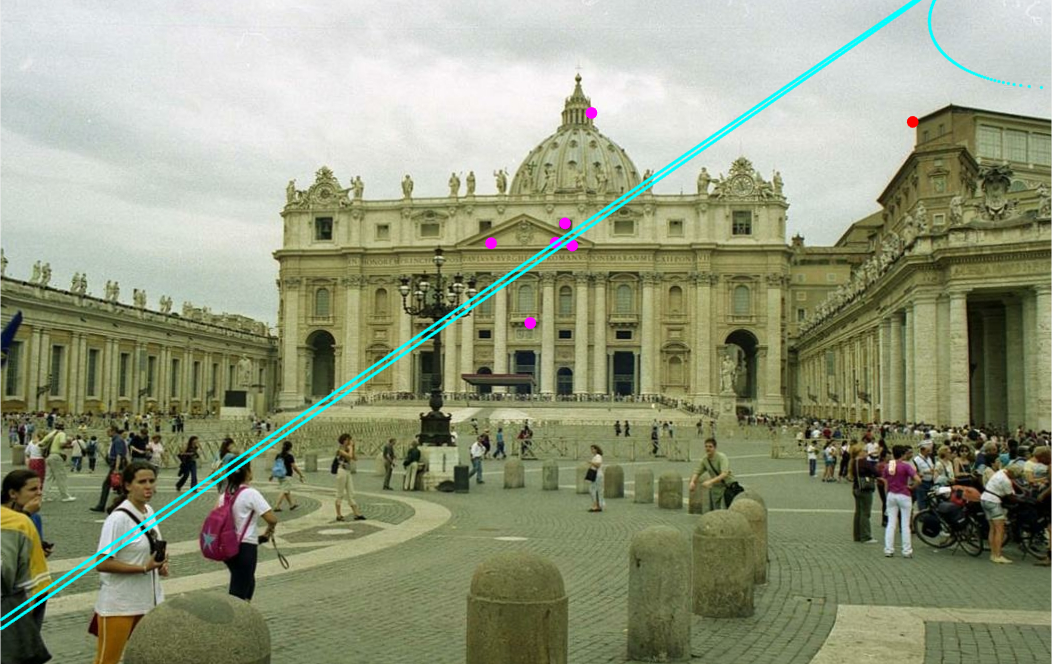} \\
      (c)\includegraphics[height = 0.30\linewidth]{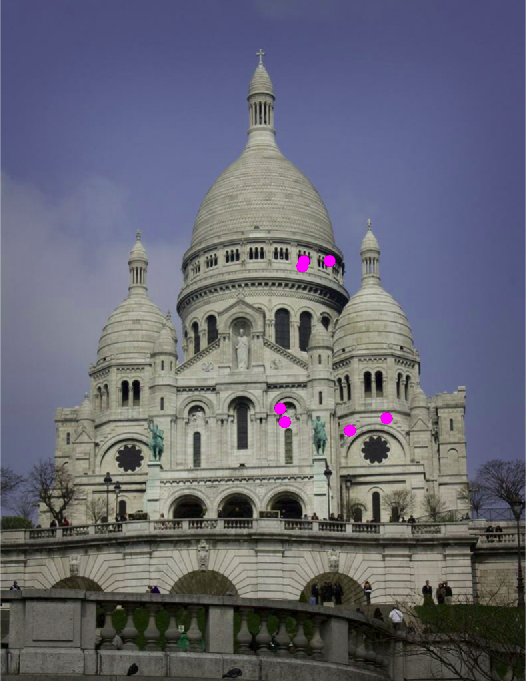}
      \includegraphics[height = 0.30\linewidth]{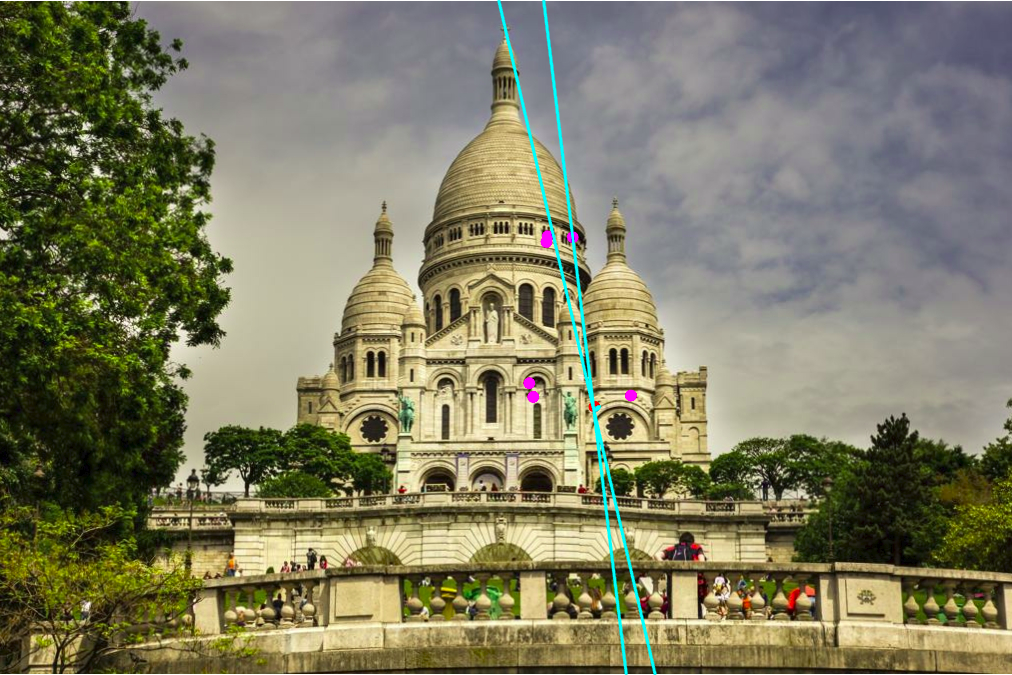}
      \includegraphics[height = 0.30\linewidth]{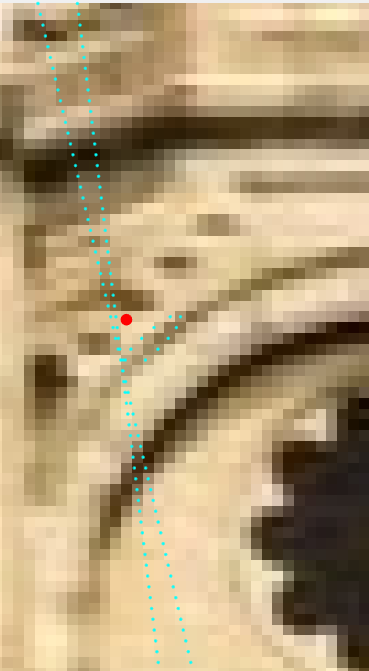}\\ 
      (d)\includegraphics[height = 0.38\linewidth]{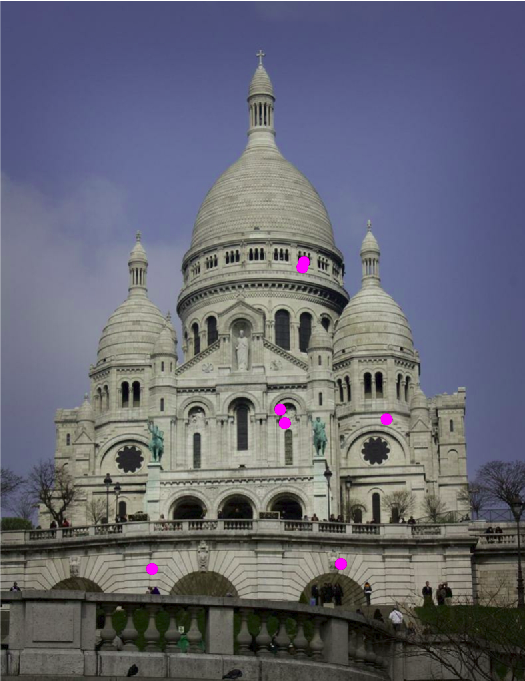}
      \includegraphics[height = 0.38\linewidth]{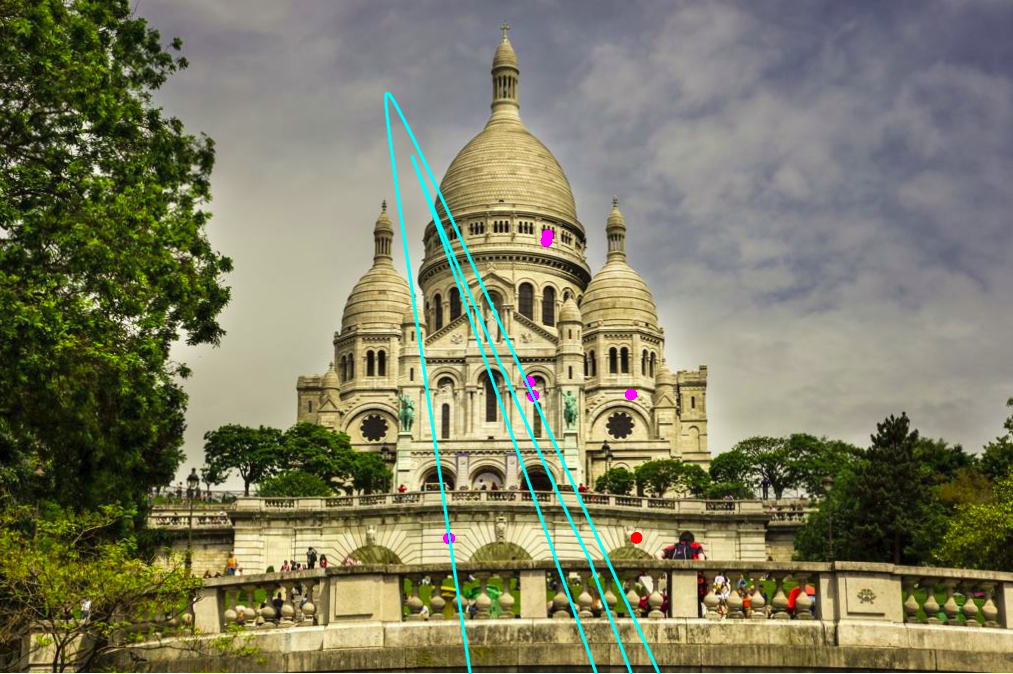}
    \caption{Illustrative results showing how the $6.5$-point curve may help with near-planar world scenes. (a)(c) $7$ near-planar point pair selections (magenta) generate small distances from the target point (red) to the $6.5$-point curve (cyan). (b)(d) $5$ near-planar point pair selections and $2$ off-plane point pairs (magenta) generate large distance from the target point (red) to the $6.5$-point curve (cyan).}
    \label{fig:planar}
\end{figure}

\begin{figure*}[h]
\centering
(a)\includegraphics[width=0.4\linewidth]{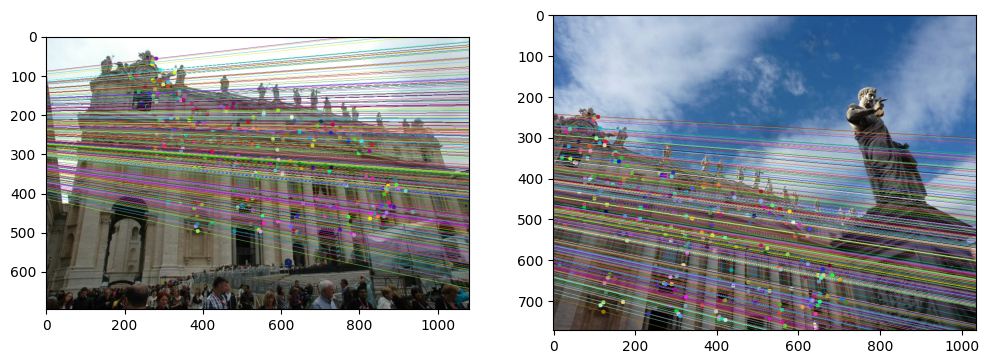} (b)\includegraphics[width=0.4\linewidth]{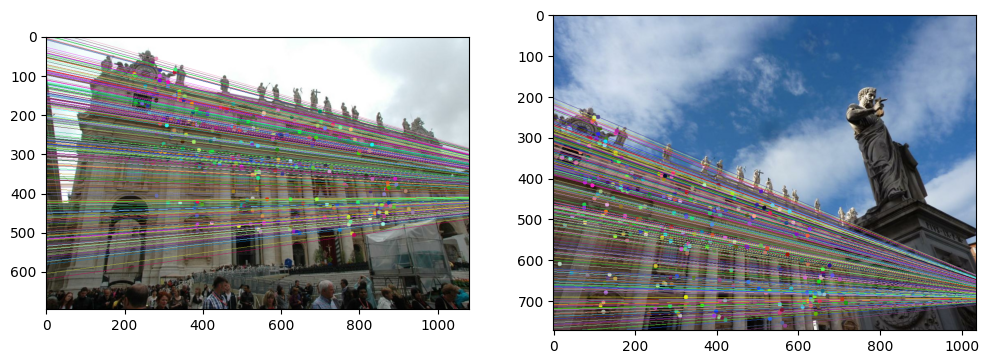}
(c)\includegraphics[width=0.4\linewidth]{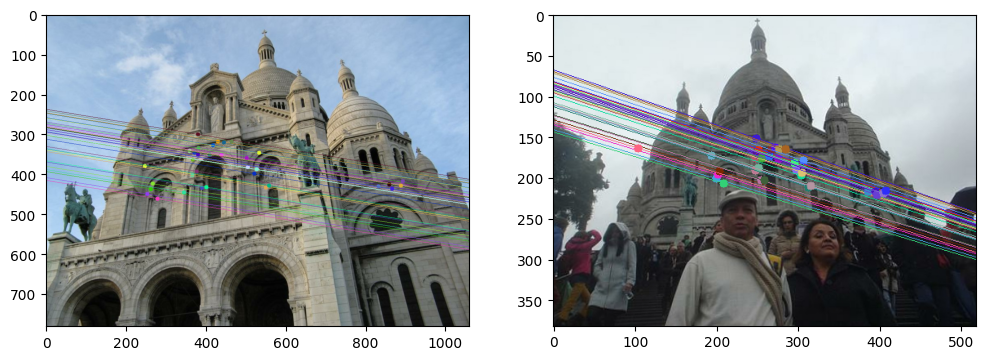} (d)\includegraphics[width=0.4\linewidth]{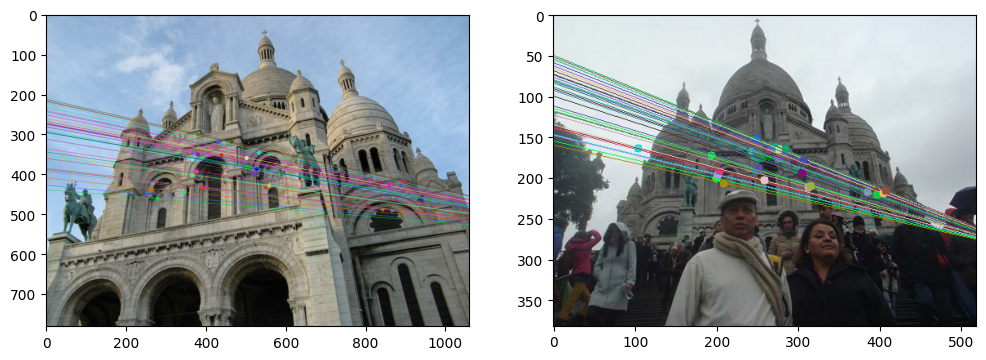}
\caption{(a) and (c): Epipolar geometry estimated by PoseLib~\cite{PoseLib}. (b) and (d): Ground Truth Epipolar geometry. }
\label{fig:epipolar_geometry}
\end{figure*}

\section{Conclusion}
Stability has been largely ignored in multiview geometry.  To the extent it has been considered, it has been conflated with degeneracy and criticality.
 In this paper, we presented a general framework for analyzing instabilities of minimal problems.
We defined condition numbers as well as ill-posed world scenes and image data.  
General methods were introduced and used to analyze the $5$- and $7$-point problems in relative pose estimation. 
Numerical experiments on synthetic and real data  support our theory.  They also suggest that instabilities negatively impact RANSAC in practice.

Opportunities for future work abound.  To name a few, first we would like to apply the framework to analyze other important minimal problems in multiview geometry.  
Second, it is very desirable to make RANSAC ``stability-aware", \eg by biasing its sampling towards well-conditioned image data, perhaps by combining \cite{barath2022learning} with the X.5-point curves.  
Lastly, we propose using related theory \cite{breiding2021condition} to analyze the conditioning of non-minimal estimation.

{\small
\bibliographystyle{ieee}
\bibliography{egbib}
}

\appendices
\newcommand{\mainref}[1]{M-\ref{#1}}

\section{Proof of Proposition~\mainref{prop:quotient-ess}}
Let $\mathcal{V}$ denote the open subset of $\mathcal{C}^{\times 2} \times (\mathbb{R}^3)^{\times 5}$ in (\mainref{eq:W-essential}).

The action (\mainref{eq:ess-action}) is clearly smooth.  

For freeness, we must show that each stabilizer is trivial.
Equivalently, each orbit contains a point with trivial stabilizer. 
Note that each orbit contains a point $v=(P, \bar{P}, \Gamma_1, \ldots, \Gamma_5)$ in $\mathcal{V}$ where $P = (I \,\, 0)$ and $\bar{P}$ has a nonzero fourth column.  
If $g \in G$ stabilizes $v$, then $P  = Pg $ and $\bar{P} = \bar{P} g$ as elements of $\mathcal{C}$.  
Write $g = \begin{pmatrix} \mathbf{R} & \mathbf{T} \\ 0 & \lambda \end{pmatrix}$.  
From $P = Pg$, it follows $\mathbf{R} = I$ and $\mathbf{T} = 0$.  From $\bar{P}  = \bar{P}g$, it follows $\lambda = 1$.  Altogether $g = I$, so the stabilizer is trivial.

Next we show properness.  
By \cite[Ch.~1, Sec.~10]{bourbaki2013general}, this means  $\varphi : \mathcal{V} \times G \rightarrow \mathcal{V} \times \mathcal{V}$, $\left( v, g \right) \mapsto \left( v, g \cdot v \right)$ is a closed map with compact fibers.  
By freeness (previous paragraph), $\varphi$ is injective.  So its fiber are either empty or singletons, hence compact.
For closedness, we will prove the stronger property that $\varphi$ is a homeomorphism onto its image and its image is a closed subset of $\mathcal{V} \times \mathcal{V}$.
Note that the image of $\varphi$ consists of all $(v^{(1)}, v^{(2)}) \in \mathcal{V} \times \mathcal{V}$ with $v^{(i)} = (P^{(i)}, \bar{P}^{(i)}, \Gamma_1^{(i)}, \ldots, \Gamma_5^{(i)})$  where $P^{(i)} = (\mathbf{R}^{(i)} \,\,\, \mathbf{T}^{(i)})$ and $\bar{P}^{(i)} = (\bar{\mathbf{R}}^{(i)} \,\,\, \bar{\mathbf{T}}^{(i)})$ for $i=1,2$ 
such that: 
\begin{enumerate} \setlength\itemsep{0.35em}
\item[\text{i)}] $(P^{(1)}, \bar{P}^{(1)})$ and $(P^{(2)}, \bar{P}^{(2)})$ have the same essential matrix;
\item[\text{ii)}] ${\mathbf{R}}^{(1)} (\bar{\mathbf{R}}^{(1)})^{\top} = {\mathbf{R}}^{(2)} (\bar{\mathbf{R}}^{(2)})^{\top}$;
\item[\text{iii)}] $v^{(1)}$ and $v^{(2)}$ map to the same image point pairs.
\end{enumerate}
Indeed, this set contains $\operatorname{Im}(\varphi)$, and the reverse containment follows from \cite[Res.~9.19]{hartleyzisserman}.
In particular, $\operatorname{Im}(\varphi)$ is a closed subset of $\mathcal{V} \times \mathcal{V}$.
For the homeomorphism claim, let $\psi : \operatorname{Im}(\varphi) \rightarrow \mathcal{V} \times G$ be the left-inverse of $\varphi$, which exists since $\varphi$ is injective.  Then $\psi(v^{(1)}, v^{(2)}) = (v^{(1)}, g)$ for some $g = \begin{pmatrix} \mathbf{R} & \mathbf{T} \\ 0 & \lambda  \end{pmatrix}$.   
We need to show that $g$ is a continuous function of $(v^{(1)}, v^{(2)})$.
From $P^{(1)}  = P^{(2)} g$, we see $\mathbf{R} = \mathbf{R}^{(1)}(\mathbf{R}^{(2)})^{\top}$, which is continuous.
Using $\bar{P}^{(1)}  = \bar{P}^{(2)} g$,  
\begin{equation} \label{eq:stacked-P}
\begin{pmatrix}
P^{(2)} \\
\bar{P}^{(2)}
\end{pmatrix} 
\begin{pmatrix}
\mathbf{T} \\
\lambda
\end{pmatrix} \,\, = \,\, 
\begin{pmatrix}
\mathbf{T}^{(1)} \\
\mathbf{T}^{(2)}
\end{pmatrix}.
\end{equation}
The $6 \times 4$ matrix in \eqref{eq:stacked-P} is left-invertible, because $P^{(2)}$ and $\bar{P}^{(2)}$ have distinct centers.
By \cite[pg.~3]{penrose1955generalized}, it follows that $\mathbf{T}$ and $\lambda$ are continuous functions of $(v^{(1)}, v^{(2)}) \in \operatorname{Im}(\varphi)$, which concludes the verification of properness.

Finally, recall the quotient manifold theorem \cite[Thm.~21.10]{lee2013smooth}.
It implies a unique smooth structure on $\mathcal{W} = \mathcal{V} / G$ such that the quotient map $\mathcal{V} \rightarrow \mathcal{V} / G$ is a smooth submersion.
From  $\dim(\mathcal{V}) = \dim(\mathcal{C}^{\times 2} \times (\mathbb{R}^3)^{\times 5}) = 2 \cdot 6 + 5 \cdot 3 = 27$ and $\dim(G) = 7$,  we get $\dim(\mathcal{W}) = 27-7 = 20$. 
 \hfill $\square$

\section{Proof of Lemma~\mainref{lem:double}}

Let the given map from $\mathcal{U}$ to $\mathcal{W}$ be $\psi$.
Write $\psi = \psi_2 \circ \psi_1$ where $\psi_1: \mathcal{U} \rightarrow \mathcal{V}$, $\mathcal{V}$ is the open subset of $\mathcal{C}^{\times 2} \times (\mathbb{R}^3)^{\times 5}$ in (\mainref{eq:W-essential}), and $\psi : \mathcal{V} \rightarrow \mathcal{V} / G = \mathcal{W}$ is the quotient map.

We first check that $\psi$ is surjective. 
Let $v = (P, \bar{P}, \Gamma_1, \ldots, \Gamma_5) \in \mathcal{V}$ where 
$P = (\mathbf{R}_1 \,\, \mathbf{T}_1)$ and $\bar{P} = (\mathbf{R}_2 \,\, \mathbf{T}_2)$. 
We wish to show $\operatorname{im} \psi_1 \, \cap \, (G \cdot v) \, \neq \,\emptyset$.  
Indeed, $((I \,\, 0 ), (\mathbf{R} \,\, \hat{\mathbf{T}}), \tfrac{1}{\lambda} \Gamma_1, \ldots, \tfrac{1}{\lambda}\Gamma_5)$ is in the intersection, where $\mathbf{R} = \mathbf{R}_2 \mathbf{R}_1^{\top}$, 
$\hat{\mathbf{T}} = \frac{\mathbf{T}_2 - \mathbf{R}_2 \mathbf{R}_1^{\top} \mathbf{T}_1}{\lambda}$ and $\lambda = \| \mathbf{T}_2 - \mathbf{R}_2 \mathbf{R}_1^{\top} \mathbf{T}_1 \|_2$.

Next we check that $\psi$ is $2$-to-$1$ everywhere. 
Assume $u^{(i)} = (\mathbf{R}^{(i)}, \hat{\mathbf{T}}^{(i)}, \Gamma^{(i)}_1, \ldots, \Gamma_5^{(i)})$ for $i=1,2$ satisfy $\psi(u^{(1)}) = \psi(u^{(2)})$.  
Equivalently, $\psi_1(u^{(1)})$ and $\psi_1(u^{(2)})$ are in the same $G$-orbit, \ie there exists $g = \begin{pmatrix} \mathbf{R} & \mathbf{T} \\ 0 & \lambda \end{pmatrix}$ such that 
\begin{equation*}
\begin{cases}
(I \,\, 0) = (I \,\, 0)  g \\[2pt]
(\mathbf{R}^{(1)} \,\, \hat{\mathbf{T}}^{(1)}) = (\mathbf{R}^{(2)} \,\, \hat{\mathbf{T}}^{(2)}) g \\[2pt]
g (\Gamma_i^{(1)};  1) = (\Gamma_i^{(2)};  1)
\end{cases}
\end{equation*}
The only possibilities are $g = I$ and $u^{(1)} = u^{(2)}$ or 
$g = \operatorname{diag}(1,1,1,-1)$ and 
$(\mathbf{R}^{(2)}, \hat{\mathbf{T}}^{(2)}, \Gamma^{(2)}_1, \ldots, \Gamma^{(2)}_5) = (\mathbf{R}^{(1)}, - \hat{\mathbf{T}}^{(1)}, -\Gamma^{(1)}_1, \ldots, -\Gamma^{(1)}_5)$ as desired.

Lastly we check that $\psi$ is a submersion.
Let $u = (\mathbf{R}, \hat{\mathbf{T}}, \Gamma_1, \ldots, \Gamma_5)$ and $v = \psi_1(u)$.
It is easy to see  
\begin{multline} \label{eq:psi1-im-joe}
\operatorname{im} D \psi_1(u) =  0 \times \{{(} [\dot{s}]_{\times} \mathbf{R} \,\,\,\, \dot{\alpha}_1 \mathbf{T}^{\perp}_{1} + \dot{\alpha}_2 \mathbf{T}^{\perp}_{2} {)} : \\ \dot{s} \in \mathbb{R}^3,  \dot{\alpha} \in \mathbb{R}^2  \}  \times (\mathbb{R}^3)^{\times 5} 
\end{multline}
using facts \eqref{eq:tangent-SO}, \eqref{eq:sphere-tangent}, \eqref{eq:proj-space} below.
On the other hand,   by Proposition~\mainref{prop:quotient-ess},
\begin{equation} \label{eq:psi2-tang-joe}
\operatorname{ker} D \psi_2(v) = T(G \cdot v, v).
\end{equation}
A chain rule computation shows
\begin{multline}\label{eq:chain-rule-orbit}
T(G \cdot v, v) = \Big{ \{ } (-([\dot{s}]_{\times} \,\,\, \dot{\mathbf{T}}),  \tfrac{-1}{2} (\mathbf{R} [\dot{s}]_{\times} \,\,\, \mathbf{R}\dot{\mathbf{T}} + \dot{\lambda} \hat{\mathbf{T}})^{\perp},  \ldots, \\ [\dot{s}]_{\times} \Gamma_i + \dot{\mathbf{T}} - \dot{\lambda} \Gamma_i, \ldots ) \, : \, s \in \mathbb{R}^3, \dot{\mathbf{T}} \in \mathbb{R}^3, \dot{\lambda} \in \mathbb{R} \Big{ \} },
\end{multline}
where ``$\perp$" indicates the component orthogonal to $(\mathbf{R} \,\, \hat{\mathbf{T}})$ and we're identifying $T(\mathcal{C}, (\mathbf{R} \,\, \hat{\mathbf{T}}))$ with a subspace of $T(\mathbb{P}(\mathbb{R}^{3 \times 4}), (\mathbf{R} \,\, \hat{\mathbf{T}})) = (\mathbb{R}^{3 \times 4})^{\perp}$ (see Section~\ref{subsec:projsp} below).  
Comparing \eqref{eq:psi1-im-joe}, \eqref{eq:psi2-tang-joe}, \eqref{eq:chain-rule-orbit}, one verifies
\begin{equation*}
\operatorname{im} D \psi_1(u) \, \cap \, \operatorname{ker} D \psi_2(v) = 0.
\end{equation*}
Further, clearly $ D \psi_1(u)$ has rank $\dim(\mathcal{U}) = \dim(\mathcal{W})$. 
As wanted,
$D\psi(u) = D \psi_2(v) \circ D\psi_1(u)$ is surjective.

By \cite[Cor.~4.13]{lee2013smooth}, we have proven that $\psi$ is a smooth double cover of $\mathcal{W}$. \hfill $\square$

\section{Proof of Proposition~\mainref{prop:quotient-smooth}}

Let $\mathcal{V}$ be the open subset of $\mathcal{C}^{\times 2} \times (\mathbb{P}_{\mathbb{R}}^{3})^{\times 7}$ in (\mainref{eq:W-fundamental}).

The action (\mainref{eq:action}) is clearly smooth.  

For freeness, note that each orbit contains a point $v = (P, \bar{P}, \tilde{\Gamma}_1, \ldots, \tilde{\Gamma}_7)$ where $P = (I \,\, 0)$ and $\bar{P} = (\mathbf{p}_1 \,\, \mathbf{p}_2 \,\, \mathbf{p}_3 \,\, \mathbf{p}_4)$ with $\mathbf{p}_4 \neq 0$.  
Let $g \in \operatorname{PGL}(4)$ stabilize $v$.  
From $P  = P g$, 
\begin{equation}\label{eq:nice-g-real-nice}
g \, = \, \begin{pmatrix} 1 & 0 & 0 & 0 \\ 0 & 1 & 0 & 0 \\ 0 & 0 & 1 & 0 \\ g_{41} & g_{42} & g_{43} & g_{44} \end{pmatrix} 
\end{equation}
where $g_{44} \neq 0$. 
From $\bar{P}  = \bar{P} g$, get $\mathbf{p}_i + g_{4i} \mathbf{p}_4 = g_{44} \mathbf{p}_i$ in $\mathbb{R}^3$ for $i = 1, 2, 3$.  
If $g_{44} \neq 1$ then $\mathbf{p}_i = \frac{g_{4i}}{g_{44} - 1} \mathbf{p}_4$ for $i=1,2,3$.  
Then $\operatorname{rank}(\bar{P}) = 1$, contradicting $\operatorname{rank}{\bar{P}} = 3$.  
Thus $g_{44} = 1$, implying $g_{4i} = 0$ for $i = 1, 2, 3$.  Therefore $g = I$, and the stabilizer is trivial.

Next consider properness, or equivalently that  $\varphi : \mathcal{V} \times \operatorname{PGL}(4) \rightarrow \mathcal{V} \times \mathcal{V}$, $\left( v, g \right) \mapsto \left( v, g \cdot v \right)$
is a closed map with compact fibers.
By freeness, $\varphi$ is injective; in particular its fibers are compact.
For closedness, we show that $\varphi$ is a homeomorphism onto its image and $\operatorname{Im}(\varphi)$ is a closed subset of $\mathcal{V} \times \mathcal{V}$.  
By \cite[Thm.~9.10]{hartleyzisserman}, the image consists of all pairs 
$(v^{(1)}, v^{(2)}) \in \mathcal{V} \times \mathcal{V}$ with $v^{(i)} = (P^{(i)}, \bar{P}^{(i)}, \tilde{\Gamma}_1^{(i)}, \ldots, \tilde{\Gamma}_7^{(i)})$ for $i=1,2$ such that:
\begin{enumerate}\setlength\itemsep{0.35em}
\item[\text{i)}] $(P^{(1)}, \bar{P}^{(1)})$ and $(P^{(2)}, \bar{P}^{(2)})$ have the same fundamental matrix;
\item[\text{ii)}] $v^{(1)}$ and $v^{(2)}$ map to the same image point pairs in $(\mathbb{R}^2 \times \mathbb{R}^2)^{\times 7}$.
\end{enumerate}
For the homeomorphism claim, let $\psi : \operatorname{Im}(\varphi) \rightarrow \mathcal{V} \times \operatorname{PGL}(4)$ be the left-inverse of $\varphi$, which exists since $\varphi$ is injective.  Then $\psi(v^{(1)}, v^{(2)}) = (v^{(1)}, g)$ for  $g \in \operatorname{PGL}(4)$.   
We need to show that $g$ is a continuous function of $(v^{(1)}, v^{(2)})$.
Note 
\begin{equation} \label{eq:solve-g}
\begin{cases}
P^{(1)} = P^{(2)} g \\
\bar{P}^{(1)} = \bar{P}^{(2)} g
\end{cases}
\end{equation}
as elements of $\mathcal{C}$.
Regard this as a linear system in $g$, where the invertibility of $g$ is temporarily relaxed. 
Namely, since the equalities in \eqref{eq:solve-g} are up to scale, we require that the $2 \times 2$ minors of 
$( \operatorname{vec}(P^{(1)}) \,\,   \operatorname{vec}(P^{(2)} g)    \big{)}$ and $(\operatorname{vec}(\bar{P}^{(1)}) \,\,   \operatorname{vec}(\bar{P}^{(2)} g)    )$ vanish, where $\operatorname{vec}$ denotes vectorization.  
Let $\mathcal{L}{(v^{(1)}, v^{(2)})} \subseteq \mathbb{P}(\mathbb{R}^{4 \times 4})$ be the system's null space.
Then consider $\mathcal{L}{(v^{(1)}, v^{(2)})} \cap \operatorname{PGL}(4)$.
On one hand, it is a singleton because of the uniqueness of  $\psi$.
On the other hand, it is a Zariski-open subset of $\mathcal{L}{(v^{(1)}, v^{(2)})}$.
The only possibility is that $\mathcal{L}{(v^{(1)}, v^{(2)})}$ is a single projective point, and it lies in $\operatorname{PGL}(4)$.
Therefore by \cite[pg.~3]{penrose1955generalized},  $g$ is a continuous function of $(v^{(1)}, v^{(2)})$.
This establishes properness.

To conclude we apply the quotient manifold theorem, noting 
$\dim(\mathcal{V}) = \dim(\mathcal{C}^{\times 2} \times (\mathbb{P}_{\mathbb{R}}^3)^{\times 7}) = 2 \cdot 11 + 7 \cdot 3 = 43$
and $\dim(\operatorname{PGL}(4)) = 15$.   \hfill $\square$

\section{Proof of Lemma~\mainref{lem:b-coords}}
Let the first map from $\mathcal{U}$ to $\mathcal{W}$ be $\psi$.  
Write $\psi = \psi_2 \circ \psi_1$ where $\psi_1 : \mathcal{U} \rightarrow \mathcal{V}$, $\mathcal{V}$ is the open subset of $\mathcal{C}^{\times 2} \times (\mathbb{P}_{\mathbb{R}}^3)^{\times 7}$ in (\mainref{eq:W-fundamental}), and $\psi_2: \mathcal{V} \rightarrow \mathcal{V}/\operatorname{PGL}(4) = \mathcal{W}$ is the quotient map.

We first check that $\psi$ is injective.  
Assume $\psi(b^{(1)}, \tilde{\Gamma}^{(1)}_1, \ldots, \tilde{\Gamma}^{(1)}_7) = \psi({b}^{(2)}, {{\tilde \Gamma}}^{(2)}_1, \ldots, {\tilde \Gamma}^{(2)}_7)$.  
Then $\psi_1(b^{(1)}, \tilde{\Gamma}^{(1)}_1, \ldots, \tilde{\Gamma}^{(1)}_7)$ and $\psi_1({b}^{(2)}, {\tilde \Gamma}^{(2)}_1, \ldots, {\tilde \Gamma}^{(2)}_7)$ lie in the same $\operatorname{PGL}(4)$-orbit.
In particular,
there is $g \in \operatorname{PGL}(4)$  with 
\begin{equation*}\label{eq:my-g-eqn}
(I \,\, 0) \, = \, (I \,\, 0) \, g   \quad \text{and}  \quad \mathbf{M}(b^{(2)}) \, = \, \mathbf{M}({b}^{(2)}) \, g  \quad  \text{in} \quad \mathbb{P}(\mathbb{R}^{3 \times 4}).
\end{equation*}
From the first equality,  \eqref{eq:nice-g-real-nice} holds. 
In the second equality, the $0$ entries imply $g_{41} = g_{42} = g_{43} = 0$.  
Comparing the $1$ entries, $g_{44} = 1$.  
Thus $g = I \in \operatorname{PGL}(4)$.

Next we check that $\psi$ is a submersion. 
Let $u = (b, \tilde \Gamma_1, \ldots, \tilde \Gamma_7)$ and $v = \psi_1(u)$.  
It is easy to see
\begin{multline} \label{eq:Dpsi1-joe}
\operatorname{im} D \psi_1 (u) = 0 \times \big{\{} \! \begin{scriptsize} \begin{pmatrix} 0 & \dot{b}_1 & \dot{b}_2 & \dot{b}_3 \\
\dot{b}_4 & \dot{b}_5 & \dot{b}_6 & \dot{b}_7 \\
0 & 0 & 0 & 0
\end{pmatrix}^{\!\! \perp} \end{scriptsize} \!\! : \dot{b} \in \mathbb{R}^7 \big{\}} \\ \times T(\mathbb{P}_{\mathbb{R}}^{3}, \tilde{\Gamma}_1) \times \ldots \times T(\mathbb{P}_{\mathbb{R}}^{3}, \tilde{\Gamma}_7),
\end{multline}
where we're identifying $T( \mathcal{C}, \mathbf{M}(b))$ with $T( \mathbb{P}(\mathbb{R}^{3 \times 4}), \mathbf{M}(b)) = (\mathbb{R}^{3 \times 4})^{\perp}$ (similarly to Section~\ref{subsec:projsp}). 
Here and below, ``$\perp$" denotes the component orthogonal to the center of the tangent space.  On the other hand, Proposition~\mainref{prop:quotient-smooth} gives
\begin{equation}\label{eq:tangent-to-orbit}
\operatorname{ker}(D \psi_2(v)) = T(\operatorname{PGL}(4) \cdot v, v).
\end{equation} 
A chain rule computation shows
\begin{multline}\label{eq:Dpsi2-joe}
T(\operatorname{PGL}(4) \cdot v, v) =  \Big{\{} {(} \tfrac{-1}{\sqrt{3}} ( (I \, 0) \dot{g})^{\perp}, \tfrac{-1}{\sqrt{2 + \| b \|_2^2}}(\mathbf{M}(b) \dot{g})^{\perp}, \\ \ldots, \tfrac{1}{ \| \tilde{\Gamma}_i \|_2^2} (\dot{g} \tilde{\Gamma}_i)^{\perp}, \ldots {)} \, : \, \dot{g} \in \mathbb{R}^{16} \Big{\}}.
\end{multline}
Comparing \eqref{eq:Dpsi1-joe}, \eqref{eq:tangent-to-orbit}, \eqref{eq:Dpsi2-joe}, one verifies 
\begin{equation*} \label{eq:zero-intersection}
\operatorname{im}(D \psi_1(u)) \cap \operatorname{ker}(D \psi_2(v)) = 0.
\end{equation*}
Further, clearly  $D \psi_1(u)$ has rank $\dim(\mathcal{U}) = \dim(\mathcal{W})$.  We deduce 
$D\psi(u) = D \psi_2(v) \circ D\psi_1(u)$ is surjective.

By \cite[Thm.~4.14]{lee2013smooth},  $\psi$ is a 
smooth parameterization of an open subset of $\mathcal{W}$.
Further, $\psi(\mathcal{U}) \subseteq \mathcal{W}$ is dense because $\operatorname{PGL}(4) \cdot \psi_1(\mathcal{U}) \subseteq \mathcal{V}$ is dense (details omitted).
Analogous considerations apply to the other maps in the lemma.
Finally,  their images do cover $\mathcal{W}$.  
The main check is that for $(P, \bar{P}) \in \mathcal{C}^{\times 2}$ with distinct centers, we can find $g \in \operatorname{PGL}(4)$ so $P g = (I \, 0)$ and $\bar{P} g$ has $(0 \, 0 \, 0 \, 1)$ as a row and has an entry of $1$ in its first three columns (details omitted).  \hfill $\square$

\section{Proof of Proposition~\mainref{prop:world-lifting}}

Take $x \in \mathcal{X} \setminus \operatorname{disc}(\mathscr{M}, \mathbf{L})$.  
We need to show $x \in \mathcal{X} \setminus \operatorname{ill}(\mathscr{M}, \mathcal{{X}})$.
Let $w \in \Phi^{-1}(x) \in \mathcal{W}$ and $y = \Psi(w)$.  
By the first line of \eqref{eq:epi-general},  $y \in \mathcal{Y} \cap \operatorname{ker} \mathbf{L}(x)$.  
Clearly, $w \in \Phi^{-1}(x) \cap \Psi^{-1}(y)$.  
So the smooth lifting property applies.
From $\Phi \circ h = \operatorname{id}_{\mathcal{X}_0}$ and $h(x) = w$, 
the chain rules gives $D \Phi (w) \circ D h (x) = \operatorname{id}$.
In particular, $D \Phi(w)$ is invertible.
Since $w \in \Phi^{-1}(x)$ was arbitrary, $x \notin \operatorname{ill}(\mathscr{M}, \mathcal{{X}})$ by Definition~\ref{def:ill-X}.

\section{Proof of Proposition~\mainref{prop:formula-E}}

The condition number of $w = (\mathbf{R}, \hat{\mathbf{T}}, \Gamma_1, \ldots, \Gamma_5)$ is 
\begin{equation}\label{eq:my-cond-supp}
\operatorname{cond}(\mathscr{M}, w) = \sigma_{\max}(D \Psi \circ (D \Phi_2 \circ D \Phi_1)^{-1}(w)),
\end{equation}
where $\Phi_1, \Phi_2, \Psi$ are the maps (\mainref{eq:E-Phi1-main}), (\mainref{eq:E-Phi2-main}), (\mainref{eq:E-psi-map}), assuming that the inverse in \eqref{eq:my-cond-supp} exists.

The differentials map between the following spaces:
\begin{small}
\begin{align*}
& D\Phi_1(w) : T(\mathcal{W}, w) \rightarrow T((\mathbb{R}^3 \times \mathbb{R}^3)^{\times 5}, \Phi_1(w)), \\[4pt]
& D\Phi_2( \Phi_1(w) ) : T((\mathbb{R}^3 \times \mathbb{R}^3)^{\times 5}, \Phi_1(w)) \rightarrow T((\mathbb{R}^2 \times \mathbb{R}^2)^{\times 5}, \Phi(w)), \\[4pt]
& D\Psi(w) : T(\mathcal{W}, w)  \rightarrow T(\mathcal{Y}, \Psi(w)).
\end{align*}
\end{small}
\!\!\!The tangent spaces are described in Section~\ref{sec:background-mflds}, using standard facts and Lemma~\mainref{lem:double}.
We must choose ordered bases for these tangent spaces,  to write down matrix formulas for the differentials.

For $T(\mathcal{W}, w)$, fix the ordered basis
\begin{equation}\label{eq:ordered-basis-W-1}
e^{(1)}_1, \, e^{(1)}_2, \, \ldots, \, e^{(5)}_3, \, [e_1]_{\times} \mathbf{R}, \, [e_2]_{\times} \mathbf{R}, \, [e_3]_{\times} \mathbf{R}, \, \hat{\mathbf{T}}_1^{\perp}, \, \hat{\mathbf{T}}_2^{\perp},
\end{equation}
where $e^{(i)}_{\bullet}$ are standard basis vectors in the $i$th copy of $\mathbb{R}^3$.
For $T((\mathbb{R}^3 \times \mathbb{R}^3)^{\times 5}, \Phi_1(w))$, fix the ordered basis 
\begin{small}
\begin{equation*}
e^{(1,1)}_{1}, \, e^{(1,1)}_{2}, \, e^{(1,1)}_{3}, \, e^{(2,1)}_1, \, \ldots, \, e^{(5,1)}_3, \, e^{(1,2)}_1, \, e^{(1,2)}_2, \, \ldots, \, e^{(5,2)}_3,
\end{equation*}
\end{small}
\!\!\!\!\! where $e^{(i,j)}_{\bullet}$ are standard basis vectors in the copy of $\mathbb{R}^3$ corresponding to the $i$th world point and $j$th camera.
For $T((\mathbb{R}^2 \times \mathbb{R}^2)^{\times 5}, \Phi(w))$, fix the ordered orthornormal basis
\begin{equation*}
e^{(1,1)}_{1}, \, e^{(1,1)}_{2}, \, e^{(2,1)}_1, \, \ldots, \, e^{(5,1)}_2, \, e^{(1,2)}_1, \, e^{(1,2)}_2, \, \ldots, \, e^{(5,2)}_2,
\end{equation*}
where $e^{(i,j)}_{\bullet}$ are standard basis vectors in the copy of $\mathbb{R}^2$ corresponding to the $i$th world point and $j$th camera.  
For $T(\mathcal{Y}, \Psi(w))$, fix the ordered non-orthonormal basis 
\eqref{eq:non-orthonormal}.

Now, $D \Phi_1(w)$ is represented by the $30 \times 20$ Jacobian matrix in Fig.~\ref{fig:my-jac1-joe}.  
Next, $D\Phi_2(\Phi_1(w))$ is the $30 \times 20$ matrix
\begin{equation}\label{eqn:jac-2}
\begin{pmatrix}
D \pi(\Gamma_1) & & & & & \\
& \ddots & & & & \\
& & D \pi(\Gamma_5) & & & \\
& & & D \pi(\bar{\Gamma}_1) & & \\
& & & & \ddots & \\
& & & & & D \pi(\bar{\Gamma}_5)
\end{pmatrix},
\end{equation}
where $\pi : \mathbb{R}^3 \dashrightarrow \mathbb{R}^2$ is orthographic projection, explicitly
\begin{equation}\label{eq:Dpi-supp}
D \pi(x,y,z) = \begin{pmatrix} \tfrac{1}{z} & 0 & \tfrac{-x}{z^2} \\[2pt] 0 & \tfrac{1}{z} & \tfrac{-y}{z^2} \end{pmatrix},
\end{equation}
and $\bar{\Gamma}_i := \mathbf{R} \Gamma_i + \hat{\mathbf{T}}$ for $i = 1, \ldots, 5$.  
Last, $D\Psi(w)$ is simply
\begin{equation} \label{eq:trivial}
\begin{pmatrix}
I_{5 \times 5} & 0_{5 \times 15}
\end{pmatrix},
\end{equation}
by the choice of non-orthonormal basis for $T(\mathcal{Y}, \Psi(w))$.
To convert an orthonormal basis, as needed for singular values, we use Lemma~\ref{lem:linear-alg} below.  
Specifically, we left-multiply \eqref{eq:trivial} by $G^{1/2}$, where $G \in \mathbb{R}^{5 \times 5}$ is the Gram matrix in Fig.~\ref{fig:joe-k}.

Putting these expressions for the differentials into \eqref{eq:my-cond-supp}, Proposition~\mainref{prop:formula-E} is finished. \hfill $\square$

\section{Proof of Proposition~\mainref{prop:formula-F}}

The condition number of $w = (b, \Gamma_1, \ldots, \Gamma_7)$ is
\begin{equation}\label{eq:my-cond-supp-2}
\operatorname{cond}(\mathscr{M}, w) = \sigma_{\max}(D \Psi \circ (D \Phi_2 \circ D \Phi_1)^{-1}(w)),
\end{equation}
where $\Phi_1, \Phi_2, \Psi$ are the maps (\mainref{eq:Phi1-7pt}), (\mainref{eq:Phi2-7pt}), (\mainref{eq:fund-output}), assuming the inverse exists.  

The differentials map between the following spaces:
\begin{small}
\begin{align*}
& D\Phi_1(w) : T(\mathcal{W}, w) \rightarrow T((\mathbb{R}^3 \times \mathbb{R}^3)^{\times 7}, \Phi_1(w)), \\[4pt]
& D\Phi_2( \Phi_1(w) ) : T((\mathbb{R}^3 \times \mathbb{R}^3)^{\times 7}, \Phi_1(w)) \rightarrow T((\mathbb{R}^2 \times \mathbb{R}^2)^{\times 7}, \Phi(w)), \\[4pt]
& D\Psi(w) : T(\mathcal{W}, w)  \rightarrow T(\mathcal{Y}, \Psi(w)).
\end{align*}
\end{small}
\!\!\!The tangent spaces are described in Section~\ref{sec:background-mflds}, using well-known facts and Lemma~\mainref{lem:b-coords}.

For $T(\mathcal{W}, w)$, fix the ordered basis 
\begin{equation}\label{eq:myworld-7}
e^{(1)}_{1}, e^{(1)}_2, \ldots, e^{(7)}_3, e_1, \ldots, e_7
\end{equation}
where $e^{(i)}_{\bullet}$ are standard basis vectors in the $i$th copy of $\mathbb{R}^3$ and $e_{\bullet}$ are standard basis vectors in $\mathbb{R}^{7}$.  For $T((\mathbb{R}^3 \times \mathbb{R}^3)^{\times 7}, \Phi(w))$ and $T((\mathbb{R}^2 \times \mathbb{R}^2)^{\times 7}, \Phi(w))$, fix ordered bases analogously to what we chose in the proof of Proposition~\mainref{prop:formula-E}. 
For $T(\mathcal{Y}, \Psi(w))$, fix the non-orthonormal basis \eqref{eq:easy-basis}.

Then, $D\Phi_1(w)$ is represented by the $42 \times 28$ Jacobian matrix in Fig.~\ref{eq:DPhi1-7pt-supp}.  Next, $D\Phi_2(\Phi_1(w))$ is the $42 \times 28$ Jacobian
\begin{equation*}
\begin{pmatrix}
D \pi(\Gamma_1) & & & & & \\
& \ddots & & & & \\
& & D \pi(\Gamma_7) & & & \\
& & & D \pi(\bar{\Gamma}_1) & & \\
& & & & \ddots & \\
& & & & & D \pi(\bar{\Gamma}_7)
\end{pmatrix},
\end{equation*}
where $\bar{\Gamma}_i = \mathbf{M}(b) \Gamma_i$ with $\mathbf{M}(b)$ as in Lemma~\mainref{lem:b-coords} (regarded as being in $\mathbb{R}^{3 \times 4}$) and $D\pi$ is given by \eqref{eq:Dpi-supp}.  
Also, $D\Psi(w)$ is 
\begin{equation} \label{eq:trivial-2}
\begin{pmatrix}
I_{7 \times 7} & 0_{7 \times 21}
\end{pmatrix},
\end{equation}
by choice of non-orthonormal basis for $T(\mathcal{Y}, \Psi(w))$.  To convert this to an orthonormal basis, we left-multiply \eqref{eq:trivial-2} by $G^{1/2}$, where $G \in \mathbb{R}^{7 \times 7}$ is the Gram matrix \eqref{eq:gram-7pt} below. 

This finishes the proof.  Note that similar formulas apply if the world scene is expressed in terms of other charts from Lemma~\mainref{lem:b-coords}.
\hfill $\square$

\section{Proof of Theorem~\mainref{thm:illposed-world}}

Let $w = (\mathbf{R}, \hat{\mathbf{T}}, \Gamma_1, \ldots, \Gamma_5) \in \mathcal{W}$ be a world scene.  
Let $\delta w = (\dot{\Gamma}_1, \ldots, \dot{\Gamma}_5, \dot{s}, \dot{\alpha})$ be an element of the tangent space $T(\mathcal{W}, w)$, expressed in terms of the ordered basis \eqref{eq:ordered-basis-W-1}, where $\dot{\Gamma}_i \in \mathbb{R}^3$ ($i=1, \ldots, 5$), $\dot{s} \in \mathbb{R}^3$ and $\dot{\alpha} \in \mathbb{R}^2$.
The task is to characterize those $w$ for which
\begin{equation}\label{eq:supp-key-eqn-joe}
D\Phi(w) \delta w \, = \, 0
\end{equation}
has a nonzero solution in $\delta w$.
Below we will use $D\Phi(w) = D\Phi_2(\Phi_1(w)) \circ D\Phi_1(w)$ whence \eqref{eq:supp-key-eqn-joe} is equivalent to 
\begin{equation}\label{eq:supp-key-eqn-joe2}
D\Phi_1(w) \delta w \,\, \in \, \operatorname{ker} D \Phi_2(\Phi_1(w)),
\end{equation}
and also the explicit matrix formulas from Proposition~\mainref{prop:formula-E}.

By \eqref{eq:Dpi-supp}, if $z \neq 0$ then $D \pi(x,y,z)$ has rank $2$ and its kernel is spanned by $(x \,\, y \,\, z)^{\top}$.  
Therefore \eqref{eqn:jac-2} implies
$$
\operatorname{ker} D\Phi_2(\Phi_1(w)) = \mathbb{R} \Gamma_1 \oplus \cdots \oplus \mathbb{R} \Gamma_5 \oplus \mathbb{R}(\mathbf{R} {\Gamma}_1 + \hat{\mathbf{T}}) \oplus \cdots 
$$
By the formula for $D\Phi_1(w)$ in Fig.~\ref{fig:my-jac1-joe},  
\eqref{eq:supp-key-eqn-joe2} is equivalent to
\begin{equation} \label{eq:good-stuff}
\begin{cases}
   \dot{\Gamma}_i  = \lambda_i \Gamma_i  \\[2pt]
 \mathbf{R} \dot{\Gamma}_i + [\dot{s}]_{\times} \mathbf{R} \Gamma_i + \hat{\mathbf{T}}_{\ast}^{\perp} = \mu_i (\mathbf{R} \Gamma_i + \hat{\mathbf{T}}) 
  \end{cases} 
\end{equation}
for $i=1, \ldots, 5$, some $\lambda, \mu \in \mathbb{R}^5$ and $\hat{\mathbf{T}}^{\perp}_{\ast} := \dot{\alpha}_1 \hat{\mathbf{T}}^{\perp}_1 + \dot{\alpha}_2 \hat{\mathbf{T}}^{\perp}_2$.
We need to answer: For which $\mathbf{R}, \hat{\mathbf{T}}, \Gamma_i$, does \eqref{eq:good-stuff} admit a solution where $\dot{\Gamma}_i, \dot{s}, \dot{\alpha}$ are not all $0$?

Substituting the top line of \eqref{eq:good-stuff} into the bottom line gives
\begin{equation}\label{eq:my-new}
\lambda_i \mathbf{R} \Gamma_i + [\dot{s}]_{\times} \mathbf{R} \Gamma_i + \hat{\mathbf{T}}^{\perp}_{\ast} = \mu_i \mathbf{R} \Gamma_i + \mu_i \hat{\mathbf{T}},
\end{equation}
for $i = 1, \ldots, 5$.  This eliminates $\dot{\Gamma}_i$.
We now need to characterize when  \eqref{eq:my-new} admits a nonzero solution in $\lambda, \mu, \dot{s}, \dot{\alpha}$.  (Proof of equivalence: Assume \eqref{eq:good-stuff} has a solution in which $\dot{\Gamma}_i$, $\dot{s}$, $\dot{\alpha}$ are not all zero.  Then \eqref{eq:my-new} has a nonzero solution, as 
$\dot{\Gamma}_i \neq 0$ if and only if $\lambda_i \neq 0$ 
since $(\Gamma_i)_3 \neq 0$.  Conversely, assume \eqref{eq:my-new} has a nonzero solution.
Then one of $\lambda_i$, $\dot{s}$, $\hat{\mathbf{T}}^{\perp}_{\ast}$ must be nonzero by 
$(\mathbf{R} \Gamma_i + \hat{\mathbf{T}})_3 \neq 0$, so that \eqref{eq:good-stuff} has a nonzero solution.)

Next, for convenience we replace $\lambda_i-\mu_i$ by $\lambda_i$.  This corresponds to an invertible linear change of variables in \eqref{eq:my-new}, and so preserves  rank. Now the linear system reads
\begin{equation}\label{eq:my-new-2}
\lambda_i \mathbf{R} \Gamma_i + [\dot{s}]_{\times} \mathbf{R} \Gamma_i + \hat{\mathbf{T}}^{\perp}_{\ast} =  \mu_i \hat{\mathbf{T}}.
\end{equation}

Next, we rotate the world scene to assume $\hat{\mathbf{T}}^{\perp}_1 = e_1$, $\hat{\mathbf{T}}^{\perp}_2 = e_2$, $\hat{\mathbf{T}} = e_3$ with no loss of generality.
To be more precise, let us left-multiply \eqref{eq:my-new-2} by $\mathbf{Q} \in \operatorname{SO}(3)$ such that 
$\mathbf{Q} \mathbf{T}^{\perp}_1 = e_1$, $\mathbf{Q} \mathbf{T}^{\perp}_2 = e_2$, and $\mathbf{Q} \hat{\mathbf{T}} = e_3$, and then replace 
$ \mathbf{Q} [\dot{s}]_{\times} \mathbf{Q}^{\top}$ by $[\dot{s}]_{\times}$.  This gives an invertible linear transformation of \eqref{eq:my-new-2} and so preserves rank; we obtain
\begin{equation} \label{eq:my-newest}
\lambda_i \mathbf{Q} \mathbf{R} \Gamma_i + [\dot{s}]_{\times} \mathbf{Q} \mathbf{R} \Gamma_i + \dot{\alpha}_1 e_1 + \dot{\alpha}_2 e_2 =  \mu_i e_3.
\end{equation}

\noindent Write $\bar{\Gamma}_i := \mathbf{Q} \mathbf{R} \Gamma_i$ as shorthand.
The first two coordinates of \eqref{eq:my-newest} are
\begin{equation}\label{eq:my-new-3}
 \lambda_i \! \begin{pmatrix} (\bar{\Gamma}_i)_1 \\[1pt] (\bar{\Gamma}_i)_2 \end{pmatrix} \, + \, \begin{pmatrix} -\dot{s}_3 (\bar{\Gamma}_i)_2 + \dot{s}_2 (\bar{\Gamma}_i)_3  \\[1pt] \dot{s}_3 (\bar{\Gamma}_i)_1 - \dot{s}_1 (\bar{\Gamma}_i)_3  \end{pmatrix} 
   \, + \, \dot{\alpha} \,\, = \,\, 0.
\end{equation}
Clearly \eqref{eq:my-new-3} has a nonzero solution if and only if \eqref{eq:my-newest} does.  
In this way $\mu$ is eliminated.

Next, left-multiplying \eqref{eq:my-new-3} by $\begin{pmatrix} -(\bar{\Gamma}_i)_2 & (\bar{\Gamma}_i)_1 \end{pmatrix}$   eliminates $\lambda$:
\begin{align} \label{eq:amazing-0}
    & \dot{s}_3((\bar{\Gamma}_i)_1^2 + (\bar{\Gamma}_i)_2^2) \, - \, \dot{s}_2 (\bar{\Gamma}_i)_2 (\bar{\Gamma}_i)_3 \, - \, \dot{s}_1 (\bar{\Gamma}_i)_1 (\bar{\Gamma}_i)_3 \, \nonumber \\[1.15pt] 
    &  \hspace{2.5cm} + \dot{\alpha}_1 (\bar{\Gamma}_i)_2 \, - \, \dot{\alpha}_2 (\bar{\Gamma}_i)_1 \, = \, 0.
\end{align}
That is, \eqref{eq:amazing-0} has a nonzero solution if and only if \eqref{eq:my-new-3} does.
(Proof:   Case 1 -- $\begin{pmatrix} (\bar{\Gamma}_i)_1 & (\bar{\Gamma}_i)_2 \end{pmatrix}$ is zero for some $i$.  Then \eqref{eq:my-new-3} has a solution where $\lambda_i = 1$ and all other variables are $0$ ($\dot{s}=0$, $\dot{
\alpha
}=0$, and $\lambda_j=0$ for $j \neq i$).  On the other hand, the $i$th equation in \eqref{eq:amazing-0} is identically $0$. Therefore \eqref{eq:amazing-0} admits a nonzero solution too, being homogeneous and rank-deficient.  Case 2 -- $\begin{pmatrix} (\bar{\Gamma}_i)_1 & (\bar{\Gamma}_i)_2 \end{pmatrix}$ is nonzero for all $i$.  Then $\begin{pmatrix} (\bar{\Gamma}_i)_1 & (\bar{\Gamma}_i)_2 \end{pmatrix}^{\top}$ and $\begin{pmatrix} -(\bar{\Gamma}_i)_2 & (\bar{\Gamma}_i)_1 \end{pmatrix}^{\top}$ form an orthogonal basis for $\mathbb{R}^2$ for all $i$. 
We recognize \eqref{eq:amazing-0} as \eqref{eq:my-new-3} in the $\begin{pmatrix} (-\bar{\Gamma}_i)_2 & (\bar{\Gamma}_i)_1 \end{pmatrix}^{\top}$ component.  
Thus if \eqref{eq:amazing-0} admits a nonzero solution, then \eqref{eq:my-new-3} does too, as there would be a unique $\lambda_i$ ensuring \eqref{eq:my-new-3} holds in the $\begin{pmatrix} (\bar{\Gamma}_i)_1 & (\bar{\Gamma}_i)_2 \end{pmatrix}^{\top}$ component.   Conversely if \eqref{eq:my-new-3} has a nonzero solution, then $\dot{s}$ and $\dot{\alpha}$ cannot both be $0$, so \eqref{eq:amazing-0} has a nonzero solution.)

At this point, we have reduced the $20 \times 20$ linear system \eqref{eq:supp-key-eqn-joe} to the $5 \times 5$ linear system \eqref{eq:amazing-0}. To finish, we interpret \eqref{eq:amazing-0} geometrically. Rewrite it as
\begin{equation}\label{eq:world-Gammai}
\begin{pmatrix}
\Gamma_i \\
1
\end{pmatrix}^{\! \top} \! Q(\dot{s}, \dot{\alpha}) \begin{pmatrix}
\Gamma_i \\
1
\end{pmatrix} = 0
\end{equation}
for $i=1, \ldots, 5$, where 
\begin{equation*} \label{eq:my-Q}
Q(\dot{s}, \dot{\alpha}) := \frac{1}{2} \begin{pmatrix} \mathbf{Q} \mathbf{R} & 0 \\
0 & 1
\end{pmatrix}^{\!\!\top} \! \! \begin{pmatrix}
2\dot{s}_3 & 0 & -\dot{s}_1 & -\dot{\alpha}_2 \\
0 & 2\dot{s}_3 & -\dot{s}_2 & \dot{\alpha}_1 \\
-\dot{s}_1 & -\dot{s}_2 & 0 & 0 \\
-\dot{\alpha}_2 & \dot{\alpha}_1 & 0 & 0
\end{pmatrix} \!\! \begin{pmatrix} \mathbf{Q} \mathbf{R} & 0 \\
0 & 1
\end{pmatrix}\!.
\end{equation*}
Clearly, \eqref{eq:world-Gammai} holds for $(\dot{s}, \dot{\alpha}) \neq 0$  means that the world points $\Gamma_i$ all lie on the quadric $\mathcal{Q}$ in $\mathbb{R}^3$ defined  by ${Q}(\dot{s}, \dot{\alpha})$.
By inspection, the following line in $\mathbb{P}_{\mathbb{R}}^3$ must lie on $\mathcal{Q}$:
\begin{equation}
{\huge{\{} }\begin{pmatrix} \mathbf{Q} \mathbf{R} & 0 \\
0 & 1
\end{pmatrix}^{\!\! \top} \begin{pmatrix} 0 \\ 0 \\ a \\ b \end{pmatrix} : a, b \in \mathbb{R} \,  {\huge{\}}},
\end{equation}
In $\mathbb{R}^3$, this equals the baseline $\operatorname{span}(\mathbf{R}^{\top} \hat{\mathbf{T}})$.
Futhermore, if $\dot{s}_3 \neq 0$, then $\mathcal{Q} \cap \mathcal{N}$ must be a circle (possibly with radius $0$), where $\mathcal{N}$ is any affine plane perpendicular to the baseline.  On the other hand, if $\dot{s} = 0$, then $\mathcal{Q} \cap \mathcal{N}$ must be a line (with at most one exception for which $\mathcal{Q} \cap \mathcal{N}$ is a plane). 
Conversely, all such quadrics containing the baseline  with circular or linear perpendicular cross-sections must come from a matrix of the form $Q(\dot{s},\dot{\alpha})$ for some $(\dot{s}, \dot{\alpha}) \neq 0$.
Altogether, Theorem~\mainref{thm:illposed-world} follows. \hfill $\square$

\section{Proof of Theorem~\mainref{thm:F-ill-posed-world}}

Let $w = (b, \Gamma_1, \ldots, \Gamma_7) \in \mathcal{W}$.
Let $\delta w = (\dot{b}, \dot{\Gamma}_1, \ldots, \dot{\Gamma}_7)$  represent an element of $T(\mathcal{W}, w)$, expressed in terms of the ordered basis \eqref{eq:myworld-7} where $\dot{b} \in \mathbb{R}^7$ and $\dot{\Gamma}_i \in \mathbb{R}^3$. 
We need to characterize world scenes such that 
\begin{equation} \label{eq:membership-2}
D \Phi_1(w) \delta w \, \in \, \operatorname{ker} D \Phi_2(\Phi_1(w))
\end{equation}
for some nonzero $\delta w$, where the Jacobians $D \Phi_1$ and $D \Phi_2$ are as in Proposition~\mainref{prop:formula-F}.

Like in the proof of Theorem~\mainref{thm:illposed-world}, 
$$
\operatorname{ker} D\Phi_2 (\Phi_1(w)) = \mathbb{R} \Gamma_1 \oplus \cdots \oplus \mathbb{R} \Gamma_7 \oplus \mathbb{R} (\mathbf{M}(b) (\Gamma_i; 1)) \oplus \cdots 
$$
By the form of $D \Phi_1(w)$ in Fig.~\ref{eq:DPhi1-7pt-supp}, \eqref{eq:membership-2} is equivalent to
\begin{equation} \label{eq:good-stuff3}
\begin{cases}
   \dot{\Gamma}_i  = \lambda_i \Gamma_i  \\[2pt]
 D \mathbf{M}(b)(\dot{b}) (\Gamma_i; 1)  + \mathbf{M}(b) (\dot{\Gamma}_i; 0) = \mu_i (\mathbf{M}(b)(\Gamma_i; 1)) 
  \end{cases} 
\end{equation}
for $i=1, \ldots, 7$, some $\lambda, \mu \in \mathbb{R}^7$ and 
\begin{equation*}
D \mathbf{M}(b)(\dot{b}) = \begin{pmatrix} 0 & \dot{b}_1 & \dot{b}_2 & \dot{b}_3 \\ \dot{b}_4 & \dot{b}_5 & \dot{b}_6 & \dot{b}_7 \\ 0 & 0 & 0 & 0 \end{pmatrix}.
\end{equation*}
Equating third coordinates in the bottom line of \eqref{eq:good-stuff3}, it must be that $\mu = 0$ as $(\mathbf{M}(b)(\Gamma_i; 1))_3 \neq 0$ for $i=1, \ldots, 7$.
Substituting the top line of \eqref{eq:good-stuff3} into the bottom line gives
\begin{equation} \label{eq:this-amazing}
\begin{pmatrix}
0 & \dot{b}_1 & \dot{b}_2 & \dot{b}_3 \\
\dot{b}_4 & \dot{b}_5 & \dot{b}_6 & \dot{b}_7
\end{pmatrix} \begin{pmatrix} \Gamma_i  \\ 1 \end{pmatrix} + \lambda_i \begin{pmatrix} 1 & b_1 & b_2 \\ b_4 & b_5 & b_ 6 \end{pmatrix} {\Gamma}_i = 0
\end{equation}
for $i = 1, \ldots, 7$.  
The task is to characterize for which $b$ and $\Gamma_i$ does \eqref{eq:this-amazing} has a nonzero solution in $\dot{b}, \lambda$ (the justification that \eqref{eq:this-amazing} is equivalent is to \eqref{eq:good-stuff3} is like in the proof of Theorem~\mainref{thm:F-ill-posed-world}).

Next, we eliminate $\lambda$, by multiplying \eqref{eq:this-amazing} on the left by 
\begin{equation*} \label{eq:bX-perp}
    \begin{pmatrix}
    -b_4(\Gamma_i)_1 - b_5(\Gamma_i)_2 - b_6(\Gamma_i)_3 \\
    (\Gamma_i)_1 + b_1(\Gamma_i)_2 + b_2(\Gamma_i)_3
    \end{pmatrix}^{\!\! \top}
\end{equation*}
which kills $\begin{pmatrix} 1 & b_1 & b_2 \\ b_4 & b_5 & b_6 \end{pmatrix} \Gamma_i$.  The resulting linear system in $\dot{b}$ is
\begin{equation} \label{eq:i-like-this-1}
    \begin{pmatrix}
    -b_4(\Gamma_i)_1(\Gamma_i)_2 - b_5(\Gamma_i)_2^2 - b_6(\Gamma_i)_2(\Gamma_i)_3 \\[3pt]
    -b_4(\Gamma_i)_1(\Gamma_i)_3 - b_5(\Gamma_i)_2(\Gamma_i)_3 - b_6(\Gamma_i)_3^2 \\[3pt]
    -b_4(\Gamma_i)_1 - b_5(\Gamma_i)_2 - b_6(\Gamma_i)_3 \\[3pt]
    (\Gamma_i)_1^2 + b_1(\Gamma_i)_1 (\Gamma_i)_2 + b_2(\Gamma_i)_1 (\Gamma_i)_3 \\[3pt]
    (\Gamma_i)_1(\Gamma_i)_2 + b_1(\Gamma_i)_2^2 + b_2(\Gamma_i)_2(\Gamma_i)_3 \\[3pt]
    (\Gamma_i)_1(\Gamma_i)_3 + b_1(\Gamma_i)_2(\Gamma_i)_3 + b_2(\Gamma_i)_3^2 \\[3pt]
    (\Gamma_i)_1 + b_1(\Gamma_i)_2 + b_2(\Gamma_i)_3
    \end{pmatrix}^{\!\! \top}  \dot{b} \,\, = \,\, 0 
\end{equation}
for $i = 1, \ldots, 7$.  Then \eqref{eq:i-like-this-1} has a nonzero solution if and only if \eqref{eq:this-amazing} does (the proof of this is similar as for Theorem~\mainref{thm:F-ill-posed-world}).

Now, let us interpret \eqref{eq:i-like-this-1} geometrically.  Rewrite it \nolinebreak as
\begin{equation} \label{eq:they-pass-thru}
\begin{pmatrix} 
\Gamma_i \\
1
\end{pmatrix}^{\!\! \top} 
Q(\dot{b}) \begin{pmatrix} 
\Gamma_i \\
1
\end{pmatrix} = 0 
\end{equation}
for $i = 1, \ldots, 7$, where 
$Q(\dot{b}) := \dot{b}_1 Q_1 + \ldots + \dot{b}_7 Q_7$ and
\begin{footnotesize}
\begin{align}
&Q_1 = \begin{pmatrix} 0 & -b_4 & 0 & 0 \\ -b_4 & -2 b_5 & -b_6 & 0 \\ 0 & -b_6 & 0 & 0 \\ 0 & 0 & 0 & 0  \end{pmatrix}, \hspace{1em} 
 Q_2 = \begin{pmatrix} 0 & 0 & -b_4 &  0 \\ 0 &  0 &  -b_5 & 0 \\  -b_4 & -b_5 &  -2b_6 &  0 \\ 0 & 0 &  0 & 0 \end{pmatrix} \nonumber \\[4pt]
&Q_3 = \begin{pmatrix} 0 & 0 & 0 & -b_4 \\ 0 &  0 &  0 & -b_5 \\ 0 &  0 &  0 & -b_6 \\ -b_4 & -b_5 & -b_6 & 0 \end{pmatrix}, \hspace{1em} 
 Q_4 = \begin{pmatrix} 2 & b_1 & b_2 & 0 \\ b_1 & 0 & 0 & 0 \\ b_2 & 0 & 0 & 0 \\ 0 & 0 & 0 & 0 \end{pmatrix} \nonumber \\[4pt]
&Q_5 = \begin{pmatrix} 
0 & 1 & 0 & 0 \\ 1 & 2b_1 & b_2 & 0 \\ 0 & b_2 & 0 & 0 \\ 0 & 0 & 0 & 0
\end{pmatrix},
\hspace{1em} Q_6 = \begin{pmatrix}
0 & 0 & 1 & 0 \\ 0 &  0 & b_1 &  0 \\ 1 & b_1 & 2b_2 & 0 \\ 0 & 0 & 0 & 0
\end{pmatrix} \nonumber \\[4pt]
& Q_7 = \begin{pmatrix}
0 & 0 & 0 & 1 \\ 0 & 0 & 0 & b1 \\ 0 & 0 &  0 & b2 \\ 1 & b1 & b2 & 0
\end{pmatrix}. \nonumber
\end{align}
\end{footnotesize}
\!\!\!\!\!\! It remains to prove that $Q_1, \ldots, Q_7$ form a basis for the subspace of $\operatorname{Sym}^2(\mathbb{R}^4)$ corresponding to all quadrics $\mathcal{Q}$ containing the baseline.  
Clearly this will imply Theorem~\mainref{thm:F-ill-posed-world}, because \eqref{eq:they-pass-thru} means that $\mathcal{Q}$ passes through $\Gamma_1, \ldots, \Gamma_7$.

Note that the subspace  of $\operatorname{Sym}^2(\mathbb{R}^4)$ in question has dimension $7$, since a quadric containing a given line imposes $3$ linearly independent conditions.
Further, by direct calculation (Cramer's rule) the baseline of the scene is 
\begin{equation}\label{eq:my-baseline-7}
\operatorname{span} \begin{pmatrix} 
b_2b_5 - b_1b_6 & -b_2b_4 + b_6 & b_1b_4 - b_5 
\end{pmatrix}^{\!\top} \,\, \subseteq \,\,\, \mathbb{R}^3.
\end{equation}
It's easy to check the quadrics corresponding to $Q_i$ contain \eqref{eq:my-baseline-7}.  
So we are done if $Q_1, \ldots, Q_7$ are linearly independent.

Assume $Q(\dot{b}) = 0$.  Comparing top-left entries, $\dot{b}_4 =0$.
Comparing bottom rows and rightmost columns,
\begin{equation}\label{eq:quick-eqn}
\begin{pmatrix} 1 & b_1 & b_2 \\ 
b_4 & b_5 & b_6
\end{pmatrix} \begin{pmatrix}
\dot{b}_7 \\[1pt] -\dot{b}_3
\end{pmatrix} = 0.
\end{equation}
It implies $\dot{b}_3 = \dot{b}_7 = 0$, because $\mathbf{M}(b)(1:2,1:3)$ must be rank-$2$ if $\mathbf{M}(b)$ is rank-$3$.
Next define 
\begin{equation}\label{eq:my-Qtilde}
\tilde{Q} := \dot{b}_1 Q_1 + \dot{b}_5 Q_5 = -\dot{b}_2 Q_2 - \dot{b}_6 Q_6.
\end{equation}
Given the supports of $Q_i$, $\tilde{Q} = \lambda (e_2 e_3^{\top} + e_3 e_2^{\top})$ for some $\lambda \in \mathbb{R}$.  
Further, \eqref{eq:my-Qtilde} can be reshaped to
\begin{equation}\label{eq:my-dotb15}
\begin{pmatrix} \dot{b}_5 & -\dot{b}_1 \end{pmatrix} \begin{pmatrix} 1 & 2b_1 & b_2 \\ b_4 & 2b_5 & b_6 \end{pmatrix} = \begin{pmatrix} 0 & 0 & \lambda \end{pmatrix}
\end{equation}
and
\begin{equation}\label{eq:my-dotb26}
\begin{pmatrix} \dot{b}_6 & -\dot{b}_2 \end{pmatrix} \begin{pmatrix} 1 & b_1 & 2b_2 \\ b_4 & b_5 & 2b_6 \end{pmatrix} = \begin{pmatrix} 0 & \lambda & 0 \end{pmatrix}.
\end{equation}
If $\lambda \neq 0$, we deduce that $\begin{pmatrix}1 & b_1 & b_2 \\ b_4 & b_5 & b_6 \end{pmatrix}$ has row-space $\operatorname{span}(e_2, e_3)$, which is absurd; hence $\lambda = 0$. 
Then similarly to \eqref{eq:quick-eqn}, \eqref{eq:my-dotb15} implies $\dot{b}_1 = \dot{b}_5 = 0$ and \eqref{eq:my-dotb26} implies $\dot{b}_2 = \dot{b}_6 = 0$.  We have shown that $Q(\dot{b}) = 0$ implies $\dot{b} = 0$ as wanted. 

The proof of Theorem~\mainref{thm:illposed-world} is complete.  We remark that a  similar proof shows the analog of Theorem~\mainref{thm:F-ill-posed-world}, when the world scene lives in a different chart of Lemma~\mainref{lem:b-coords}.  \hfill $\square$

.

\section{Proof of Theorem~\mainref{thm:image-E}}

Given the sketch in the main paper, it remains only to verify the property from Proposition~\mainref{prop:world-lifting}.  

Let $x = ((\gamma_1, \bar{\gamma}_1), \ldots, (\gamma_5, \bar{\gamma}_5)) \in \mathcal{X} \setminus \operatorname{disc}(\mathscr{M}, \mathbf{L})$, $E \in \mathcal{Y} \cap \operatorname{ker} \mathbf{L}(x)$, and $w = ((I \,\, 0), (\mathbf{R} \,\, \hat{\mathbf{T}}), \Gamma_1, \ldots, \Gamma_5) \in \Psi^{-1}(E) \cap \Phi^{-1}(x)$.  
In terms of $E$, there are four possibilities for $(\mathbf{R}, \hat{\mathbf{T}}) \in \operatorname{SO}(3) \times  \mathbb{S}^2$, see \cite[Res.~9.19]{hartleyzisserman}.   (Only two are distinct up to the action of $G$, see the comment after Lemma~\mainref{lem:double}.) 
Further, the explicit formulas for $(\mathbf{R}, \hat{\mathbf{T}})$ in \cite[Res.~9.19]{hartleyzisserman} are smooth functions of $E$'s SVD factors.  
Shrinking $\mathcal{X}_0$ and $\mathcal{Y}_0$ if necessary, $E$'s SVD factors can be chosen as smooth functions of the essential matrix, even though $E$ has repeated singular values.  Indeed, it is possible because 
\begin{equation*}
\operatorname{SO}(3) \times \operatorname{SO}(3) \rightarrow \mathcal{Y}, \,\,\,\, (\mathbf{R}_1, \mathbf{R}_2) \mapsto \mathbf{R}_1 \operatorname{diag}(1,1,0) \mathbf{R}_2^{\top}
\end{equation*}
is an $\operatorname{SO}(2) \times (\mathbb{Z}/ 2 \mathbb{Z})^{\times 2}$-smooth bundle map, see  \cite{tron2017space}.  
It follows that $\mathbf{R}$ and $\hat{\mathbf{T}}$ locally extend to smooth functions of $E$, thus to smooth functions of the image data $x$ through $f$ from Lemma~\mainref{lem:local-disc}.  Precisely, there exists a smooth function $h_1 : \mathcal{X}_0 \rightarrow \operatorname{SO}(3) \times \mathbb{S}^2$ such that $h_1(x) = (\mathbf{R}, \hat{\mathbf{T}})$ and $\big{(} (\mathbf{R}', \hat{\mathbf{T}}') \mapsto [\hat{\mathbf{T}}']_{\times} \mathbf{R}' \big{)} \circ h_1 = f$.

As for the world points, because $\Phi(w) = x$, they satisfy
\begin{equation}\label{eq:for-worldpoints}
\begin{cases}
(I \,\, 0) (\Gamma_i; \, 1) \,= \, ( \gamma_i; \, 1 ) \\[0.3em]
(\mathbf{R} \,\, \hat{\mathbf{T}}) (\Gamma_i; \, 1) \, = \, ( \bar{\gamma}_i; \, 1 ).
\end{cases}
\end{equation}
Further $\Gamma_i$ are uniquely determined by \eqref{eq:for-worldpoints}, \textbf{provided} the backprojections
\begin{equation}\label{eq:world-lines}
(I \,\, 0)^{-1} (\gamma_i; \, 1) \subseteq \mathbb{P}_{\mathbb{R}}^3 \,\,\,\,\,\, \text{ and } \,\,\,\,\, (\mathbf{R} \,\, \hat{\mathbf{T}})^{-1} (\bar{\gamma}_i; \, 1) \subseteq \mathbb{P}_{\mathbb{R}}^3
\end{equation}
are distinct lines.  
In this case, $\Gamma_i$ locally extend to smooth functions of the image data $x$.  
Indeed, the world points are determined by \eqref{eq:for-worldpoints}, where $(\gamma_i, \bar{\gamma_i})$ are replaced by $(\gamma_i', \bar{\gamma}_i')$ varying in $\mathcal{X}_0$ and $(\mathbf{R} \,\, \hat{\mathbf{T}})$ is replaced by $h_1((\gamma_1', \bar{\gamma}_1'), \ldots, (\gamma_5', \bar{\gamma}_5'))$.
More precisely, after possibly shrinking $\mathcal{X}_0$ again, there is a smooth map $h_2 : \mathcal{X}_0 \rightarrow (\mathbb{R}^3)^{\times 5}$ such that $h(x') := ((I \,\, 0), h_1(x'), h_2(x')) \in \mathcal{W}$ for $x' \in \mathcal{X}_0$ satisfies the requirements in the smooth lifting property.

To finish, we just need to argue that \eqref{eq:world-lines} are distinct lines.
For a contradiction, suppose not.  
Then both lines are the baseline.  
So $\gamma_i$ and $\bar{\gamma_i}$ are the epipoles, \ie 
$(\gamma_i; 1) = \mathbf{R}^{\top} \hat{\mathbf{T}}$  and  $(\bar{\gamma}_i; 1) = \hat{\mathbf{T}}$  in $\mathbb{P}_{\mathbb{R}}^2$.  
But the corresponding epipolar constraint is now orthogonal to  $T(\mathcal{Y}, E)$.  This is because
\begin{equation*}
\hat{\mathbf{T}}^{\top} \big{(} [\hat{\mathbf{T}}^{\perp}_i]_{\times} \mathbf{R} \big{)} \mathbf{R}^{\top}  \hat{\mathbf{T}}  \,\, = \,\, \hat{\mathbf{T}}^{\top} [\hat{\mathbf{T}}^{\perp}_i]_{\times} \hat{\mathbf{T}} \,\, = \,\, 0
\end{equation*}
for $i=1,2$, and 
\begin{equation*}
\hat{\mathbf{T}}^{\top} \big{(} [\hat{\mathbf{T}}]_{\times} [\dot{s}]_{\times} \mathbf{R} \big{)} \mathbf{R}^{\top} \hat{\mathbf{T}} \,\, = \,\,  \big{(} \hat{\mathbf{T}}^{\top} [\hat{\mathbf{T}}]_{\times} \big{)} [\dot{s}]_{\times} \mathbf{R} \mathbf{R}^{\top} \hat{\mathbf{T}} \,\, = \,\, 0
\end{equation*}
for all $\dot{s} \in \mathbb{R}^3$.
It implies $\mathcal{Y} \cap \operatorname{ker} \mathbf{L}(x)$ in $\mathbb{P}(\mathbb{R}^{3 \times 3})$ is not transverse at $E$, contradicting $ x \notin \operatorname{disc}(\mathscr{M}, \mathbf{L})$.  
\hfill $\square$

\section{Proof of Theorem~\mainref{thm:F-image}}
 
We only need to verify the smooth lifting property.

Let $x = ((\gamma_1, \bar{\gamma}_1), \ldots, (\gamma_7, \bar{\gamma}_7)) \in \mathcal{X} \setminus \operatorname{disc}(\mathscr{M}, \mathbf{L})$, $F \in \mathcal{Y} \cap \operatorname{ker} \mathbf{L}(x)$, and $w = ((I \,\, 0), \bar{P}, \tilde{\Gamma}_1, \ldots, \tilde{\Gamma}_7) \in \Psi^{-1}(F) \cap \Phi^{-1}(x)$.  
In terms of $F$, there is only one possibility for $\bar{P}$ up to the action of $\operatorname{PGL}(4)$, see \cite[Res.~9.14]{hartleyzisserman}.  
The explicit formula for $\bar{P}$ in \cite[Res.~9.14]{hartleyzisserman} is a smooth function of $F$, since the epipole in $\mathbb{P}_{\mathbb{R}}^2$ is a smooth function of $F \in \mathcal{Y}$.
So $\bar{P}$ extends to a smooth function of $F$, thus to a smooth function of the image data $x$ as wanted.

The argument for the world points $\tilde{\Gamma}_i$ is similar to the one for Theorem~\mainref{thm:image-E}.  
Again we need to check  
$(I \,\, 0)^{-1} (\gamma_i; \, 1)$ and  $\bar{P}^{-1} (\bar{\gamma}_i; \, 1)$  are distinct lines.
Supposing not, both are the baseline.  
Then $(\gamma_i; 1)$ and  $(\bar{\gamma}_i; 1)$ are the epipoles. 
Their outer product is orthogonal to $T(\mathcal{Y}, F)$ by Lemma~\ref{lem:fund-tangent-char} below.  
Thus $\mathcal{Y} \cap \operatorname{ker} \mathbf{L}(x) \subseteq \mathbb{P}(\mathbb{R}^{3 \times 3})$ is not transverse at $F$, contradicting $x \notin \operatorname{disc}(\mathscr{M}, \mathbf{L})$.  \hfill $\square$

\section{Proof of Proposition~\mainref{prop:crit-ill}}

Consider the uncalibrated case.  Let $w'$ be represented $(P, \bar{P}, \tilde{\Gamma}_1, \ldots, \tilde{\Gamma}_N) \in \mathcal{C}^{\times 2} \times (\mathbb{P}_{\mathbb{R}}^3)^{\times N}$ where $N \geq 7$.  
As explained in the main paper, it suffices to show that there exists a common quadric $\mathcal{Q}$ satisfying the conditions of Theorem~\mainref{thm:F-ill-posed-world} for each minimal subscene $w$ of $w'$.  

Let ${Q} \in S^2(\mathbb{R}^4)$ be a nonzero $4 \times 4$ real symmetric matrix corresponding to the quadric surface $\mathcal{Q}$ in $\mathbb{P}_{\mathbb{R}}^3$.
For different choices of $w$, the second bullet in Theorem~\mainref{thm:F-ill-posed-world} is the same.
The set of ${Q}$'s satisfying the condition is a codimension-$3$ linear subspace of $S^2(\mathbb{R}^4)$.
Parameterize it as
\begin{equation*}
\{{Q}(\alpha) = \alpha_1 {Q}_1 + \ldots + \alpha_7 {Q}_7 : \alpha \in \mathbb{R}^7 \}
\end{equation*}
where ${Q}_i$ form a fixed basis.

For a minimal subscene $w$, the first bullet in Theorem~\mainref{thm:illposed-world} reads:
\begin{equation}\label{eq:one-w}
\tilde{\Gamma}_i^{\top}  {Q}(\alpha) \, \tilde{\Gamma}_i  \, = \,0  \,\,\,\, \text{ for each } \tilde{\Gamma}_i \text{ in } w,
\end{equation}
By assumption, \eqref{eq:one-w} has a nonzero solution in $\alpha$.
Varying $w$, it means that if we consider the $N \times 7$ linear system:
\begin{equation}\label{eq:all-w}
{M}  \alpha \, =  \, 0, \quad \quad
\text{where } M_{ij} := \tilde{\Gamma}_i^{\top} {Q}_j \tilde{\Gamma}_i,
\end{equation}
then each $7 \times 7$ subsystem has a nonzero solution.
We deduce that ${M} \in \mathbb{R}^{N \times 7}$ has rank at most $6$, so \eqref{eq:all-w} has a nonzero solution.
The corresponding quadric surface does the job. 

The calibrated case in Proposition~\mainref{prop:crit-ill} is nearly same. 
The differences are that second and third bullets in Theorem~\mainref{thm:illposed-world} define a codimension-$5$ subspace of $S^2(\mathbb{R}^4)$, and \eqref{eq:all-w} is now an $N \times 5$ linear system of rank at most $4$. 
\hfill $\square$

\section{Background on Manifolds}\label{sec:background-mflds}

We list the tangent spaces and (where relevant) Riemannian metrics for the manifolds used in this paper.  
Recall that in our framework for the condition number, the input and output spaces need to be Riemannian manifolds, while for the world scene space only the smooth structure is used.

\subsection{Image data spaces}

\noindent Consider $\mathcal{X} = (\mathbb{R}^2 \times \mathbb{R}^2)^{\times n}$ where $n=5$ or $7$.  Then $\mathcal{X} \cong \mathbb{R}^{4n}$.  The tangent spaces are all canonically isomorphic to $\mathbb{R}^{4n}$.  The inner products on the tangent spaces are all the Euclidean inner product.

\subsection{World scene space for $5$-point problem}

\noindent Consider $\mathcal{W}$ from Proposition~\mainref{prop:quotient-ess}.
By Lemma~\mainref{lem:double}, an open dense subset of $\operatorname{SO}({3}) \times \, \mathbb{S}^2 \times (\mathbb{R}^{3})^{\times 5}$ is a smooth double cover of $\mathcal{W}$.  
It induces isomorphisms on tangent spaces:
\begin{align} \label{eq:tangent-worldE}
T(\mathcal{W}, ((I \,\, 0), (\mathbf{R} \,\, \hat{\mathbf{T}} ), & \Gamma_1, \ldots, \Gamma_5))  \cong  \nonumber \\ &  T(\operatorname{SO}(3), \mathbf{R}) \times T(\mathbb{S}^2, \hat{\mathbf{T}}) \times (\mathbb{R}^{3})^{\times 5}.
\end{align}
 The tangent spaces to $\operatorname{SO}(3)$ are 
\begin{equation}\label{eq:tangent-SO}
T(\operatorname{SO}(3), \mathbf{R}) = \{ [\dot{s}]_{\times} \mathbf{R} : \dot{s} \in \mathbb{R}^3 \},
\end{equation}
where $[\dot{s}]_{\times} := \begin{pmatrix} 0 & -\dot{s}_3 & \dot{s}_2 \\ \dot{s}_3 & 0 & -\dot{s}_1 \\ -\dot{s}_2 & \dot{s}_1 & 0 \end{pmatrix}$ for $\dot{s} = \begin{pmatrix} \dot{s}_1 \\ \dot{s}_2 \\ \dot{s}_3 \end{pmatrix}$.  
A basis for $T(\operatorname{SO}(3), \mathbf{R})$ is
    $[e_1]_{\times} \R , \,\, [e_2]_{\times} \R , \,\,  [e_3]_{\times} \R,$
where $e_1, e_2, e_3$ is the standard basis on $\mathbb{R}^3$.
The tangent spaces to $\mathbb{S}^2$ are
\begin{equation}\label{eq:sphere-tangent}
    T(\mathbb{S}^2, \hat{\mathbf{T}}) = \{ \dot{\mathbf{T}} \in \mathbb{R}^3 : \langle \dot{\mathbf{T}}, \hat{\mathbf{T}}  \rangle = 0\}.
\end{equation}
We let $\hat{\mathbf{T}}^{\perp}_1, \, \hat{\mathbf{T}}^{\perp}_2$
denote an orthonormal basis of $T(\mathbb{S}^2, \hat{\mathbf{T}})$.

\subsection{World scene space for $7$-point problem}

\noindent Consider $\mathcal{W}$ from Proposition~\mainref{prop:quotient-smooth}.
Lemma~\mainref{lem:b-coords} describes an atlas for $\mathcal{W}$.  
The parameterizations there induce isomorphisms on tangent spaces:
\begin{equation}
T(\mathcal{W}, ((I \,\, 0), \mathbf{M}(b), \tilde{\Gamma}_1, \ldots, \tilde{\Gamma}_7))  \cong 
\mathbb{R}^7 \times (\mathbb{R}^{3})^{\times 7}.
\end{equation}

\subsection{Projective space}\label{subsec:projsp}

\noindent Consider $\mathbb{P}(\mathbb{R}^{3 \times 3})$, the projective space of real $3 \times 3$ matrices, in which essential and fundamental matrices live.
The map:
\begin{equation*} \label{eq:Riem-map}
    \mathbb{S}^{8} = \{M \in \mathbb{R}^{3 \times 3} : \| M \|_F = 1 \} \rightarrow \mathbb{P}(\mathbb{R}^{3 \times 3}), \,\, M \mapsto M,
\end{equation*}
sending a unit-Frobenius norm matrix to its projective point 
expresses $\mathbb{P}(\mathbb{R}^{3 \times 3})$ as a quotient of $\mathbb{S}^8$ by  $\mathbb{Z}/2\mathbb{Z}$. 
Here $\mathbb{Z}/2\mathbb{Z}$ acts by sending $M$ to $-M$, which is an isometry. 
So 
 $\mathbb{P}(\mathbb{R}^{3 \times 3})$ inherits a Riemannian structure such that the quotient map is a local isometry \cite[Exam.~2.34 and Prop.~2.32]{lee2006riemannian}. Therefore,
 \begin{align} \label{eq:proj-space}
T(\mathbb{P}(\mathbb{R}^{3 \times 3}), M)   &\cong  T(\mathbb{S}^8, M/\|M\|_F) \nonumber \\ &= \{\dot{M} \in \mathbb{R}^{3 \times 3} : \langle \dot{M}, M \rangle = 0\},
\end{align}
with inner product the restriction of the Frobenius inner product.  An analogous Riemannian structure exists on other projective spaces, such as $\mathbb{P}(\mathbb{R}^{3 \times 4})$ where cameras live.

\subsection{Output space for $5$-point problem}

\noindent Consider $\mathcal{Y} = \mathcal{E} \subseteq \mathbb{P}(\mathbb{R}^{3 \times 3})$, the manifold of real essential matrices from (\mainref{eq:essential-ideal}).  
It is known that $\mathcal{E}$ is a compact Riemannian manifold of dimension $5$. 
Its Riemannian metric is inherited from $\mathbb{P}(\mathbb{R}^{3 \times 3})$, \ie the inner products on tangent spaces to $\mathcal{E}$ are restrictions of the Frobenius inner product on \eqref{eq:proj-space}. 
Its tangent spaces are characterized by the next lemma.

\begin{lemma} \label{lem:E-submer}
The parameterization
\begin{equation*}
\operatorname{SO}(3) \times \, \mathbb{S}^2 \rightarrow \mathcal{E}, \,\,\, (\R,\hat{\mathbf{T}}) \mapsto [\hat{\mathbf{T}}]_{\times}\R 
\end{equation*}
is submersion, \ie its differential has rank $5$ everywhere.  
\end{lemma}
\begin{proof}
At $(\mathbf{R}, \hat{\mathbf{T}})$, the differential sends $([\dot{s}]_{\times} \mathbf{R}, \dot{\mathbf{T}})$ to 
\begin{equation*}\label{eq:ess-diff}
[\dot{\mathbf{T}}]_{\times} \mathbf{R} + [\hat{\mathbf{T}}]_{\times} [\dot{s}]_{\times} \mathbf{R}.
\end{equation*}
Assume this equals zero.  
Left-multiplying by $\hat{\mathbf{T}}$ and using \eqref{eq:sphere-tangent}, it follows $\dot{\mathbf{T}} = 0$.  Right-multiplying by $\mathbf{R}^{\top}$, we get
\begin{equation*}
[\hat{\mathbf{T}}]_{\times} [\dot{s}]_{\times} = 0.
\end{equation*}
Since $[\hat{\mathbf{T}}]_{\times}$ has rank $2$, $[\dot{s}]_{\times}$ has rank at most $1$.  
It follows that $\dot{s}=0$, and so 
the differential is injective.
\end{proof}

\medskip

It follows that the tangent spaces to $\mathcal{E}$ are 
\begin{multline}
T(\mathcal{E}, [\hat{\mathbf{T}}]_{\times} \mathbf{R}) = \{ [\dot{\mathbf{T}}]_{\times} \mathbf{R} + [\hat{\mathbf{T}}]_{\times} [\dot{s}]_{\times} \mathbf{R} : \\ \dot{\mathbf{T}} \in \mathbb{R}^3, \langle \dot{\mathbf{T}}, \hat{\mathbf{T}} \rangle = 0, \, \dot{s} \in \mathbb{R}^3  \}.
\end{multline}
A non-orthonormal basis for $T(\mathcal{E}, [\hat{\mathbf{T}}]_{\times} \mathbf{R})$ is 
\begin{equation} \label{eq:non-orthonormal}
[\hat{\mathbf{T}}]_{\times} [e_1]_{\times} \mathbf{R}, \,
[\hat{\mathbf{T}}]_{\times} [e_2]_{\times} \mathbf{R}, \,
[\hat{\mathbf{T}}]_{\times} [e_3]_{\times} \mathbf{R},
[\hat{\mathbf{T}}^{\perp}_1]_{\times} \mathbf{R}, \,
[\hat{\mathbf{T}}^{\perp}_2]_{\times} \mathbf{R}.
\end{equation}
Orthonormalizing this gives an orthonormal basis.  
Specifically, we pass between bases using a linear algebra fact:
\begin{lemma}\label{lem:linear-alg}
Let ${X} \in \mathbb{R}^{n \times d}$ have linearly independent columns, and ${z} = {X} \alpha$  for $\alpha \in \mathbb{R}^d$. 
Define ${G} = {X}^{\top} {X} \in \mathbb{R}^{d \times d}$ as the Gram matrix.
Then ${Y} = {X} {G}^{-1/2} \in \mathbb{R}^{n \times d}$ has orthonormal columns which form a basis for the column space of ${X}$, and ${z} = {Y} \beta$ for $\beta = {G}^{1/2} \alpha \in \mathbb{R}^d$.
\end{lemma}
\begin{proof}
Let ${X} = U \Sigma V^{\top}$ be a thin SVD.  
Then ${G}= V \Sigma^2 V^{\top}$,  $({X}^{\top} {X})^{-1/2} = V \Sigma^{-1} V^{\top}$ and ${Y} = U V^{\top}$.  
Clearly $Y$ has orthonormal columns which form a basis for the column space of ${X}$.  
The statement involving $\beta$ is also immediate.
\end{proof}

\medskip

Consequently if $\alpha \in \mathbb{R}^5$ represents an element of 
$T(\mathcal{E}, [\hat{\mathbf{T}}]_{\times} \mathbf{R})$ 
with respect to \eqref{eq:non-orthonormal}, then 
$\beta = {G}^{1/2} \alpha \in \mathbb{R}^{5}$
represents the same element with respect to an orthonormal basis.  Here $G$ is the Gram matrix associated to \eqref{eq:non-orthonormal}, explicitly shown in Fig.~\ref{fig:joe-k}.

\begin{figure*}[h] 
\normalsize
\begin{equation*}\label{eqn:jac-1}
D\Phi_{1}(w) \,\, = \,\,
  \begin{pmatrix}
  & &  & \rvline & & & & & \\
  & &  & \rvline & & & & & \\
  &  \bigeye_{15 \times 15} & &  \rvline & \hspace{8em} \bigzero_{15 \times 5} & & & 
  & \\ & & &  \rvline & & & & & \\[0.7em]
\hline 
& & & \rvline &&& \\[-0.5em]
\R & & & \rvline & 
   \begin{matrix} [{e}_{1}]_{\times} \R \Gamma_{1} & [{ e}_{2}]_{\times} \R\Gamma_{1} & [{ e}_{3}]_{\times}\R \Gamma_{1} \end{matrix} & \hat{\mathbf{T}}_1^{\perp} & \hat{\mathbf{T}}_2^{\perp} \\[1em]
    & \ddots & & \rvline 
  &\vdots & \vdots  & \vdots \\[1em]
  & & \R & \rvline &
 \begin{matrix} [{e}_{1}]_{\times} \R\Gamma_{5} & [{ e}_{2}]_{\times} \R\Gamma_{5} & [{ e}_{3}]_{\times} \R\Gamma_{5} \end{matrix} & \hat{\mathbf{T}}_1^{\perp} & \hat{\mathbf{T}}_2^{\perp} 
\end{pmatrix}
\end{equation*} 
\vspace*{4pt}
\caption{Jacobian matrix for the differential of (\mainref{eq:E-Phi1-main}) at $w = (\mathbf{R}, \hat{\mathbf{T}}, \Gamma_1, \ldots, \Gamma_5)$.  This is used to compute the condition number of the $5$-point problem.}
\label{fig:my-jac1-joe}
\end{figure*}

\begin{figure*}[h]
\begin{equation*} \label{eq:my-crazy-gram}
\begin{small}
G = \begin{pmatrix}
\hat{\mathbf{T}}_1^2 + 1 & \hat{\mathbf{T}}_1 \hat{\mathbf{T}}_2 & \hat{\mathbf{T}}_1 \hat{\mathbf{T}}_3 & \hat{\mathbf{T}}_3 (\hat{\mathbf{T}}^{\perp}_1)_2 - \hat{\mathbf{T}}_2 (\hat{\mathbf{T}}^{\perp}_1)_3 & \hat{\mathbf{T}}_3 (\hat{\mathbf{T}}^{\perp}_2)_2 - \hat{\mathbf{T}}_2 (\hat{\mathbf{T}}^{\perp}_2)_3 \\[2pt] 
\hat{\mathbf{T}}_1 \hat{\mathbf{T}}_2 & \hat{\mathbf{T}}_2^2+1 & \hat{\mathbf{T}}_2 \hat{\mathbf{T}}_3 & -\hat{\mathbf{T}}_3 (\hat{\mathbf{T}}^{\perp}_1)_1+\hat{\mathbf{T}}_1 (\hat{\mathbf{T}}^{\perp}_1)_3 & -\hat{\mathbf{T}}_3 (\hat{\mathbf{T}}^{\perp}_2)_1+\hat{\mathbf{T}}_1 (\hat{\mathbf{T}}^{\perp}_2)_3 \\[2pt] 
\hat{\mathbf{T}}_1 \hat{\mathbf{T}}_3 & \hat{\mathbf{T}}_2 \hat{\mathbf{T}}_3 & \hat{\mathbf{T}}_3^2 + 1 & \hat{\mathbf{T}}_2 (\hat{\mathbf{T}}^{\perp}_1)_1-\hat{\mathbf{T}}_1 (\hat{\mathbf{T}}^{\perp}_1)_2 & \hat{\mathbf{T}}_2 (\hat{\mathbf{T}}^{\perp}_2)_1-\hat{\mathbf{T}}_1 (\hat{\mathbf{T}}^{\perp}_2)_2 \\[2pt] 
\hat{\mathbf{T}}_3 (\hat{\mathbf{T}}^{\perp}_1)_2-\hat{\mathbf{T}}_2 (\hat{\mathbf{T}}^{\perp}_1)_3 & -\hat{\mathbf{T}}_3 (\hat{\mathbf{T}}^{\perp}_1)_1+\hat{\mathbf{T}}_1 (\hat{\mathbf{T}}^{\perp}_1)_3 & \hat{\mathbf{T}}_2 (\hat{\mathbf{T}}^{\perp}_1)_1-\hat{\mathbf{T}}_1 (\hat{\mathbf{T}}^{\perp}_1)_2 & 2 & 0 \\[2pt] 
\hat{\mathbf{T}}_3 (\hat{\mathbf{T}}^{\perp}_2)_2-\hat{\mathbf{T}}_2 (\hat{\mathbf{T}}^{\perp}_2)_3 & -\hat{\mathbf{T}}_3 (\hat{\mathbf{T}}^{\perp}_2)_1+\hat{\mathbf{T}}_1 (\hat{\mathbf{T}}^{\perp}_2)_3 & \hat{\mathbf{T}}_2 (\hat{\mathbf{T}}^{\perp}_2)_1-\hat{\mathbf{T}}_1 (\hat{\mathbf{T}}^{\perp}_2)_2 & 0 & 2
\end{pmatrix}
\end{small}
\end{equation*}
\caption{Gram matrix for basis \eqref{eq:non-orthonormal} of the tangent space to essential matrices.  This is used to compute the condition number for the $5$-point problem.}
\label{fig:joe-k}
\end{figure*}

\begin{figure*}[h] 
\normalsize 
\begin{equation*}
\label{eq:large-displaye}
D\Phi_{1}(w)  = 
  \begin{pmatrix}
  & &  & \rvline & & & & & \\
  & &  & \rvline & & & & & \\
  &  \bigeye_{21 \times 21} & &  \rvline & \hspace{8em} \bigzero_{21 \times 7} & & & 
   \\ & & &  \rvline & & &  \\[0.7em]
\hline 
& & & \rvline &&& \\[-0.5em]
\mathbf{M}(b)(:,1:3) & & & \rvline & \frac{\partial \mathbf{M}(b)}{\partial b_1} \begin{pmatrix} \Gamma_1 \\ 1 \end{pmatrix} & \cdots & \frac{\partial \mathbf{M}(b)}{\partial b_7} \begin{pmatrix} \Gamma_1 \\ 1  \end{pmatrix} \\[1em]
    & \ddots & & \rvline 
  &\vdots & & \vdots \\[1em]
  & & \mathbf{M}(b)(:,1:3) & \rvline & \frac{\partial \mathbf{M}(b)}{\partial b_1} \begin{pmatrix} \Gamma_7 \\ 1  \end{pmatrix}  & \cdots & \frac{\partial \mathbf{M}(b)}{\partial b_7} \begin{pmatrix} \Gamma_7 \\ 1  \end{pmatrix} \\[1em]
\end{pmatrix}
\end{equation*} 
\caption{{Jacobian matrix for the differential of (\mainref{eq:Phi1-7pt}) at $w = (b, \Gamma_1, \ldots, \Gamma_7)$.  This is used to compute the condition number for the $7$-point problem.  Here $\mathbf{M}(b)$ is like in Lemma~\mainref{lem:b-coords}, regarded as in $\mathbb{R}^{3 \times 4}$, and $\mathbf{M}(b)(:,1:3)$ is the left $3 \times 3$ submatrix.  The relevant Gram matrix is given in \eqref{eq:gram-7pt}.}}
\label{eq:DPhi1-7pt-supp}
\vspace*{4pt}
\end{figure*}

\subsection{Output space for $7$-point problem}

\noindent Consider $\mathcal{Y} = \mathcal{F} \subseteq \mathbb{P}(\mathbb{R}^{3 \times 3})$, the manifold of real fundamental matrices.
It is known that $\mathcal{F}$ is a non-compact Riemannian manifold of dimension $7$. 
Its Riemannian metric is inherited from $\mathbb{P}(\mathbb{R}^{3 \times 3})$. 
Its tangent spaces are as follows. 

\begin{lemma}\label{lem:fund-tangent-char}
Let $F \in \mathcal{F} \subseteq \mathbb{P}(\mathbb{R}^{3 \times 3})$.
Using the identification \eqref{eq:proj-space}, the tangent space to $\mathcal{F}$ is 
\begin{equation}\label{eq:F-tangspace}
T(\mathcal{F}, F) = \{\dot{M} \in \mathbb{R}^{3 \times 3}: \langle \dot{M}, F \rangle = 0,  \, \mathbf{e}'^{\top} \dot{M} \mathbf{e}  = 0 \},
\end{equation}
where $\mathbf{e}$, $\mathbf{e}'$ are the right and left epipoles of $F$ respectively.
\end{lemma}
\begin{proof}
The determinant is the defining equation for $\mathcal{F}$. 
So, the tangent space $T(\mathcal{F}, F)$ is the subspace of \eqref{eq:proj-space} orthogonal to the gradient of the determinant. 
By Laplace expansion, 
\begin{align*}
& \frac{\partial \det(F) }{\partial f_{11}} = \det\begin{pmatrix} f_{22} & 
 f_{23} \\ f_{32} & f_{33} \end{pmatrix}, \\[5pt]
 & \frac{\partial \det(F) }{\partial f_{12}} = - \det\begin{pmatrix} f_{21} & 
 f_{23} \\ f_{31} & f_{33} \end{pmatrix}, \\[5pt] 
 & \frac{\partial \det(F) }{\partial f_{13}} = \det\begin{pmatrix} f_{21} & 
 f_{22} \\ f_{31} & f_{32} \end{pmatrix}. 
\end{align*}
By Cramer's rule, this vector spans the right kernel of  
\begin{equation} \label{eq:F-remove}
\begin{pmatrix}
f_{21} & f_{22} & f_{23} \\
f_{31} & f_{32} & f_{33}
\end{pmatrix}
\end{equation}
if \eqref{eq:F-remove} has rank $2$, or else it is $0$.  Thus the first row of $\nabla\det(F)$ is proportional to $\mathbf{e}$.  Similarly, the second and third rows of $\nabla\det(F)$ are proportional to $\mathbf{e}$, because $\mathbf{e}$ spans the right kernel of $F$
Likewise, each column of $\nabla\det(F)$ is proportional to the left epipole $\mathbf{e}'$ of $F$.
Moreover, $\nabla\det(F)$ does not vanish because $F$ has rank $2$.
It follows that the gradient is a nonzero scalar multiple of $\mathbf{e}' \mathbf{e}^{\top}$.
The lemma is clear.
\end{proof}

\medskip

A particular non-orthonormal basis of the tangent space appears in our condition number formula.  
Namely, one can build fundamental matrices using the parameterization $\mathbf{M}(b)$ in Lemma~\mainref{lem:b-coords} for world scenes and the formula after (\mainref{eq:fund-output}) for $F$.  
Doing so parameterizes $\mathcal{F}$ as
\begin{equation} \label{eq:Fb-param}
F(b) = \begin{pmatrix} b_4 & b_5 & b_6 \\ 
-1 & -b_1 & -b_2 \\ 
-b_3 b_4+b_7 & -b_3 b_5+b_1b_7 & -b_3b_6+b_2b_7
\end{pmatrix}.
\end{equation}
The relevant basis for $T(\mathcal{F}, F)$ is then
\begin{equation} \label{eq:easy-basis}
DF(b)(e_1), \ldots, DF(b)(e_7).
\end{equation}
Note that $DF$ may be computed as $DF_3 \circ DF_2 \circ DF_1$, where $F_1 : \mathbb{R}^7 \dashrightarrow \mathbb{R}^{3 \times 3}$ is the same as \eqref{eq:Fb-param} but with output regarded as in Euclidean space, $F_2 : \mathbb{R}^{3 \times 3} \dashrightarrow \mathbb{S}^8$ sends $M$ to $M / \| M \|_F$, and $F_3 : \mathbb{S}^8 \rightarrow \mathbb{P}(\mathbb{R}^{3 \times 3})$ is the natural quotient map.
One checks that $DF_1(b)$ is represented by 
\begin{small}
\begin{equation} \label{eq:my-jay}
J = \begin{pmatrix}
0 & 0 & 0 & 1 & 0 & 0 & 0 \\
0 & 0 & 0 & 0 & 1 & 0 & 0 \\ 
0 & 0 & 0 & 0 & 0 & 1 & 0 \\
0 & 0 & 0 & 0 & 0 & 0 & 0 \\ 
-1 & 0 & 0 & 0 & 0 & 0 & 0 \\ 
0 & -1 & 0 & 0 & 0 & 0 & 0 \\ 
0 & 0 & -b_4 & -b_3 & 0 & 0 & 1 \\ 
b_7 & 0 &  -b_5 &  0 &  -b_3 &  0 &  b_1 \\ 
0 & b_7 & -b_6 &  0 &  0 &  -b_3 &  b_2
\end{pmatrix},
\end{equation}
\end{small}
\!\!$DF_2(F_1(b))$ sends $\dot{M} \mapsto \frac{1}{\|F_1(b)\|_F} \dot{M} - \frac{\langle \dot{M}, F_1(b) \rangle}{\|F_1(b)\|^3_F} F_1(b)$, and $DF_3$ is the identity up to the identification \eqref{eq:proj-space}.

We can orthonormalize  \eqref{eq:easy-basis} using Lemma~\ref{lem:linear-alg}.  
The associated Gram matrix works out to:
\begin{equation}\label{eq:gram-7pt}
G \, = \, \frac{1}{\|f\|_2^2} J^{\top} J - \frac{1}{\| f \|_2^4} (J^{\top} f) (J^{\top} f)^{\top}   \in \mathbb{R}^{7 \times 7},
\end{equation}
with $J$ as in \eqref{eq:my-jay} and $f = \operatorname{vec}(F_1(b))$.

\section{Background on Pl\"ucker coordinates}

In algebraic geometry, it is well-known that the set of codimension-$d$ linear subspaces of $\mathbb{P}^n_{\mathbb{R}}$  form an irreducible smooth projective variety, called the  \textbf{Grassmannian}:
\begin{equation*}
\operatorname{Gr}(\mathbb{P}^{n-d}_{\mathbb{R}}, \mathbb{P}^{n}_{\mathbb{R}}) = \left\{ L \subseteq \mathbb{P}_{\mathbb{R}}^n : \dim(L) = n-d \right\}\!.
\end{equation*}
In Theorem~\mainref{thm:hurwitz} and Section~\mainref{sec:symbolic-curve}, we use a classic coordinate system (more precisely, embedding) for the Grassmannian.

To explain it, let $L \in \operatorname{Gr}(\mathbb{P}^{n-d}_{\mathbb{R}}, \mathbb{P}^{n}_{\mathbb{R}})$.
Write $L$ as the kernel of a full-rank matrix $M \in \mathbb{R}^{d \times (n+1)}$. 
The \textbf{primal Pl\"ucker coordinates} of $L$ are defined as the maximal minors of $M$:
\begin{equation*} \label{eq:primal-plucker-def}
p(L) \, = \, \left( p_{\mathcal{I}}(L) := \det(M(:,\mathcal{I})) \, : \, \mathcal{I} \in \tbinom{[n+1]}{d} \right)
\end{equation*}
It turns out that $p(L)$ is a well-defined point in $\mathbb{P}_{\mathbb{R}}^{\binom{n+1}{d}-1}$, independent of the choice of $M$.
Alternatively we can write $L$ as the row span of a full-rank matrix $N \in \mathbb{R}^{(n-d+1) \times (n+1)}$.  The \textbf{dual Pl\"ucker coordinates} of $L$ are the maximal minors of $N$:
\begin{equation*} \label{eq:dual-plucker}
q(L) \, = \, \left( q_{\mathcal{J}}(L) := \det(N(:,\mathcal{J})) \, : \, \mathcal{J} \in \tbinom{[n+1]}{n+1-d} \right).
\end{equation*}
Again this is a well-defined point $\mathbb{P}_{\mathbb{R}}^{\binom{n+1}{d}-1}$, independent of the choice of $N$.  
Actually, the primal and dual coordinates are equal up to permutation and sign flips. 
 Specifically, 
\begin{small}
\begin{align*}
&\left( p_{\mathcal{I}}(L) : \mathcal{I} \in \tbinom{[n+1]}{d} \right) 
= 
\left( (-1)^{n+|\mathcal{I}|}q_{[n+1] \setminus \mathcal{I}}(L) : \mathcal{I} \in \tbinom{[n+1]}{d}\right)
\end{align*}
\end{small}
\!\!where $|\mathcal{I}| := \sum_{i \in \mathcal{I}} i$. 

In Theorem~\mainref{thm:hurwitz}, the polynomial $\mathbf{P}_{\mathcal{Y}}$ is in $\binom{n+1}{d}$ variables.  
The variables are the maximal minors of $\mathbf{L}(x)$.  The variables are primal Pl\"ucker coordinates of the subspace  $\operatorname{ker} \mathbf{L}(x) \subseteq \mathbb{P}^n_{\mathbb{R}}$ when $\operatorname{rank} \mathbf{L}(x) = d$.  Otherwise, when $\operatorname{rank} \mathbf{L}(x) < d$, the variables are identically zero.  

\section{Details of Figure~\mainref{fig:surface}}

In this section, we provide detailed information about the 
examples shown in Figure~\mainref{fig:surface}. These examples correspond to two cases that satisfy the conditions outlined in Theorem~\mainref{thm:illposed-world}, which characterizes ill-posed world scenes for the $5$-point problem. The first case involves a ruled quadric surface whose intersection with planes perpendicular to the baseline are circles; the second involves a ruled quadric surface intersecting such planes along straight lines. The ``circles" case is illustrated in Figure~\mainref{fig:surface}(a) and (b), while the ``lines" case is illustrated in Figure~\mainref{fig:surface}(c) and (d). The corresponding world points, relative camera poses, and the quadric surface are shown in the MATLAB code below.  We verified in a computer algebra package that the Jacobian of the forward map indeed drops rank from $20$ to $19$ for both of these world scenes; and moreover, the corresponding instance of the $5$-point problem to solve for the essential matrix has a real double root.

\begin{lstlisting}[caption={The details of the ``circles" example in Fig.~\mainref{fig:surface}.}, label={lst:mycode}]
% This is a simple MATLAB example
Gamma1 = [0.378708974148031; 0.200913848191894; 0.873720791235097];
Gamma2 = [0.93044778587466; 0.277575518534703; 0.853099056736192];
Gamma3 = [0.0841289369354994; 0.35332833995602; 0.882313524352031];
Gamma4 = [0.463900203859958; 0.372419972054743; 0.985487184100939];

X5 = 0.5;
Y5 = 0.3;
Z5 = 641450078126244931881602647087753^(1/2)/24541708048417116 - 1629676066696589/24541708048417116;
Gamma5 = [X5;Y5;Z5]

R = [0.111055681307022, 0.542619911087834, -0.83260450859965;...
-0.137592133483409, 0.838113683992865, 0.527857800460874;...
0.984243384764621, 0.0559382230227225, 0.167737517425718];
T = [0.919751547243792; -0.376825581554874; 0.109816084561784]

coeff = [-0.300349680028666, 0.909079843996704, -0.340584616737265,...
0.758476963496853, +1.70002109656252, -3.77717182902557,...
-1.36529273548917, 3.07789643444671, +0.237908403128233, 0];

[X,Y,Z] = meshgrid(-scale:stepSize:scale,-scale:stepSize:scale,-scale:stepSize:scale);
%%Quadric Surface is the 0-level set of the following expression
V = coeff(1) .* X.^2 + coeff(2) .* Y.^2 + coeff(3) .* Z.^2 + coeff(4) .* X.*Y + coeff(5) .* X.*Z...
+ coeff(6) .* Y.*Z + coeff(7) .* X + coeff(8) .* Y + coeff(9) .* Z + coeff(10);
\end{lstlisting}

\begin{lstlisting}[caption={The details of the ``lines" examples in Fig.~\mainref{fig:surface}.}, label={lst:mycode2}]
% This is a simple MATLAB example
Gamma1 = [ 1.91877114; -1.84484369; 48.90957218];
Gamma2 = [ 4.03401915; -3.62525296; 16.73753803];
Gamma3 = [-3.60723653; 3.07391289; 10.52738283];
Gamma4 = [-4.84178757; 4.29437234; 14.68963401];
Gamma5 = [ 3.00632673; -2.67025726; 14.89110311];

R = [1, 0, 0;...
0, 1, 0;...
0, 0, 1];
T = [0; 0; 1]

[X,Y,Z] = meshgrid(-scale:stepSize:scale,-scale:stepSize:scale,-scale:stepSize:scale);
%Quadric Surface is the 0-level set of the following expression
V = Y .* Z + X .* Z + X + 3 .* Y;
\end{lstlisting}

\end{document}